\documentclass[11pt,a4paper]{article}

\usepackage[utf8]{inputenc}
\usepackage[T1]{fontenc}
\usepackage[english]{babel}

\usepackage{mathpazo}
\usepackage{amsmath}
\usepackage{amssymb}
\usepackage{amsthm}
\usepackage{mathtools}
\usepackage{bm}
\usepackage{dsfont}

\usepackage[margin=2.5cm]{geometry}

\usepackage{graphicx}
\usepackage{tikz}
\usepackage{pgfplots}
\pgfplotsset{compat=1.18}
\usetikzlibrary{positioning,arrows.meta,shapes.geometric,calc,patterns,decorations.pathmorphing,backgrounds,fit,shadows,fadings}

\usepackage{booktabs}
\usepackage{array}
\usepackage{tabularx}
\usepackage{longtable}
\usepackage{multirow}

\usepackage{fancyhdr}
\fancypagestyle{firstpage}{
  \fancyhf{} 
  \fancyfoot[C]{\footnotesize Internal version tag: v0x40}

}

\usepackage{xcolor}
\definecolor{bedsblue}{RGB}{0,102,153}
\definecolor{bedsorange}{RGB}{230,138,0}
\definecolor{bedsgreen}{RGB}{0,153,76}
\definecolor{bedsred}{RGB}{180,50,50}
\definecolor{bedspurple}{RGB}{128,0,128}
\definecolor{lightblue}{RGB}{230,242,255}
\definecolor{lightorange}{RGB}{255,245,230}
\definecolor{lightgreen}{RGB}{230,255,240}
\definecolor{lightred}{RGB}{255,235,235}
\definecolor{lightpurple}{RGB}{245,235,250}

\usepackage[numbers,sort&compress]{natbib}
\usepackage{hyperref}
\hypersetup{
    colorlinks=true,
    linkcolor=bedsblue,
    citecolor=bedsgreen,
    urlcolor=bedsorange,
    pdfauthor={Laurent Caraffa},
    pdftitle={Dissipative Learning: A Framework for Viable Adaptive Systems},
    pdfsubject={Machine Learning, Thermodynamics, Information Geometry}
}

\usepackage{enumitem}
\usepackage{float}
\usepackage{placeins}
\usepackage{caption}
\usepackage{subcaption}
\usepackage{microtype}
\usepackage{tcolorbox}
\tcbuselibrary{theorems,skins,breakable}

\theoremstyle{definition}
\newtheorem{definition}{Definition}[section]
\newtheorem{theorem}{Theorem}[section]

\newtheorem{corollary}[theorem]{Corollary}
\newtheorem{proposition}[theorem]{Proposition}

\newtheorem{principle}{Principle}

\newtheorem{assumption}{Assumption}

\theoremstyle{remark}
\newtheorem{remark}{Remark}[section]

\DeclareMathOperator{\KL}{KL}

\newcommand{\R}{\mathbb{R}}

\newcommand{\E}{\mathbb{E}}

\newcommand{\HH}{\mathbb{H}}
\newcommand{\MvM}{\mathcal{M}_{\text{vM}}}
\newcommand{\Normal}{\mathcal{N}}

\newcommand{\bedsL}{\mathcal{L}_{\text{BEDS}}}
\newcommand{\bedsM}{\mathcal{M}_{\text{BEDS}}}
\newcommand{\dF}{d_{\text{F}}}
\newcommand{\kB}{k_{\text{B}}}

\newcommand{\proven}{\textsc{[Proven under A1--A3]}}
\newcommand{\conjectured}{\textsc{[Conjectured]}}
\newcommand{\speculative}{\textsc{[Speculative]}}

\newtcolorbox{keyresult}[1][]{
    enhanced,
    colback=lightblue,
    colframe=bedsblue,
    fonttitle=\bfseries,
    title=#1,
    boxrule=1pt,
    arc=3pt,
    left=8pt,
    right=8pt,
    top=6pt,
    bottom=6pt
}

\newtcolorbox{insight}[1][]{
    enhanced,
    colback=lightorange,
    colframe=bedsorange,
    fonttitle=\bfseries,
    title=#1,
    boxrule=1pt,
    arc=3pt,
    left=8pt,
    right=8pt,
    top=6pt,
    bottom=6pt
}

\newtcolorbox{warning}[1][]{
    enhanced,
    colback=lightred,
    colframe=bedsred,
    fonttitle=\bfseries,
    title=#1,
    boxrule=1pt,
    arc=3pt,
    left=8pt,
    right=8pt,
    top=6pt,
    bottom=6pt
}

\newtcolorbox{conjbox}[1][]{
    enhanced,
    colback=lightpurple,
    colframe=bedspurple,
    fonttitle=\bfseries,
    title=#1,
    boxrule=1pt,
    arc=3pt,
    left=8pt,
    right=8pt,
    top=6pt,
    bottom=6pt
}

\title{\vspace{-1cm}\textbf{Dissipative Learning}\\[0.3cm]
\Large Framework for Viable Adaptive Systems\\[0.2cm]}

\author{Laurent Caraffa\\
\small Universit\'e Gustave Eiffel, LASTIG, IGN-ENSG\\
\small \texttt{laurent.caraffa@ign.fr}}

\date{}

\begin{document}

\maketitle
\thispagestyle{firstpage}

\begin{abstract}
\noindent
Machine learning systems are increasingly deployed in settings requiring continual adaptation, long-term stability, and operation under finite resources. Yet fundamental phenomena such as forgetting, overfitting, and representation collapse remain poorly understood at a principled level, while regularization techniques---weight decay, dropout, batch normalization, exponential moving averages---work empirically without a unifying theoretical explanation.

This work proposes a perspective in which learning is modeled as an intrinsically dissipative process. Rather than viewing forgetting and regularization as heuristic add-ons, we argue they are structural necessities for the viability of adaptive systems. Building on information theory, thermodynamics, and information geometry, we introduce the BEDS (Bayesian Emergent Dissipative Structures) framework, where learning dynamics are described through compressed belief states evolving under dissipation constraints.

A central contribution is the Conditional Optimality Theorem, showing that Fisher--Rao regularization---measuring change in terms of information divergence rather than Euclidean distance---is the unique thermodynamically optimal regularization strategy. This result provides a theoretical lower bound on dissipation and reveals standard Euclidean regularization as structurally suboptimal. The framework naturally unifies existing methods (Ridge, SIGReg, EMA, SAC) as special cases of a single fundamental equation.

From this perspective, overfitting becomes over-crystallization, catastrophic forgetting reflects insufficient dissipation control, and many empirical heuristics can be reinterpreted as implicit mechanisms regulating information dissipation.
The framework introduces a fundamental classification: \emph{BEDS-crystallizable} problems (where beliefs can converge to stable equilibria) versus \emph{BEDS-maintainable} problems (requiring continuous adaptation).
The framework extends to continual and multi-agent systems, where viability---remaining stable, adaptable, and resource-bounded---replaces asymptotic optimality as the primary criterion. Remarkably, several modern architectures---Transformer attention, diffusion models, and Gaussian splatting---naturally implement BEDS coordinates, suggesting deep structural connections across learning paradigms. Overall, this work reframes learning as maintaining viable belief states under finite resources, providing a reference framework for interpreting forgetting, regularization, and stability as first-class principles in adaptive learning.

\end{abstract}

\vspace{0.5cm}
\hrule
\vspace{0.5cm}


\part{Introduction}


\section{Introduction}
\label{sec:introduction}


\subsection{Machine Learning: Developments and Limits}
\label{subsec:ml-history}

\subsubsection{Foundations (1943--1990)}

The history of machine learning begins with formal models of neural computation. McCulloch and Pitts~\cite{mcculloch_pitts_1943} introduced the first mathematical model of a neuron, showing that networks of simple threshold units could compute logical functions. Rosenblatt's perceptron~\cite{rosenblatt_1958} added learning through weight adjustment, raising hopes for general artificial intelligence.

These hopes were tempered by Minsky and Papert's~\cite{minsky_papert_1969} analysis demonstrating fundamental limitations of single-layer perceptrons---notably their inability to learn the XOR function. This contributed to the first ``AI winter,'' a period of reduced funding and interest.

The solution---multi-layer networks trained by backpropagation---was developed independently by Werbos~\cite{werbos_1974}, LeCun~\cite{lecun_1985}, and Rumelhart et al.~\cite{rumelhart_hinton_williams_1986}. Yet practical applications remained limited by computational resources and the lack of large datasets.

\subsubsection{Statistical Learning (1990--2012)}

The 1990s saw the rise of statistical learning theory, providing the first rigorous foundations for generalization. Vapnik's~\cite{vapnik_1995} work on Support Vector Machines introduced the VC dimension and margin-based bounds, enabling provable guarantees on test error.

This era witnessed three pivotal developments. The kernel trick enabled nonlinear classification while preserving the computational tractability of convex optimization. Ensemble methods---notably Random Forests introduced by Breiman~\cite{breiman_2001} and the Boosting algorithms of Freund and Schapire~\cite{freund_schapire_1997}---demonstrated that combining weak learners could yield robust predictors. Meanwhile, Valiant's~\cite{valiant_1984} Probably Approximately Correct (PAC) learning framework provided the first rigorous formalization of the computational complexity of learning, establishing machine learning as a mathematically grounded discipline.

This period established machine learning as a rigorous discipline with theoretical foundations, though practical performance often lagged human capabilities.

\subsubsection{Deep Learning (2012--2020)}

The deep learning revolution began with AlexNet~\cite{krizhevsky_2012}, which achieved 16\% error on ImageNet versus 26\% for previous methods. Three factors converged: large datasets (ImageNet), computational power (GPUs), and architectural innovations (CNNs with ReLU activations).

The subsequent years witnessed rapid architectural innovations including VGG and ResNet~\cite{he_2016}, the latter introducing skip connections that enabled training of networks with hundreds of layers. New regularization techniques emerged, notably Dropout~\cite{srivastava_2014} and Batch Normalization~\cite{ioffe_szegedy_2015}, each providing empirical benefits without clear theoretical justification. Optimization also advanced with Adam~\cite{kingma_ba_2015} and sophisticated learning rate schedules, alongside gradient clipping for sequence models.

A puzzling gap emerged between theory and practice: deep networks with more parameters than training examples nonetheless generalized well, contradicting classical bias-variance tradeoffs. The techniques that made this work---dropout, batch norm, weight decay, data augmentation, early stopping---remained largely disconnected heuristics.

\subsubsection{Foundation Models (2020--2025)}

The Transformer architecture~\cite{vaswani_2017} enabled a new paradigm: massive models trained on internet-scale data exhibiting emergent capabilities. GPT-3~\cite{brown_2020} (175B parameters) and its successors demonstrated surprising abilities from in-context learning to chain-of-thought reasoning.

Several key developments define this era. Kaplan et al.~\cite{kaplan_2020} established scaling laws revealing power-law relationships between compute, data, parameters, and performance. In computer vision, self-supervised methods such as SimCLR~\cite{chen_2020_simclr}, BYOL~\cite{grill_2020}, DINO~\cite{caron_2021}, and I-JEPA~\cite{assran_2023} demonstrated that powerful representations can be learned without labels. LeCun's~\cite{lecun_2022_jepa} Joint Embedding Predictive Architecture proposed learning in latent space rather than pixel space. Concurrently, diffusion models~\cite{ho_2020} achieved state-of-the-art image synthesis through iterative denoising processes.

\textit{Both architectures---Transformers and diffusion models---will emerge as natural
special cases of the BEDS framework: attention implements Bayesian belief updates with
temperature-controlled precision, while diffusion models instantiate the fundamental
dissipation-reconstruction cycle that defines all learning under BEDS.}

\subsubsection*{Energy-Based Models and the Path to BEDS}

LeCun's Energy-Based Models (EBM) framework provides a unifying perspective: learning is energy minimization over a parameterized energy function $E_\theta(x)$. The Boltzmann distribution $p(x) \propto \exp(-E(x)/T)$ connects energy landscapes to probability distributions.

BEDS extends this framework in three critical ways:
\begin{enumerate}
    \item \textbf{Dissipation}: Classical EBMs have $\gamma = 0$---once converged, they freeze. BEDS adds dissipation ($\gamma > 0$, formally defined in Section~\ref{subsec:dissipation}), enabling recovery from crystallization when distributions shift.
    \item \textbf{Learned topology}: EBMs have fixed interaction structures; BEDS learns which agents should communicate.
    \item \textbf{Open systems}: EBMs minimize energy in closed systems; BEDS explicitly models entropy export to the environment.
\end{enumerate}

This distinction---whether a learning problem can ``crystallize'' into a stable state or must be continuously ``maintained''---motivates our classification into \emph{BEDS-crystallizable} and \emph{BEDS-maintainable} regimes, developed formally in Section~\ref{sec:beds-framework}.

\subsubsection{Limits and Pathologies}

Despite remarkable achievements, fundamental problems persist:

\begin{table}[htbp]
\centering
\caption{\textbf{Pathologies and limits of machine learning.} These phenomena motivate our search for a unified theoretical framework.}
\label{tab:pathologies}
\begin{tabular}{p{3cm}p{5cm}p{4cm}}
\toprule
\textbf{Problem} & \textbf{Description} & \textbf{Theoretical Status} \\
\midrule
Hallucinations & LLMs generate false content with high confidence & No predictive theory \\
Energy cost & GPT-4 training estimated at \$100M+ & No fundamental bound \\
Mode collapse & GANs/SSL converge to trivial solutions & Ad hoc heuristics \\
Overfitting / Underfitting & Bias-variance tradeoff & Empirical criteria \\
Double descent & Non-monotonic test error in overparameterized regime & Not predicted classically \\
Catastrophic forgetting & New tasks erase old knowledge & No general solution \\
Adversarial examples & Imperceptible perturbations cause misclassification & Partially understood \\
\bottomrule
\end{tabular}
\end{table}

These pathologies reveal deep structural gaps in our understanding of learning systems.
\textbf{Hallucinations} occur when large language models generate factually incorrect or entirely fabricated content while exhibiting high confidence---a phenomenon for which no predictive theory currently exists, making it impossible to anticipate when a model will confabulate.
\textbf{Energy costs} have reached staggering levels, with GPT-4 training estimated at over \$100 million, yet we lack fundamental bounds relating computational resources to learning capacity, leaving practitioners without principled guidance for resource allocation.
\textbf{Mode collapse} afflicts generative adversarial networks and self-supervised learning systems when they converge to trivial solutions---producing identical outputs or ignoring input variation entirely---a failure mode addressed only through ad hoc architectural heuristics rather than theoretical understanding.
The classical \textbf{overfitting/underfitting} dichotomy, while conceptually clear through the bias-variance tradeoff, still relies on empirical criteria such as cross-validation rather than predictive bounds.
More puzzling is the \textbf{double descent} phenomenon, where test error first decreases, then increases, then decreases again as model complexity grows into the overparameterized regime---a behavior that classical learning theory not only fails to predict but actively contradicts.
\textbf{Catastrophic forgetting} plagues continual learning systems: acquiring new tasks erases previously learned knowledge, with no general solution beyond replay buffers or elastic weight consolidation that themselves introduce new tradeoffs.
Finally, \textbf{adversarial examples} demonstrate that imperceptible perturbations can cause dramatic misclassifications, revealing that learned representations are fragile in ways that remain only partially understood.
Together, these phenomena suggest that current theoretical frameworks are insufficient to capture the dynamics of modern learning systems.

\subsubsection{The Need for a Unified Foundation}
After more than 80 years of development, machine learning remains theoretically fragmented.
Techniques such as dropout, batch normalization, and weight decay each carry separate and
largely heuristic justifications.
No unified theory predicts when pathologies such as overfitting, representation collapse,
or catastrophic forgetting will occur.
Fundamental bounds relating performance, data, memory, and energy remain elusive.
Even the major paradigms---supervised, unsupervised, and reinforcement learning---are
treated as distinct disciplines rather than facets of a single underlying process.

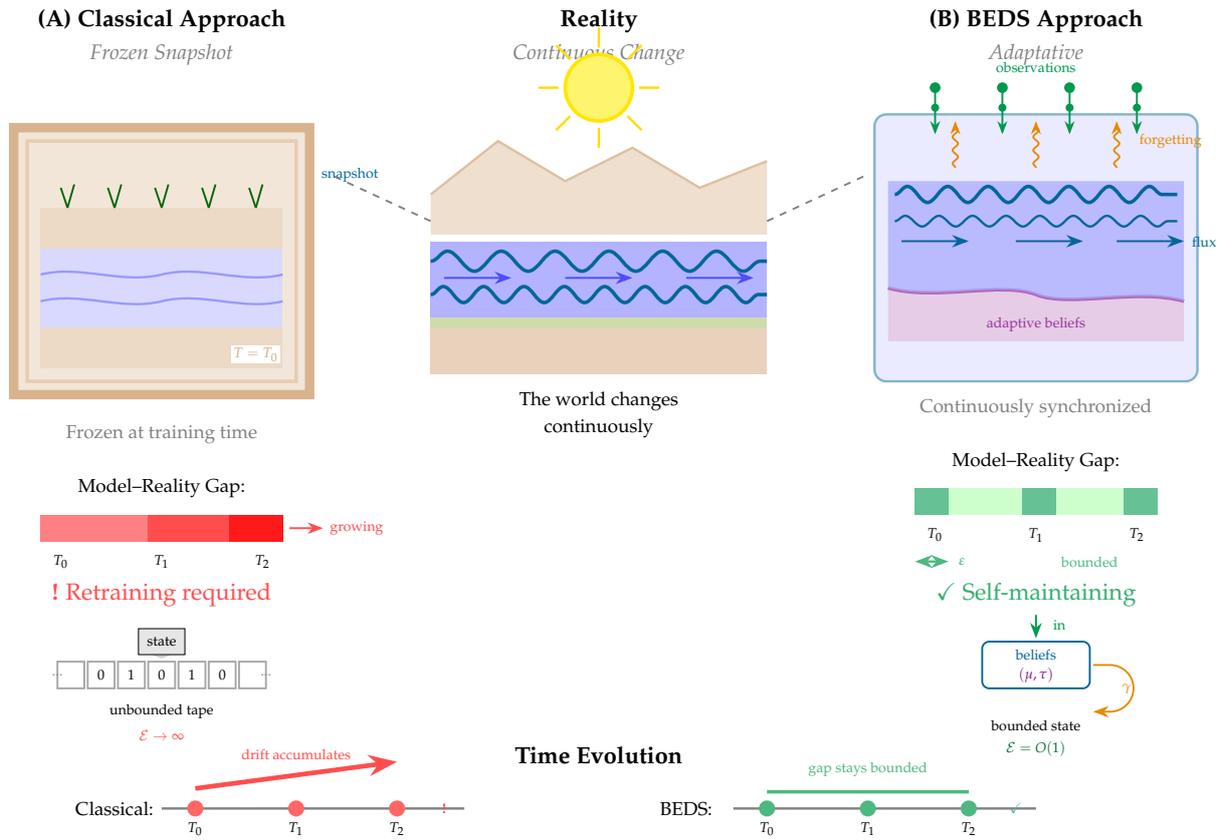
\begin{figure}[H]
\centering
\resizebox{\textwidth}{!}{%
\begin{tikzpicture}[
  font=\small,
  flow/.style={-{Stealth[length=2.5mm]}, thick},
  dissipation/.style={-{Stealth[length=2mm]}, thick, bedsorange},
  observation/.style={-{Stealth[length=2mm]}, thick, bedsgreen},
  tinytext/.style={font=\scriptsize},
  paneltitle/.style={font=\bfseries\small, align=center}
]

\node[font=\bfseries\large] at (0, 5.8) {Digital Twin: Two Paradigms};

\begin{scope}[shift={(-6.5,0)}]
  \node[paneltitle] at (0, 4.8) {(A) Classical Approach};
  \node[font=\footnotesize, gray] at (0, 4.3) {\textit{Frozen Snapshot}};

  \fill[brown!20] (-2.2, -0.8) rectangle (2.2, 3.2);
  \draw[brown!60, line width=4pt] (-2.2, -0.8) rectangle (2.2, 3.2);
  \draw[brown!40, line width=1.5pt] (-2.0, -0.6) rectangle (2.0, 3.0);

  \fill[blue!15] (-1.8, 0.2) rectangle (1.8, 1.4);
  \draw[blue!40, line width=1pt] (-1.8, 1.0)
    .. controls (-1.2, 1.15) and (-0.6, 0.85) .. (0, 1.0)
    .. controls (0.6, 1.15) and (1.2, 0.85) .. (1.8, 1.0);
  \draw[blue!40, line width=1pt] (-1.8, 0.6)
    .. controls (-1.2, 0.75) and (-0.6, 0.45) .. (0, 0.6)
    .. controls (0.6, 0.75) and (1.2, 0.45) .. (1.8, 0.6);
  \fill[brown!30] (-1.8, 1.4) rectangle (1.8, 2.0);
  \fill[brown!30] (-1.8, -0.4) rectangle (1.8, 0.2);
  \foreach \x in {-1.4, -0.7, 0, 0.7, 1.4} {
    \draw[green!40!black, line width=0.8pt] (\x, 2.0) -- (\x-0.1, 2.3);
    \draw[green!40!black, line width=0.8pt] (\x, 2.0) -- (\x+0.1, 2.35);
  }

  \node[font=\tiny, brown!70, fill=white, inner sep=1pt] at (1.4, -0.2) {$T = T_0$};

  \node[font=\scriptsize, align=center, gray] at (0, -1.4) {Frozen at training time};

  \node[font=\tiny, bedsblue] at (2.8, 2.5) {snapshot};

  \node[font=\scriptsize, align=center] at (0, -2.2) {Model--Reality Gap:};
  \fill[red!20, rounded corners=2pt] (-1.8, -3.0) rectangle (1.8, -2.6);
  \fill[red!50] (-1.8, -3.0) rectangle (-0.2, -2.6);
  \fill[red!70] (-0.2, -3.0) rectangle (1.0, -2.6);
  \fill[red!90] (1.0, -3.0) rectangle (1.8, -2.6);
  \draw[-{Stealth}, red!70, thick] (1.9, -2.8) -- (2.4, -2.8);
  \node[font=\tiny, red!70] at (2.9, -2.8) {growing};

  \node[font=\tiny] at (-1.5, -3.3) {$T_0$};
  \node[font=\tiny] at (0, -3.3) {$T_1$};
  \node[font=\tiny] at (1.5, -3.3) {$T_2$};

  \node[font=\small, red!70] at (0, -3.8) {\textbf{!} Retraining required};

  \begin{scope}[shift={(0,-5.2)}]
    \foreach \i in {-3,-2,-1,0,1,2,3} {
      \draw[gray!70, thick] (\i*0.45-0.2, 0) rectangle (\i*0.45+0.2, 0.4);
    }
    \foreach \i/\val in {-2/0, -1/1, 0/0, 1/1, 2/0} {
      \node[font=\tiny] at (\i*0.45, 0.2) {\val};
    }
    \node[font=\tiny, gray] at (-1.55, 0.2) {...};
    \node[font=\tiny, gray] at (1.55, 0.2) {...};

    \fill[gray!40] (-0.15, 0.5) -- (0.15, 0.5) -- (0, 0.45) -- cycle;
    \node[draw, fill=gray!20, font=\tiny, minimum width=0.6cm] at (0, 0.7) {state};

    \node[font=\tiny, align=center] at (0, -0.35) {unbounded tape};
    \node[font=\tiny, red!70, align=center] at (0, -0.7) {$\mathcal{E} \to \infty$};
  \end{scope}
\end{scope}

\begin{scope}[shift={(0,0)}]
  \node[paneltitle] at (0, 4.8) {Reality};
  \node[font=\footnotesize, gray] at (0, 4.3) {\textit{Continuous Change}};

  \fill[yellow!70] (0, 3.8) circle (0.5);
  \draw[yellow!90!orange, line width=1.5pt] (0, 3.8) circle (0.5);
  \foreach \angle in {0, 45, 90, 135, 180, 225, 270, 315} {
    \draw[yellow!80!orange, line width=1pt] (\angle:0.6) ++(0, 3.8) -- ++(\angle:0.3);
  }

  \fill[brown!25] (-2.5, 2.2) -- (-1.5, 3.0) -- (-0.5, 2.4) -- (0.5, 2.9) -- (1.5, 2.3) -- (2.5, 2.7) -- (2.5, 1.6) -- (-2.5, 1.6) -- cycle;
  \draw[brown!50, line width=1pt] (-2.5, 2.2) -- (-1.5, 3.0) -- (-0.5, 2.4) -- (0.5, 2.9) -- (1.5, 2.3) -- (2.5, 2.7);

  \fill[blue!30] (-2.5, 0.2) rectangle (2.5, 1.5);
  \draw[bedsblue, line width=1.5pt, decorate, decoration={snake, amplitude=1.5mm, segment length=8mm}]
    (-2.5, 1.2) -- (2.5, 1.2);
  \draw[bedsblue, line width=1.5pt, decorate, decoration={snake, amplitude=1.2mm, segment length=6mm}]
    (-2.5, 0.7) -- (2.5, 0.7);
  \draw[flow, blue!70] (-2.3, 0.95) -- (-1.3, 0.95);
  \draw[flow, blue!70] (-0.5, 0.95) -- (0.5, 0.95);
  \draw[flow, blue!70] (1.3, 0.95) -- (2.3, 0.95);

  \fill[brown!35] (-2.5, -0.5) rectangle (2.5, 0.2);
  \fill[green!30!brown!40] (-2.5, 0.2) rectangle (2.5, 0.35);

  \node[font=\scriptsize, align=center] at (0, -0.9) {The world changes};
  \node[font=\scriptsize, align=center] at (0, -1.3) {continuously};

  \draw[dashed, gray, line width=0.8pt] (-2.5, 1.8) -- (-4.0, 2.5);
  \draw[dashed, gray, line width=0.8pt] (2.5, 1.8) -- (4.0, 2.5);
\end{scope}

\begin{scope}[shift={(6.5,0)}]
  \node[paneltitle] at (0, 4.8) {(B) BEDS Approach};
  \node[font=\footnotesize, gray] at (0, 4.3) {\textit{Adaptative}};

  \fill[blue!8, rounded corners=5pt] (-2.4, -0.6) rectangle (2.4, 3.4);
  \draw[bedsblue!50, line width=1pt, rounded corners=5pt] (-2.4, -0.6) rectangle (2.4, 3.4);

  \foreach \x in {-1.5, -0.5, 0.5, 1.5} {
    \draw[observation] (\x, 3.8) -- (\x, 3.1);
    \fill[bedsgreen] (\x, 3.8) circle (0.08);
    \fill[bedsgreen] (\x, 3.5) circle (0.06);
  }
  \node[font=\tiny, bedsgreen, align=center] at (0, 4.1) {observations};

  \foreach \x in {-1.2, 0, 1.2} {
    \draw[dissipation, decorate, decoration={snake, amplitude=0.4mm, segment length=2mm}]
      (\x, 2.6) -- (\x, 3.3);
  }
  \node[font=\tiny, bedsorange] at (2.0, 3.0) {forgetting};

  \fill[bedspurple!20] (-2.2, 0.8)
    .. controls (-1.5, 0.6) and (-0.5, 0.9) .. (0, 0.7)
    .. controls (0.5, 0.5) and (1.5, 0.8) .. (2.2, 0.6)
    -- (2.2, 0.0) -- (-2.2, 0.0) -- cycle;
  \draw[bedspurple!60, line width=1.5pt] (-2.2, 0.8)
    .. controls (-1.5, 0.6) and (-0.5, 0.9) .. (0, 0.7)
    .. controls (0.5, 0.5) and (1.5, 0.8) .. (2.2, 0.6);
  \node[font=\tiny, bedspurple!80] at (0, 0.25) {adaptive beliefs};

  \fill[blue!35, opacity=0.7] (-2.2, 0.8)
    .. controls (-1.5, 0.6) and (-0.5, 0.9) .. (0, 0.7)
    .. controls (0.5, 0.5) and (1.5, 0.8) .. (2.2, 0.6)
    -- (2.2, 2.4) -- (-2.2, 2.4) -- cycle;

  \draw[bedsblue, line width=1.5pt, decorate, decoration={snake, amplitude=1.2mm, segment length=6mm}]
    (-2.1, 2.2) -- (2.1, 2.2);
  \draw[bedsblue, line width=1pt, decorate, decoration={snake, amplitude=0.8mm, segment length=5mm}]
    (-2.1, 1.8) -- (2.1, 1.8);

  \draw[flow, bedsblue] (-2.0, 1.5) -- (-1.0, 1.5);
  \draw[flow, bedsblue] (-0.3, 1.5) -- (0.7, 1.5);
  \draw[flow, bedsblue] (1.2, 1.5) -- (2.2, 1.5);
  \node[font=\tiny, bedsblue] at (2.5, 1.5) {flux};

  \node[font=\scriptsize, align=center, gray] at (0, -1.0) {Continuously synchronized};

  \node[font=\scriptsize, align=center] at (0, -1.8) {Model--Reality Gap:};
  \fill[green!20, rounded corners=2pt] (-1.8, -2.6) rectangle (1.8, -2.2);
  \fill[bedsgreen!60] (-1.8, -2.6) rectangle (-1.3, -2.2);
  \fill[bedsgreen!60] (-0.2, -2.6) rectangle (0.3, -2.2);
  \fill[bedsgreen!60] (1.3, -2.6) rectangle (1.8, -2.2);

  \node[font=\tiny] at (-1.5, -2.9) {$T_0$};
  \node[font=\tiny] at (0, -2.9) {$T_1$};
  \node[font=\tiny] at (1.5, -2.9) {$T_2$};

  \draw[bedsgreen!70, line width=1pt, {Stealth}-{Stealth}] (-1.8, -3.3) -- (-1.3, -3.3);
  \node[font=\tiny, bedsgreen!80] at (-1.1, -3.3) {$\varepsilon$};
  \node[font=\tiny, bedsgreen!70] at (0.8, -3.3) {bounded};

  \node[font=\small, bedsgreen!80] at (0, -3.8) {\checkmark\ Self-maintaining};

  \begin{scope}[shift={(0,-5.2)}]
    \draw[bedsblue, thick, rounded corners=3pt] (-0.8, 0) rectangle (0.8, 0.7);
    \node[font=\tiny, bedsblue] at (0, 0.5) {beliefs};
    \node[font=\tiny, bedspurple] at (0, 0.2) {$(\mu, \tau)$};

    \draw[bedsgreen, thick, -{Stealth}] (0, 1.1) -- (0, 0.75);
    \node[font=\tiny, bedsgreen] at (0.35, 0.95) {in};

    \draw[bedsorange, thick, -{Stealth}] (0.85, 0.35) -- (1.1, 0.35)
      arc[start angle=90, end angle=-90, radius=0.35] -- (0.85, -0.35);
    \node[font=\tiny, bedsorange] at (1.35, 0) {$\gamma$};

    \node[font=\tiny, align=center] at (0, -0.55) {bounded state};
    \node[font=\tiny, bedsgreen!70!black, align=center] at (0, -0.9) {$\mathcal{E} = O(1)$};
  \end{scope}
\end{scope}

\begin{scope}[shift={(0,-7.0)}]
  \node[font=\bfseries\small] at (0, 0.8) {Time Evolution};

  \draw[gray, line width=1pt] (-6.5, 0) -- (-2, 0);
  \node[font=\scriptsize, align=right] at (-7.2, 0) {Classical:};
  \foreach \x/\label in {-6/T_0, -4.5/T_1, -3/T_2} {
    \fill[red!60] (\x, 0) circle (0.12);
    \node[font=\tiny, below=2pt] at (\x, 0) {$\label$};
  }
  \draw[red!70, line width=2pt, -{Stealth}] (-6, 0.3) -- (-3, 0.7);
  \node[font=\tiny, red!70, above] at (-4.5, 0.6) {drift accumulates};
  \node[font=\tiny, red!70] at (-2.3, 0) {\textbf{!}};

  \draw[gray, line width=1pt] (2, 0) -- (6.5, 0);
  \node[font=\scriptsize, align=right] at (1.3, 0) {BEDS:};
  \foreach \x/\label in {2.5/T_0, 4/T_1, 5.5/T_2} {
    \fill[bedsgreen!70] (\x, 0) circle (0.12);
    \node[font=\tiny, below=2pt] at (\x, 0) {$\label$};
  }
  \draw[bedsgreen!70, line width=1.5pt] (2.5, 0.25) -- (5.5, 0.25);
  \node[font=\tiny, bedsgreen!70, above] at (4, 0.35) {gap stays bounded};
  \node[font=\tiny, bedsgreen!70] at (6.2, 0) {\checkmark};
\end{scope}

\end{tikzpicture}%
}

\caption{Two paradigms for digital twins and their computational models. Left, a frozen snapshot trained at T0. As reality evolves, the model reality gap grows unboundedly, requiring periodic retraining. The underlying computational model is a Turing machine with unbounded tape: memory accumulates without limit, and the energy cost of maintaining fidelity diverges. Right, a continuously synchronized system that receives observations, updates beliefs, and exports entropy through dissipation. The gap remains bounded through perpetual synchronization. The computational model has bounded state: forgetting prevents memory explosion. This contrast, unbounded tape vs. bounded dissipative state, captures why classical models require ever-growing resources while BEDS systems remain viable indefinitely. The BEDS parameters $(\mu, \tau)$ shown are formally introduced in Definition~\ref{def:beds-state}.}
\label{fig:turing_vs_beds_digital_twin}
\end{figure}

A particularly revealing manifestation of this fragmentation arises in systems that must
operate continuously under finite resources, such as digital twins.
Unlike classical learning settings, these systems are not trained once and deployed,
but must adapt indefinitely to non-stationary environments while remaining stable and viable.
Yet classical computational abstractions, rooted in the Turing machine model, assume unbounded
memory and abstract away the cost of information retention and erasure---a Turing machine
can store arbitrarily many bits without energy cost, making perpetual fidelity trivially
achievable in principle but physically impossible in practice (Figure~\ref{fig:turing_vs_beds_digital_twin}).
By contrast, BEDS systems maintain bounded state through continuous dissipation, trading
perfect fidelity for thermodynamic viability.

In the absence of a principled theory of forgetting, regularization techniques act as
ad hoc mechanisms for information disposal, stabilizing learning without a clear notion
of optimality or necessity.
This gap motivates the search for a unified foundation in which learning is understood
as an intrinsically dissipative process, and in which forgetting, stability, and efficiency
are governed by fundamental constraints rather than empirical tuning.

\begin{insight}[Central Questions]
\begin{enumerate}
    \item Is there a single principle from which effective heuristics derive?
    \item Are regularization techniques optional add-ons or thermodynamic necessities?
    \item Can we predict overfitting, collapse, and loss of adaptability before they occur?
    \item What are the fundamental trade-offs between accuracy, data, memory, and energy?
\end{enumerate}
\end{insight}


\subsection{Dissipative Structures: The Physics}
\label{subsec:dissipative-structures}

\subsubsection{Classical Thermodynamics (1850--1900)}

The second law of thermodynamics, formulated by Clausius~\cite{clausius_1865}, states that entropy in an isolated system never decreases. Boltzmann~\cite{boltzmann_1877} gave this a statistical interpretation: $S = k_B \ln W$, where $W$ counts microstates.

This raises a fundamental puzzle: \emph{how can ordered structures---living organisms, learning systems, crystalline patterns---exist in a universe tending toward maximum entropy?}

\subsubsection{Open Systems and Entropy}

Schr\"odinger's ``What is Life?''~\cite{schrodinger_1944} introduced the concept of \emph{negative entropy} (negentropy): living systems maintain order by exporting entropy to their environment. They are not closed systems approaching equilibrium, but \emph{open systems} exchanging energy and matter with their surroundings.

The key distinction is that closed systems evolve toward equilibrium (maximum entropy), while open systems can maintain non-equilibrium steady states through continuous flux of energy or matter.

\subsubsection{Prigogine's Theory of Dissipative Structures (1960--1977)}

Ilya Prigogine developed a thermodynamic theory of systems far from equilibrium, introducing the concept of \emph{dissipative structures}: ordered patterns that emerge and are maintained through continuous energy dissipation.

\begin{table}[htbp]
\centering
\caption{\textbf{Conditions for dissipative structures (Prigogine).} These four conditions become the basis for our Assumption A3.}
\label{tab:prigogine-conditions}
\begin{tabular}{ll}
\toprule
\textbf{Condition} & \textbf{Description} \\
\midrule
Openness & Exchange of energy/matter with the environment \\
Far from equilibrium & Not at the state of maximum entropy \\
Nonlinearity & Feedback mechanisms amplifying fluctuations \\
Dissipation & Export of entropy to the environment \\
\bottomrule
\end{tabular}
\end{table}

\begin{table}[htbp]
\centering
\caption{\textbf{Examples of dissipative structures.}}
\label{tab:dissipative-examples}
\begin{tabular}{lll}
\toprule
\textbf{System} & \textbf{Incoming Flux} & \textbf{Dissipation} \\
\midrule
B\'enard cells & Temperature gradient & Heat \\
Belousov-Zhabotinsky reaction & Chemical reactants & Reaction products \\
Living organisms & Food/nutrients & Metabolic waste \\
Hurricanes & Ocean heat & Thermal dissipation \\
\bottomrule
\end{tabular}
\end{table}

Prigogine received the Nobel Prize in Chemistry~\cite{prigogine_1977_nobel} for this work, which showed that \emph{order can emerge from chaos through dissipation}.

\subsubsection{Thermodynamics of Information}

Landauer~\cite{landauer_1961} established a fundamental connection between information and thermodynamics: \emph{erasing one bit of information requires dissipating at least $E_{\min} = \kB T \ln 2$ of energy}.

At room temperature (300K), this is approximately $2.87 \times 10^{-21}$ J per bit---tiny, but non-zero. This resolved Maxwell's demon paradox: the demon must erase information about which molecules to sort, and this erasure has an irreducible thermodynamic cost.

Bennett~\cite{bennett_1982} extended this to show that computation itself can be reversible (and thus thermodynamically free), but any \emph{irreversible} operation---including resetting memory---requires dissipation proportional to the information lost.

\begin{keyresult}[Landauer's Principle]
Manipulating information has irreducible thermodynamic consequences. Any system that acquires, maintains, or erases information must pay an energy cost. This principle underpins our Assumption A3 and the resulting Energy-Precision Bound.
\end{keyresult}


\subsection{The Thermodynamic Framework for Learning}
\label{subsec:thermo-framework}

\subsubsection{Correspondence with Prigogine's Conditions}

We propose that \emph{neural networks during training satisfy all four conditions for dissipative structures}:

\begin{table}[htbp]
\centering
\caption{\textbf{Correspondence between dissipative structures and neural networks.} This mapping motivates our conditional framework.}
\label{tab:correspondence}
\begin{tabular}{p{2.5cm}p{4.5cm}p{5cm}}
\toprule
\textbf{Condition} & \textbf{Dissipative Structure} & \textbf{Neural Network in Training} \\
\midrule
Openness & Flux of energy/matter & Flux of data (mini-batches) \\
Far from equilibrium & Ordered state maintained & Structured representations $\neq$ random weights \\
Nonlinearity & Chemical/thermal feedback & Activations, backpropagation, attention \\
Dissipation & Entropy export & Regularization, forgetting, dropout \\
\bottomrule
\end{tabular}
\end{table}

\subsubsection{From Implicit to Explicit Dissipation}

Traditional machine learning treats regularization as an optional add-on: dropout prevents overfitting, weight decay improves generalization, batch normalization stabilizes training. These techniques work, but their success appears coincidental---disconnected heuristics without a unifying principle.

The thermodynamic perspective reveals that these ``heuristics'' may actually be \emph{dissipation mechanisms}. They work because they export entropy. But in current practice, dissipation is \textbf{implicit}: it happens as a side effect, not by design.

We propose a paradigm shift: make dissipation \textbf{explicit}. Rather than discovering post hoc that effective methods happen to dissipate, we design learning systems where entropy export is a first-class architectural principle. This is the core insight of the BEDS framework.

\subsubsection{The Regularization-Forgetting-Energy Chain}

The connection between regularization and thermodynamics is direct and quantitative:

\begin{enumerate}
    \item \textbf{Regularization = Forgetting}: Weight decay pushes parameters toward zero, dropout erases correlations, early stopping prevents memorization. Each technique erases information about the training data.

    \item \textbf{Forgetting = Erasure}: By Landauer's principle, erasing one bit of information requires dissipating at least $E_{\min} = \kB T \ln 2$ of energy.

    \item \textbf{Erasure = Energy Cost}: Therefore, regularization has an irreducible thermodynamic cost. The more you regularize, the more information you erase, the more energy you dissipate.
\end{enumerate}

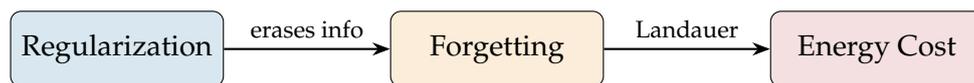
\begin{center}
\begin{tikzpicture}[node distance=5cm, >=Stealth]
    \node[draw, rounded corners, fill=bedsblue!15, minimum width=2.8cm, minimum height=1cm] (reg) {Regularization};
    \node[draw, rounded corners, fill=bedsorange!15, minimum width=2.8cm, minimum height=1cm, right of=reg] (forget) {Forgetting};
    \node[draw, rounded corners, fill=bedsred!15, minimum width=2.8cm, minimum height=1cm, right of=forget] (energy) {Energy Cost};

    \draw[->, thick] (reg) -- node[above, font=\footnotesize] {erases info} (forget);
    \draw[->, thick] (forget) -- node[above, font=\footnotesize] {Landauer} (energy);
\end{tikzpicture}
\end{center}

This chain explains why regularization works: it is not an optional add-on but the \emph{entropy export mechanism} that allows the system to maintain order. Without it, information accumulates without bound (overfitting), violating the conditions for a stable dissipative structure.

\begin{insight}[The Price of Generalization]
The price of generalization is forgetting. The price of forgetting is energy.
\end{insight}

\subsubsection{Consequences and Interpretations}

If this hypothesis is correct, several consequences follow:
\begin{enumerate}
    \item \textbf{Heuristics derive from thermodynamics}: Effective techniques (dropout, weight decay, batch norm) are mechanisms for controlling dissipation
    \item \textbf{Fundamental bounds exist}: Limits relating energy, precision, and dissipation constrain all learning systems
    \item \textbf{Pathologies are thermodynamic}: Overfitting, collapse, and forgetting correspond to identifiable thermodynamic imbalances
    \item \textbf{Paradigms unify}: Supervised, self-supervised, and reinforcement learning are special cases of a single framework
\end{enumerate}

Under this thermodynamic lens, common ML phenomena acquire new interpretations. Overfitting corresponds to information accumulation exceeding dissipation capacity. Mode collapse represents decay toward a low-energy equilibrium state. Generalization, in contrast, corresponds to a stable non-equilibrium steady state where information influx balances dissipation.

\subsubsection{This Motivates Our Three Assumptions}

The empirical success of methods that implement dissipative dynamics, combined with the theoretical apparatus of thermodynamics and information geometry, suggests formalizing the conditions under which we can prove optimality.

Rather than claiming that learning \emph{is} dissipation (a metaphysical claim), we ask: \emph{under what conditions would thermodynamic optimality uniquely determine the regularization strategy?}

This analysis motivates three explicit assumptions. Assumption A1 (Intrinsic Measure) posits that optimality requires a parametrization-invariant metric, which \v{C}encov's theorem identifies as Fisher--Rao. Assumption A2 (Maximum Entropy) specifies that belief states follow maximum-entropy distributions---Gaussians and von Mises families. Assumption A3 (Quasi-Static) asserts that optimal processes approach the quasi-static limit, where geodesics minimize dissipation and Landauer's bound becomes achievable.

The next section makes these assumptions precise and derives their consequences.


\subsection{Learning as Dissipative Processes: From B\'enard to Neural Networks}
\label{subsec:benard-lejepa}

Having established the thermodynamic foundations---Prigogine's conditions for
dissipative structures and Landauer's principle connecting information to
energy---we can now formalize how these principles apply to learning systems.

\subsubsection*{The BEDS Parameters}

The BEDS framework represents belief states using four canonical parameters:

\begin{itemize}
    \item \textbf{Spatial parameters} (from Gaussian beliefs):
    \begin{itemize}
        \item $\mu$ (\emph{position}): where the system believes the true value lies
        \item $\tau$ (\emph{precision}): how confident the system is (inverse variance)
    \end{itemize}
    \item \textbf{Temporal parameters} (from von Mises beliefs):
    \begin{itemize}
        \item $\phi$ (\emph{phase}): the angular/temporal position
        \item $\kappa$ (\emph{coherence}): how strongly synchronized the system is
    \end{itemize}
\end{itemize}

This decomposition follows from Assumption A2 (maximum entropy): Gaussians are
the unique maximum-entropy distributions for unbounded variables with known mean
and variance, while von Mises distributions are their circular analogues. Together,
$(\mu, \tau, \phi, \kappa)$ span the complete BEDS state space (formalized in Definition~\ref{def:beds-state}).

With these parameters defined, consider how self-supervised learning methods
naturally decompose into complementary thermodynamic mechanisms. B\'enard
convection---the canonical dissipative structure---finds not one but \emph{two}
neural analogues: LeJEPA for spatial precision and DINO for temporal coherence.
When a thin fluid layer is heated from below beyond a critical threshold, it
spontaneously organizes into hexagonal convection cells. These \emph{dissipative
structures} are ordered patterns that exist only because energy continuously
flows through the system. SSL methods exhibit the same thermodynamic logic.

\begin{table}[htbp]
\centering
\caption{\textbf{The same thermodynamic story in three different substrates.} LeJEPA and DINO follow the exact pattern of B\'enard convection, targeting complementary BEDS parameters. Under assumptions A1--A3, these correspondences become precise.}
\label{tab:benard-lejepa-unified}
\begin{tabular}{p{2.8cm}p{3.2cm}p{3.2cm}p{3.2cm}}
\toprule
\textbf{Aspect} & \textbf{B\'enard Cells} & \textbf{LeJEPA ($\mu$, $\tau$)} & \textbf{DINO ($\phi$, $\kappa$)} \\
\midrule
System & Fluid layer & Neural network & Neural network \\
External flux & Heat ($\Delta T$) & Data stream & Data stream \\
Equilibrium state & Uniform temperature & Random weights & Random attention \\
What emerges & Hexagonal cells & Spatial precision & Global coherence \\
Dissipation & Heat $\to$ cold & SIGReg $\to$ entropy & Centering $\to$ entropy \\
Control parameter & Rayleigh number & Masking ratio & Momentum $m$ \\
BEDS mapping & --- & Position $\mu$, Precision $\tau$ & Phase $\phi$, Coherence $\kappa$ \\
\bottomrule
\end{tabular}
\end{table}

Gaussian splatting representations follow the same pattern: each Gaussian primitive
$\mathcal{N}(\mu_i, \Sigma_i)$ carries explicit position and precision coordinates,
making it a \emph{native} BEDS representation that complements the learned latent
spaces of LeJEPA and DINO.

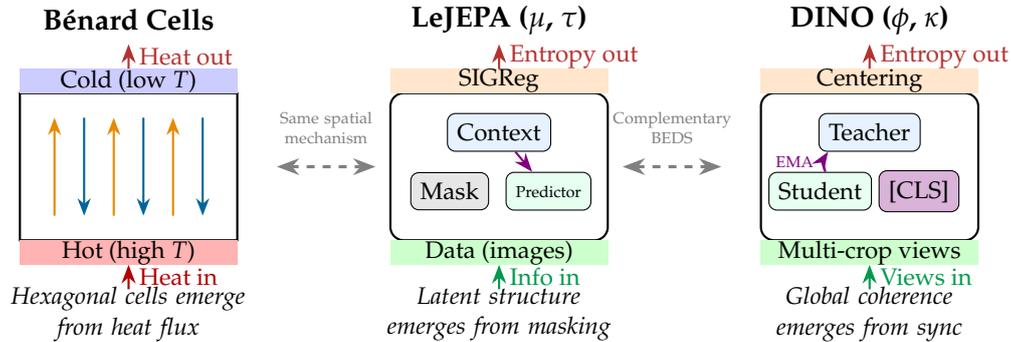
\begin{figure}[H]
\centering
\begin{tikzpicture}[
    scale=0.65,
    every node/.style={font=\small},
    flow/.style={-{Stealth[length=2.5mm]}, thick},
    dissipation/.style={-{Stealth[length=2.5mm]}, thick, bedsred}
]

\begin{scope}[shift={(-7.5,0)}]
    \node[font=\bfseries] at (0,3.5) {B\'enard Cells};

    \draw[thick] (-2.2,-1) rectangle (2.2,2);

    \fill[red!30] (-2.2,-1.5) rectangle (2.2,-1);
    \node[font=\footnotesize] at (0,-1.25) {Hot (high $T$)};

    \fill[blue!20] (-2.2,2) rectangle (2.2,2.5);
    \node[font=\footnotesize] at (0,2.25) {Cold (low $T$)};

    \foreach \x in {-1.2, 0, 1.2} {
        \draw[bedsorange, thick, -{Stealth}] (\x-0.3, -0.5) -- (\x-0.3, 1.5);
        \draw[bedsblue, thick, -{Stealth}] (\x+0.3, 1.5) -- (\x+0.3, -0.5);
    }

    \draw[flow, red!70!black] (0,-2) -- (0,-1.5) node[midway, right, font=\footnotesize] {Heat in};
    \draw[dissipation] (0,2.5) -- (0,3) node[midway, right, font=\footnotesize] {Heat out};

    \node[font=\footnotesize\itshape, text width=3.5cm, align=center] at (0,-2.5) {Hexagonal cells emerge\\from heat flux};
\end{scope}

\draw[{Stealth}-{Stealth}, very thick, gray, dashed] (-4.5, 0.5) -- (-2.5, 0.5);
\node[font=\tiny, gray, text width=1.8cm, align=center] at (-3.5, 1.3) {Same spatial mechanism};

\begin{scope}[shift={(0,0)}]
    \node[font=\bfseries] at (0,3.5) {LeJEPA ($\mu$, $\tau$)};

    \draw[thick, rounded corners=5pt] (-2.2,-1) rectangle (2.2,2);

    \fill[green!20] (-2.2,-1.5) rectangle (2.2,-1);
    \node[font=\footnotesize] at (0,-1.25) {Data (images)};

    \fill[orange!20] (-2.2,2) rectangle (2.2,2.5);
    \node[font=\footnotesize] at (0,2.25) {SIGReg};

    \node[draw, rounded corners=3pt, fill=lightblue, minimum width=1.3cm, minimum height=0.5cm]
        (ctx) at (0, 1.2) {\footnotesize Context};
    \node[draw, rounded corners=3pt, fill=lightgreen, minimum width=1.1cm, minimum height=0.5cm]
        (pred) at (1.0, 0) {\tiny Predictor};
    \node[draw, rounded corners=3pt, fill=gray!20, minimum width=0.9cm, minimum height=0.4cm]
        (mask) at (-1.0, 0) {\footnotesize Mask};

    \draw[flow, bedspurple] (ctx) -- (pred);

    \draw[flow, bedsgreen] (0,-2) -- (0,-1.5) node[midway, right, font=\footnotesize] {Info in};
    \draw[dissipation] (0,2.5) -- (0,3) node[midway, right, font=\footnotesize] {Entropy out};

    \node[font=\footnotesize\itshape, text width=3.5cm, align=center] at (0,-2.5) {Latent structure\\emerges from masking};
\end{scope}

\draw[{Stealth}-{Stealth}, very thick, gray, dashed] (2.5, 0.5) -- (4.5, 0.5);
\node[font=\tiny, gray, text width=2cm, align=center] at (3.5, 1.3) {Complementary BEDS};

\begin{scope}[shift={(7.5,0)}]
    \node[font=\bfseries] at (0,3.5) {DINO ($\phi$, $\kappa$)};

    \draw[thick, rounded corners=5pt] (-2.2,-1) rectangle (2.2,2);

    \fill[green!20] (-2.2,-1.5) rectangle (2.2,-1);
    \node[font=\footnotesize] at (0,-1.25) {Multi-crop views};

    \fill[orange!20] (-2.2,2) rectangle (2.2,2.5);
    \node[font=\footnotesize] at (0,2.25) {Centering};

    \node[draw, rounded corners=3pt, fill=lightblue, minimum width=1.3cm, minimum height=0.5cm]
        (teacher) at (0, 1.2) {\footnotesize Teacher};
    \node[draw, rounded corners=3pt, fill=lightgreen, minimum width=1.3cm, minimum height=0.5cm]
        (student) at (-1.0, 0) {\footnotesize Student};
    \node[draw, rounded corners=3pt, fill=bedspurple!30, minimum width=0.8cm, minimum height=0.4cm]
        (cls) at (1.0, 0) {\footnotesize [CLS]};

    \draw[flow, bedspurple] (student) to[bend left=25] node[midway, left, font=\tiny] {EMA} (teacher);

    \draw[flow, bedsgreen] (0,-2) -- (0,-1.5) node[midway, right, font=\footnotesize] {Views in};
    \draw[dissipation] (0,2.5) -- (0,3) node[midway, right, font=\footnotesize] {Entropy out};

    \node[font=\footnotesize\itshape, text width=3.5cm, align=center] at (0,-2.5) {Global coherence\\emerges from sync};
\end{scope}

\end{tikzpicture}
\caption{\textbf{B\'enard cells, LeJEPA, and DINO---three dissipative structures.} Left: In B\'enard convection, hexagonal cells emerge from heat flux. Center: LeJEPA learns latten  ($\mu$, $\tau$) through masked prediction and SIGReg entropy export. Right: DINO learns global coherence ($\phi$, $\kappa$) through student-teacher synchronization and centering. Together, LeJEPA and DINO span the full BEDS state space. Under A1--A3, these systems satisfy the conditions for thermodynamically optimal learning.}
\label{fig:benard-lejepa}
\end{figure}

The correspondence is not superficial: all three systems maintain ordered structure through continuous flux and entropy export. Figure~\ref{fig:benard-lejepa} visualizes this parallel---heat flows upward in Bénard cells just as gradient information flows through LeJEPA, and both require continuous dissipation to prevent collapse into thermal equilibrium or representation collapse. Remove the flux, and structure collapses.

\textbf{LeJEPA learns \emph{where} things are:} the context encoder forms beliefs $\mu$ about visible regions, while the predictor quantifies precision $\tau$ needed to forecast masked areas. SIGReg prevents over-confidence ($\tau \to \infty$) by exporting entropy---under our framework, this is thermodynamic necessity, not optional regularization.

\textbf{DINO learns \emph{when} features align:} the momentum teacher creates a slow reference frame that student views must synchronize with. The class token [CLS] acts as a global order parameter, and centering/sharpening dissipate incoherent modes. With momentum $m \approx 0.99$, this implies coherence $\kappa \approx 100$---strong phase locking analogous to Kuramoto oscillators.

The two methods are complementary: LeJEPA targets the spatial parameters ($\mu$, $\tau$) while DINO targets the temporal parameters ($\phi$, $\kappa$). Together, they span the full BEDS state space.

\subsubsection*{The Computational Cost of Forgetting}

This thermodynamic view makes a surprising prediction: \emph{forgetting has a cost}. To see this concretely, consider a LeJEPA model with a ViT-Large encoder (300M parameters). Where does the GPU memory go?

\begin{table}[htbp]
\centering
\caption{\textbf{Memory budget of LeJEPA (ViT-Large, 300M params).} A significant fraction of GPU memory is dedicated to implementing controlled forgetting.}
\label{tab:jepa-memory}
\begin{tabular}{lll}
\toprule
\textbf{Component} & \textbf{Memory} & \textbf{Thermodynamic Role} \\
\midrule
Student weights & 1.2 GB & Current beliefs $(\mu, \tau)$ \\
Teacher weights (EMA) & 1.2 GB & Slowly-forgetting reference \\
Adam optimizer states & 2.4 GB & Decaying gradient memory \\
Activations (batch 256) & $\sim$15 GB & Transient computation \\
\midrule
\textbf{Total} & $\sim$\textbf{20 GB} & \\
\bottomrule
\end{tabular}
\end{table}

Of these 20 GB, approximately \textbf{3.6 GB serve exclusively to manage forgetting}. The EMA teacher (1.2 GB) constitutes a memory that forgets slowly: with momentum $m = 0.996$, it retains 99.6\% and forgets 0.4\% per step. The Adam states (2.4 GB) store momentum and variance estimates that decay exponentially---$\beta_1 = 0.9$ implies 10\% forgetting per step.

This is the \textbf{computational price of dissipation}. Weight decay ($\lambda = 0.05$) actively erases 5\% of each weight per step---information about the training data is deliberately destroyed. The EMA teacher exists precisely to provide a stable target \emph{despite} this continuous erasure.

\begin{remark}[Landauer vs Computational Cost]
This memory cost is not the thermodynamic minimum (Landauer's bound gives $\sim 10^{-12}$ J/step, negligible). It is the \emph{engineering cost} of implementing controlled dissipation with current hardware. The analogy holds: maintaining order against entropy requires dedicated resources---whether measured in joules or gigabytes.
\end{remark}

\subsubsection*{Why This Suggests Our Assumptions}

The empirical success of LeJEPA and DINO is suggestive. Both methods share revealing features. They employ metrics that approximate Fisher--Rao geometry through cosine similarity and normalized representations, supporting Assumption A1. They regularize toward Gaussian-like distributions---SIGReg explicitly targets $\Normal(0, I)$---consistent with Assumption A2. Their EMA updates approximate quasi-static processes, aligning with Assumption A3.

This empirical convergence motivates our theoretical investigation: \emph{why} do successful methods share these features?

\begin{keyresult}[Dissipative Learning Hypothesis]
Under assumptions A1--A3, SSL methods work \emph{because} they implement dissipative structures optimally. LeJEPA's SIGReg exports spatial entropy (controls $\tau$); DINO's centering exports temporal entropy (controls $\kappa$). The ``tricks'' are not optional heuristics but thermodynamic necessities.
\end{keyresult}


\subsection{The Conditional Approach}
\label{subsec:conditional-approach}

Rather than claiming universal truths, we adopt a conditional methodology:
\begin{enumerate}
    \item State explicit assumptions (A1, A2, A3)
    \item Derive consequences rigorously
    \item Distinguish proven results from conjectures
    \item Propose testable predictions
\end{enumerate}

This approach has precedent in physics: thermodynamics does not prove that entropy increases; it assumes the second law and derives consequences. Similarly, we assume conditions under which optimality is meaningful, then derive what optimality requires.

\begin{insight}[Central Question]
\emph{Does there exist a fundamentally optimal regularization strategy, derivable from first principles?}
\end{insight}

This work answers affirmatively, with a crucial caveat: optimality is \emph{conditional} on explicit assumptions. We do not claim that learning systems ``are'' dissipative structures in some metaphysical sense. Rather, we prove that \emph{if} one accepts three physically motivated assumptions, \emph{then} a unique regularization strategy emerges as thermodynamically optimal.

\subsubsection{Building on Prior Work}

This framework builds upon foundational contributions in information geometry---particularly \v{C}encov's uniqueness theorem for the Fisher--Rao metric and Amari's work on statistical manifolds---and thermodynamics of computation, including Landauer's principle and Prigogine's theory of dissipative structures. Section~\ref{sec:related-work} provides a detailed positioning relative to these works and recent developments in self-supervised learning.

\begin{insight}[Positioning and Scope of This Work]
This paper is a theoretical contribution.
Its goal is not to propose a new learning algorithm, nor to provide quantitative
performance guarantees on existing methods.
Instead, it introduces a unifying conceptual framework for reasoning about learning
as a dissipative process under finite resources.

\smallskip
\noindent
\textbf{What this work does:}
\begin{itemize}
    \item Introduces a conditional thermodynamic and information-geometric framework (BEDS)
    for modeling learning dynamics under explicit assumptions.
    \item Derives idealized lower bounds and reference geometries for information dissipation.
    \item Provides a principled lens to interpret regularization, forgetting, and stability
    across learning paradigms.
    \item Identifies qualitative regimes, trade-offs, and failure modes relevant to continual
    and resource-constrained systems, such as digital twins.
\end{itemize}

\smallskip
\noindent
\textbf{What this work does not claim:}
\begin{itemize}
    \item It does not claim that real learning algorithms operate in the quasi-static regime
    or follow Fisher--Rao geodesics.
    \item It does not claim that BEDS defines an implementable or optimal training procedure.
    \item It does not provide empirical benchmarks or performance predictions.
    \item It does not model hardware-level energy consumption or discrete computational costs.
\end{itemize}

\smallskip
\noindent
The framework should therefore be understood as a reference model and diagnostic tool,
designed to structure future empirical and methodological work rather than to replace
existing algorithmic approaches.
\end{insight}

\subsection{Structure of This Document}
\label{subsec:structure}

The document proceeds as follows. Section~\ref{sec:beds-visual-guide} provides an overview from entropy to BEDS, establishing intuition before formalism. Section~\ref{sec:related-work} positions BEDS within existing literature on information geometry, thermodynamics of computation, and machine learning. Section~\ref{sec:assumptions} states the three explicit assumptions. Section~\ref{sec:main-theorem} proves the Conditional Optimality Theorem. Section~\ref{sec:beds-framework} develops the BEDS parameter framework. Section~\ref{sec:unification} shows how existing methods emerge as special cases. Section~\ref{sec:efficiency} introduces thermodynamic efficiency of learning. Section~\ref{sec:recursive-hierarchy} develops the recursive structure of BEDS, its connection to Markov Random Fields, Energy-Based Models, and JEPA. Section~\ref{sec:problem-taxonomy} classifies learning problems by crystallization behavior. Section~\ref{sec:predictions} proposes testable predictions. Section~\ref{sec:discussion} addresses limitations and open questions. The appendices contain proofs and speculative material including the GLP conjecture.


\section{BEDS overview}
\label{sec:beds-visual-guide}

This section provides a visual overview of the BEDS framework before diving into the formal mathematical treatment. We trace the logical path from entropy to the complete BEDS state space, showing how each step is necessary rather than arbitrary.


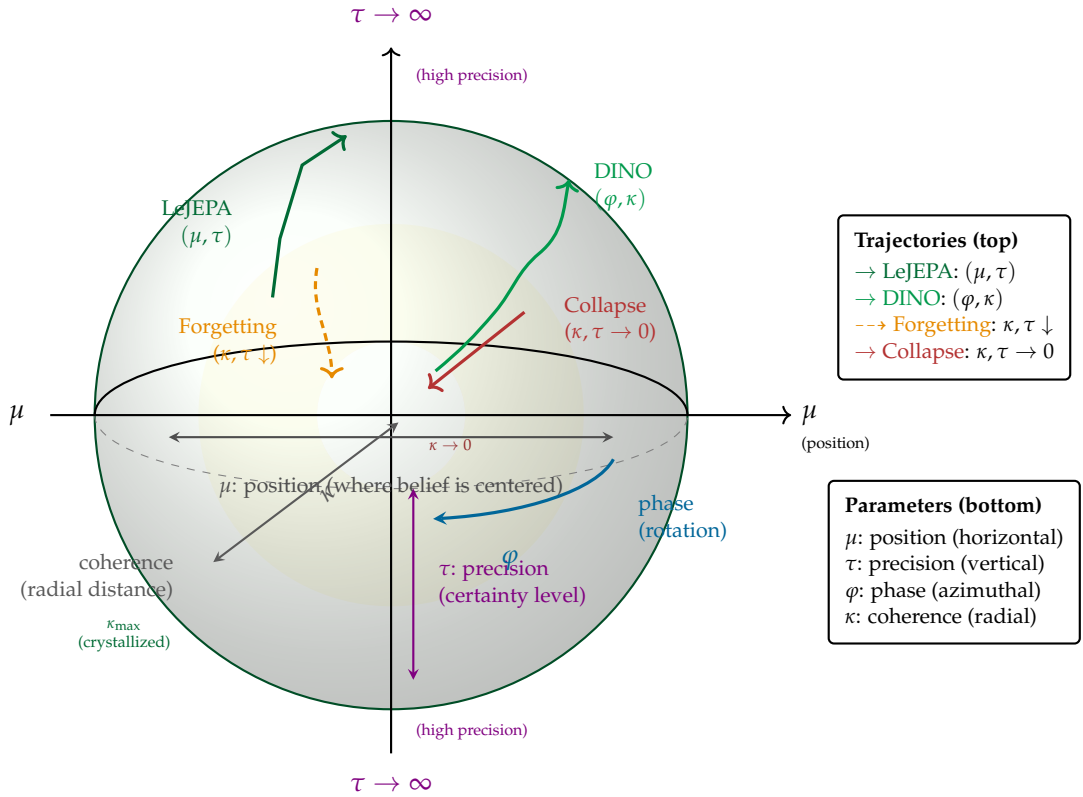
\begin{figure}[H]
\centering
\begin{tikzpicture}[scale=1.95]

  \fill[red!15, opacity=0.7] (0,0) circle (0.5cm);

  \fill[yellow!20, opacity=0.5] (0,0) circle (1.3cm);
  \fill[white] (0,0) circle (0.5cm); 

  \shade[ball color=green!5, opacity=0.4] (0,0) circle (2cm);
  \draw[thick, bedsgreen!50!black] (0,0) circle (2cm);

  \draw[dashed, gray] (-2,0) arc (180:360:2 and 0.5);
  \draw[thick] (-2,0) arc (180:0:2 and 0.5);


  \draw[thick, ->] (0,-2.3) -- (0,2.5);

  \draw[thick, ->] (-2.3,0) -- (2.7,0);


  \node[above, font=\small\bfseries, text=bedspurple] at (0, 2.6) {$\tau \to \infty$};
  \node[right, font=\tiny, text=bedspurple] at (0.1, 2.3) {(high precision)};

  \draw[very thick, bedsgreen!70!black, ->]
    (-0.8, 0.8) -- (-0.75, 1.2) -- (-0.6, 1.7) -- (-0.3, 1.9);
  \node[left, font=\scriptsize, text=bedsgreen!70!black, align=right] at (-1.0, 1.3) {LeJEPA\\$(\mu, \tau)$};

  \draw[very thick, bedsgreen, ->]
    (0.3, 0.3) to[out=40, in=-130] (0.7, 0.7)
    to[out=50, in=-140] (1.0, 1.1)
    to[out=40, in=-100] (1.2, 1.6);
  \node[right, font=\scriptsize, text=bedsgreen, align=left] at (1.3, 1.55) {DINO\\$(\varphi, \kappa)$};

  \draw[very thick, bedsorange, ->, densely dashed]
    (-0.5, 1.0) to[out=-100, in=90] (-0.4, 0.25);
  \node[left, font=\scriptsize, text=bedsorange, align=right] at (-0.7, 0.5) {Forgetting\\($\kappa, \tau \downarrow$)};

  \draw[very thick, bedsred, ->]
    (0.9, 0.7) -- (0.25, 0.18);
  \node[right, font=\scriptsize, text=bedsred, align=left] at (1.1, 0.65) {Collapse\\($\kappa, \tau \to 0$)};


  \node[below, font=\small\bfseries, text=bedspurple] at (0,-2.4) {$\tau \to \infty$};
  \node[right, font=\tiny, text=bedspurple] at (0.1, -2.15) {(high precision)};

  \node[right, font=\small\bfseries] at (2.7, 0) {$\mu$};
  \node[right, font=\tiny] at (2.7, -0.2) {(position)};
  \node[left, font=\small] at (-2.4,0) {$\mu$};

  \draw[thick, black!70, <->, >=stealth] (-1.5, -0.15) -- (1.5, -0.15);
  \node[below, font=\scriptsize, text=black!70] at (0, -0.35) {$\mu$: position (where belief is centered)};

  \draw[thick, bedspurple, <->, >=stealth] (0.15, -0.5) -- (0.15, -1.8);
  \node[right, font=\scriptsize, text=bedspurple, align=left] at (0.25, -1.15) {$\tau$: precision\\(certainty level)};

  \draw[very thick, bedsblue, ->, >=stealth]
    (1.5, -0.3) arc (-10:-80:1.5 and 0.5);
  \node[below, font=\small\bfseries, text=bedsblue] at (0.8, -0.85) {$\varphi$};
  \node[right, font=\scriptsize, text=bedsblue, align=left] at (1.6, -0.7) {phase\\(rotation)};

  \draw[thick, gray!70!black, <->, >=stealth] (0.05,-0.05) -- (-1.2, -1.0);
  \node[below, font=\small\bfseries, text=gray!70!black, rotate=40] at (-0.5, -0.45) {$\kappa$};
  \node[left, font=\scriptsize, text=gray!70!black, align=right] at (-1.4, -1.1) {coherence\\(radial distance)};

  \node[font=\tiny, text=bedsred!80!black] at (0.4, -0.2) {$\kappa \to 0$};
  \node[font=\tiny, text=bedsgreen!70!black, align=center] at (-1.8, -1.5) {$\kappa_{\max}$\\(crystallized)};


  \node[draw, thick, fill=white, font=\scriptsize, align=left,
        rounded corners=2pt, inner sep=6pt] at (3.8, 0.8) {
    \textbf{Trajectories (top)}\\[2pt]
    \textcolor{bedsgreen!70!black}{$\to$ LeJEPA}: $(\mu, \tau)$\\
    \textcolor{bedsgreen}{$\to$ DINO}: $(\varphi, \kappa)$\\
    \textcolor{bedsorange}{$\dashrightarrow$ Forgetting}: $\kappa, \tau \downarrow$\\
    \textcolor{bedsred}{$\to$ Collapse}: $\kappa, \tau \to 0$
  };

  \node[draw, thick, fill=white, font=\scriptsize, align=left,
        rounded corners=2pt, inner sep=6pt] at (3.8, -1.0) {
    \textbf{Parameters (bottom)}\\[2pt]
    $\mu$: position (horizontal)\\
    $\tau$: precision (vertical)\\
    $\varphi$: phase (azimuthal)\\
    $\kappa$: coherence (radial)
  };

\end{tikzpicture}
\caption{\textbf{The BEDS state space with the four canonical parameters.}
\emph{Top hemisphere}: Learning trajectories illustrated by two complementary paradigms---LeJEPA learns on $(\mu, \tau)$
with vertical movement toward poles (increasing precision $\tau \to \infty$), while DINO learns on $(\varphi, \kappa)$
with spiral movement \emph{toward the surface} (coherence $\kappa$ increases toward crystallization while phase $\varphi$ rotates).
Systems may also experience forgetting (decreasing $\kappa$ and $\tau$) or collapse toward center ($\kappa, \tau \to 0$).
\emph{Bottom hemisphere}: The four BEDS parameters visualized---position $\mu$ (horizontal),
precision $\tau$ (vertical), phase $\varphi$ (azimuthal rotation), and coherence $\kappa$ (radial distance).
Poles represent high precision ($\tau \to \infty$); the surface represents crystallized states with maximum coherence ($\kappa_{\max}$);
the center represents collapse ($\kappa \to 0$).}
\label{fig:beds-state-space-canonical}
\end{figure}

The framework naturally encompasses modern architectures: Transformer attention
operates primarily on the temporal component $(\phi, \kappa)$, where $\sqrt{d_k}$
controls attention sharpness (effective $\kappa$), while diffusion models traverse
the full state space---forward process as dissipation ($\tau \to 0$),
reverse process as reconstruction ($\tau$ increases via learned score).
Similarly, Gaussian splatting representations---where each primitive carries
position $\mu$ and precision $\tau = \Sigma^{-1}$---constitute native belief-space
encodings, suggesting that self-supervised learning can operate directly on
Gaussian primitives without requiring learned latent spaces.

\subsection{Neural Networks as Dissipative Structures}
\label{subsec:nn-dissipative}

The previous section established that neural networks during training satisfy Prigogine's four conditions for dissipative structures: (1) openness via continuous data through mini-batches, (2) far-from-equilibrium operation with structured weights distinct from random initialization, (3) nonlinearity in activations, attention, and backpropagation, and (4) dissipation through regularization that exports entropy. This recognition raises a fundamental question: \emph{how do we formalize dissipative learning mathematically?} We need a representation of uncertainty (beliefs about parameters), a geometry to measure distances between belief states, and a notion of optimal trajectories through belief space. The BEDS framework provides exactly this formalization. The following section traces the logical path from these physical insights to the mathematical structure.

\subsection{From Dissipation to BEDS: The Logical Path}
\label{subsec:dissipation-to-beds}

We now trace the logical steps from dissipative structure to mathematical framework. Figure~\ref{fig:entropy-to-beds} summarizes this derivation: each arrow represents a necessary logical step, from Shannon entropy through Prigogine's conditions to the BEDS state space, with assumptions A1--A3 entering at precisely identified points. Each step is necessary---not arbitrary---conditional on our assumptions.

\textbf{Step 1: Representing Uncertainty.}
A dissipative structure maintains order against entropy. To formalize ``order,'' we need to represent beliefs about the system's state. Bayes' theorem provides this: $p(\theta|D) \propto p(D|\theta) p(\theta)$. The posterior distribution encodes what the system has learned; the prior encodes what it would forget without data.

\textbf{Step 2: Which Distributions? (A2)}
Not all probability distributions are suitable. \emph{Under Assumption A2}, we use maximum-entropy distributions given known constraints. For mean and variance constraints: Gaussians $\mathcal{N}(\mu, \tau^{-1})$. For circular constraints: von Mises $\text{vM}(\phi, \kappa)$. These are thermodynamically natural---they maximize disorder subject to what we actually know.

\textbf{Step 3: Measuring Distances (A1)}
How do we measure how far the system is from equilibrium? We need a distance on probability distributions. \emph{Under Assumption A1}, this distance must be parametrization-invariant (intrinsic to the distributions, not their coordinates). \v{C}encov's theorem states: the Fisher--Rao metric is the unique such metric:
\[
ds^2 = \tau \, d\mu^2 + \frac{d\tau^2}{2\tau^2}
\]
This is the metric of the hyperbolic plane $\HH^2$---the natural geometry of Gaussian beliefs.

\textbf{Step 4: Optimal Trajectories (A3)}
In thermodynamics, quasi-static processes minimize dissipation. \emph{Under Assumption A3}, optimal learning trajectories are geodesics in Fisher--Rao geometry---the paths of minimum entropy production.

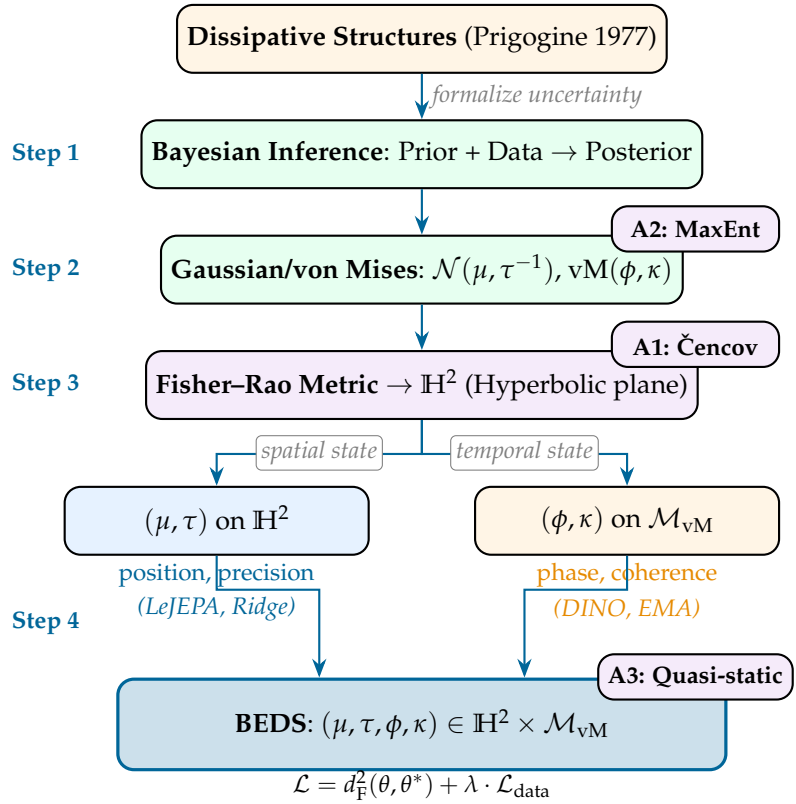
\begin{figure}[H]
\centering
\begin{tikzpicture}[
    scale=0.9,
    every node/.style={font=\small},
    concept/.style={draw, rounded corners=6pt, minimum width=5.5cm, minimum height=0.9cm, thick, fill=white},
    assumption/.style={draw, rounded corners=6pt, minimum width=2.2cm, minimum height=0.6cm, thick, fill=lightpurple, font=\footnotesize\bfseries},
    arrow/.style={-{Stealth[length=2.5mm]}, thick, bedsblue},
    annotation/.style={font=\footnotesize\itshape, text=gray}
]


\node[font=\footnotesize\bfseries, bedsblue] at (-5.5,3.9) {Step 1};
\node[font=\footnotesize\bfseries, bedsblue] at (-5.5,2.2) {Step 2};
\node[font=\footnotesize\bfseries, bedsblue] at (-5.5,0.5) {Step 3};
\node[font=\footnotesize\bfseries, bedsblue] at (-5.5,-3) {Step 4};

\node[concept, fill=lightorange] (prigogine) at (0,5.6) {\textbf{Dissipative Structures} (Prigogine 1977)};

\node[concept, fill=lightgreen] (bayes) at (0,3.9) {\textbf{Bayesian Inference}: Prior + Data $\to$ Posterior};

\node[concept, fill=lightgreen] (gauss) at (0,2.2) {\textbf{Gaussian/von Mises}: $\mathcal{N}(\mu, \tau^{-1})$, vM$(\phi, \kappa)$};
\node[concept, fill=lightpurple] (fisher) at (0,0.5) {\textbf{Fisher--Rao Metric} $\to$ $\HH^2$ (Hyperbolic plane)};

\node[assumption] (A2) at (4,2.8) {A2: MaxEnt};
\node[assumption] (A1) at (4,1.1) {A1: \v{C}encov};

\node[concept, fill=lightblue, minimum width=4cm] (spatial) at (-3,-1.5) {$(\mu, \tau)$ on $\HH^2$};
\node[concept, fill=lightorange, minimum width=4cm] (temporal) at (3,-1.5) {$(\phi, \kappa)$ on $\MvM$};

\node[font=\footnotesize, bedsblue] at (-3,-2.3) {position, precision};
\node[font=\footnotesize, bedsorange] at (3,-2.3) {phase, coherence};

\node[font=\footnotesize\itshape, bedsblue] at (-3,-2.8) {(LeJEPA, Ridge)};
\node[font=\footnotesize\itshape, bedsorange] at (3,-2.8) {(DINO, EMA)};

\node[concept, fill=bedsblue!20, minimum width=8cm, minimum height=1.2cm, very thick, draw=bedsblue, text=black] (beds) at (0,-4.5) {\textbf{BEDS}: $(\mu, \tau, \phi, \kappa) \in \HH^2 \times \MvM$};

\node[assumption] (A3) at (4,-3.8) {A3: Quasi-static};

\node[font=\footnotesize] at (0,-5.4) {$\mathcal{L} = \dF^2(\theta, \theta^*) + \lambda \cdot \mathcal{L}_{\text{data}}$};

\draw[arrow] (prigogine) -- (bayes) node[midway, right, annotation] {formalize uncertainty};
\draw[arrow] (bayes) -- (gauss);
\draw[arrow] (gauss) -- (fisher);
\draw[arrow] (fisher.south) -- ++(0,-0.5) -| (spatial.north);
\draw[arrow] (fisher.south) -- ++(0,-0.5) -| (temporal.north);
\draw[arrow] (spatial.south) -- ++(0,-0.5) -| ([xshift=-1.5cm]beds.north);
\draw[arrow] (temporal.south) -- ++(0,-0.5) -| ([xshift=1.5cm]beds.north);

\node[annotation, fill=white, draw=gray, rounded corners=2pt, inner sep=2pt] at (-1.5,-0.5) {spatial state};
\node[annotation, fill=white, draw=gray, rounded corners=2pt, inner sep=2pt] at (1.5,-0.5) {temporal state};

\end{tikzpicture}
\caption{\textbf{From entropy to BEDS: The logical path.} Each step follows from the previous, with assumptions A1, A2, A3 entering at specific points. Entropy equals uncertainty (Shannon). Order requires entropy export (Prigogine). A2 (MaxEnt) yields Gaussians and von Mises. A1 (\v{C}encov) yields Fisher--Rao geometry. A3 (quasi-static) makes geodesics optimal. The result: BEDS state space $\HH^2 \times \MvM$.}
\label{fig:entropy-to-beds}
\end{figure}

\medskip

The result: the BEDS state space $(\mu, \tau, \phi, \kappa) \in \HH^2 \times \MvM$, where:
\begin{itemize}
    \item $(\mu, \tau)$: spatial beliefs (position, precision) on the hyperbolic plane
    \item $(\phi, \kappa)$: temporal beliefs (phase, coherence) on the von Mises manifold
\end{itemize}

And one unifying loss function:
\[
\mathcal{L} = \dF^2(\theta, \theta^*) + \lambda \cdot \mathcal{L}_{\text{data}}
\]

The complete picture emerges from these considerations: a BEDS agent maintains ordered beliefs $(\mu, \tau, \phi, \kappa)$ against natural dissipation toward prior/equilibrium. This steady state requires continuous power expenditure (bounded below by $P \geq \gamma \kB T / 2$), information influx from data, and entropy export through regularization.

\subsection{Recursive Structure: Crystallization}
\label{subsec:recursive-structure}

BEDS has a \textbf{recursive structure} that mirrors how deep networks operate. Figure~\ref{fig:recursive-structure} illustrates this crystallization process: Layer 1 learns edge detectors with high uncertainty; once learned, these become fixed axioms (crystallized beliefs) that Layer 2 takes as given when learning textures, and so on up the hierarchy. The posterior at level $n$ becomes the prior at level $n+1$: $p_{n+1}(\theta) = p_n(\theta | D_n)$. What was uncertain at one level \emph{crystallizes} into assumed structure at the next.

Consider how a convolutional network works: early layers learn to detect edges (this knowledge crystallizes), and later layers take ``edges exist'' as a given axiom to learn higher-level patterns like faces. Each layer inherits the certainties of the previous layer.

Under Assumption A2, this crystallization process is well-defined: each level's beliefs are maximum-entropy distributions, and the crystallized structure at level $n+1$ is the mean of the posterior at level $n$.

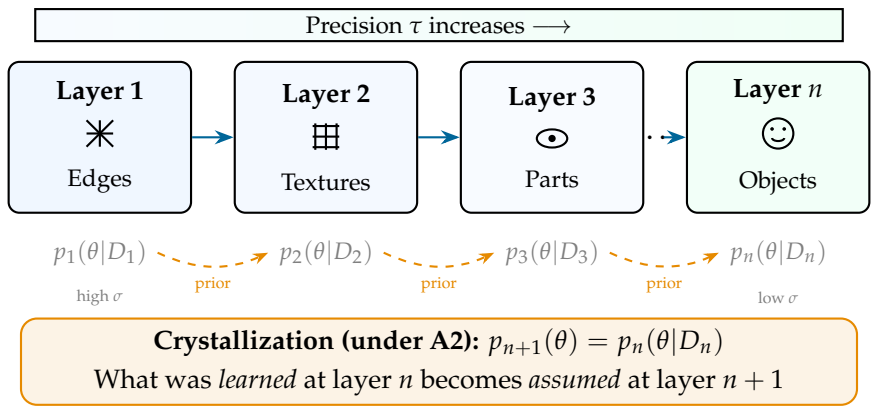
\begin{figure}[H]
\centering
\begin{tikzpicture}[
    scale=0.85,
    every node/.style={font=\small},
    layer/.style={draw, rounded corners=4pt, minimum width=2.4cm, minimum height=2cm, thick, align=center},
    arrow/.style={-{Stealth[length=2.5mm]}, thick, bedsblue},
    crystalarrow/.style={-{Stealth[length=2mm]}, thick, bedsorange, dashed},
    formula/.style={font=\footnotesize, text=gray},
    label/.style={font=\footnotesize\bfseries}
]

\node[layer, fill=lightblue!60] (L1) at (0,0) {
    \textbf{Layer 1}\\[4pt]
    \begin{tikzpicture}[scale=0.35]
        \draw[thick] (-0.5,0.5) -- (0.5,-0.5);
        \draw[thick] (0,0.5) -- (0,-0.5);
        \draw[thick] (-0.5,0) -- (0.5,0);
        \draw[thick] (-0.5,-0.5) -- (0.5,0.5);
    \end{tikzpicture}\\[2pt]
    {\footnotesize Edges}
};

\node[layer, fill=lightblue!45] (L2) at (3.5,0) {
    \textbf{Layer 2}\\[4pt]
    \begin{tikzpicture}[scale=0.35]
        \foreach \i in {-0.4,0,0.4} {
            \draw[thick] (\i,-0.5) -- (\i,0.5);
            \draw[thick] (-0.5,\i) -- (0.5,\i);
        }
    \end{tikzpicture}\\[2pt]
    {\footnotesize Textures}
};

\node[layer, fill=lightblue!30] (L3) at (7,0) {
    \textbf{Layer 3}\\[4pt]
    \begin{tikzpicture}[scale=0.4]
        \draw[thick] (0,0) ellipse (0.5 and 0.3);
        \fill (0,0) circle (0.12);
    \end{tikzpicture}\\[2pt]
    {\footnotesize Parts}
};

\node[layer, fill=lightgreen!50] (L4) at (10.5,0) {
    \textbf{Layer $n$}\\[4pt]
    \begin{tikzpicture}[scale=0.4]
        \draw[thick] (0,0) circle (0.5);
        \fill (-0.2,0.15) circle (0.06);
        \fill (0.2,0.15) circle (0.06);
        \draw[thick] (-0.2,-0.15) arc (180:360:0.2 and 0.1);
    \end{tikzpicture}\\[2pt]
    {\footnotesize Objects}
};

\draw[arrow] (L1) -- (L2);
\draw[arrow] (L2) -- (L3);
\node at (8.75,0) {\large\ldots};
\draw[arrow] ([xshift=0.3cm]L3.east) -- (L4);

\node[formula] (post1) at (0,-1.8) {$p_1(\theta|D_1)$};
\node[formula] (post2) at (3.5,-1.8) {$p_2(\theta|D_2)$};
\node[formula] (post3) at (7,-1.8) {$p_3(\theta|D_3)$};
\node[formula] (postn) at (10.5,-1.8) {$p_n(\theta|D_n)$};

\draw[crystalarrow] (post1.east) to[bend right=25] node[below, font=\tiny, bedsorange] {prior} (post2.west);
\draw[crystalarrow] (post2.east) to[bend right=25] node[below, font=\tiny, bedsorange] {prior} (post3.west);
\draw[crystalarrow] (post3.east) to[bend right=25] node[below, font=\tiny, bedsorange] {prior} (postn.west);

\node[font=\tiny, gray] at (0,-2.5) {high $\sigma$};
\node[font=\tiny, gray] at (10.5,-2.5) {low $\sigma$};

\shade[left color=lightblue!30, right color=lightgreen!60] (-1,1.5) rectangle (11.5,2.0);
\draw[thick] (-1,1.5) rectangle (11.5,2.0);
\node[font=\footnotesize] at (5.25,1.75) {Precision $\tau$ increases $\longrightarrow$};

\node[draw, rounded corners=6pt, fill=bedsorange!10, thick, draw=bedsorange, text=black,
      minimum width=11cm, align=center, font=\small] at (5.25,-3.5) {
    \textbf{Crystallization (under A2):} $p_{n+1}(\theta) = p_n(\theta | D_n)$\\[2pt]
    What was \emph{learned} at layer $n$ becomes \emph{assumed} at layer $n+1$
};

\end{tikzpicture}
\caption{\textbf{Recursive structure in deep networks.} Each layer's posterior becomes the next layer's prior. Early layers learn low-level features (edges, textures) which ``crystallize'' into axioms for later layers. Under A2, this hierarchical belief propagation is well-defined through maximum-entropy distributions.}
\label{fig:recursive-structure}
\end{figure}

This recursive crystallization pattern provides the foundation for
understanding hierarchical learning architectures. The next section
examines how this principle manifests in the energetic structure of
deep networks.


\subsection{Hierarchical Architecture}
\label{subsec:hierarchical-architecture}

The recursive principle established above has direct implications for
the energy requirements of hierarchical systems.

The recursive structure of BEDS enables hierarchical learning, where each level crystallizes patterns from the level below. This architecture (illustrated in Figure~\ref{fig:recursive-levels}) ensures that the posterior at each level becomes the prior for the next, and maintenance energy decreases geometrically with level, making abstract knowledge energetically cheap to maintain.


\subsection{The Energy-Precision Bound (Preview)}
\label{subsec:energy-precision-preview}

Under A1--A3, we derive a fundamental constraint: maintaining precise beliefs requires energy expenditure. Figure~\ref{fig:energy-precision} visualizes this trade-off as a Pareto diagram: the shaded region below $P_{\min}$ is thermodynamically forbidden---no learning system, regardless of architecture, can achieve precision $\tau$ while dissipating less than $\gamma k_B T / 2$ per unit time. Ridge regression, dropout, and BEDS-optimal methods occupy different positions in the achievable region. This is stated precisely as a Corollary in Section~\ref{sec:main-theorem}; here we preview the result visually.

\begin{figure}[H]
\centering
\begin{tikzpicture}[scale=1.0]

\draw[-{Stealth}, thick] (0,0) -- (8,0) node[right] {Precision $\tau$};
\draw[-{Stealth}, thick] (0,0) -- (0,5.5) node[above] {Power $P$};

\draw[bedsred, very thick] (0,1.5) -- (7.5,1.5);
\node[bedsred, font=\footnotesize, anchor=west] at (7.6,1.5) {$P_{\min} = \frac{\gamma \kB T}{2}$};

\fill[bedsred, opacity=0.1] (0,0) rectangle (7.5,1.5);
\node[bedsred, font=\footnotesize] at (3.75,0.7) {FORBIDDEN REGION};

\fill[bedsgreen, opacity=0.1] (0,1.5) rectangle (7.5,5);

\fill[bedsblue] (1.5,2.5) circle (4pt) node[above right, font=\footnotesize] {Ridge};
\fill[bedsorange] (3,3.2) circle (4pt) node[above right, font=\footnotesize] {Dropout};
\fill[bedspurple] (5,4) circle (4pt) node[above right, font=\footnotesize] {SAC};
\fill[bedsgreen] (6.5,4.5) circle (4pt) node[above, font=\footnotesize] {SIGReg};

\draw[gray, dashed] (0,1.5) -- (6.5,4.5);
\node[gray, font=\footnotesize, rotate=25] at (4,2.8) {Pareto frontier};

\node[font=\footnotesize, text width=3cm, align=center] at (1.5, 4.5)
    {Achievable\\(sufficient power)};

\begin{scope}[shift={(0,-0.5)}]
    \node[font=\footnotesize, anchor=west] at (0,0) {$\gamma$ = dissipation rate};
    \node[font=\footnotesize, anchor=west] at (4,0) {$\kB T$ = thermal energy};
\end{scope}

\node[draw, rounded corners=3pt, fill=lightpurple, font=\footnotesize\bfseries] at (6, 5.2) {Corollary of Theorem~\ref{thm:conditional-optimality}};

\end{tikzpicture}
\caption{\textbf{The Energy-Precision bound (Corollary~\ref{cor:energy-precision}).} Under A1--A3, no learning system can operate in the forbidden region below $P_{\min} = \gamma \kB T / 2$. Different methods achieve different trade-offs on the Pareto frontier. This bound is independent of precision---maintaining \emph{any} structured belief costs at least this much.}
\label{fig:energy-precision}
\end{figure}
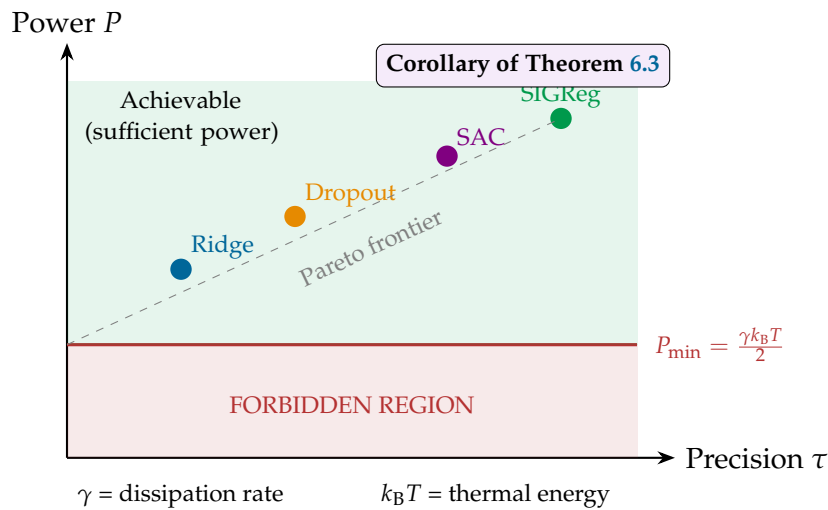

The theorem has profound implications:
\begin{enumerate}[leftmargin=*]
    \item \textbf{No free precision.} You cannot maintain arbitrarily precise beliefs without paying a computational cost. This explains why over-parameterized models require regularization.

    \item \textbf{Regularization = controlled dissipation.} Weight decay, dropout, and early stopping are all mechanisms for managing the energy-precision trade-off.

    \item \textbf{Efficiency is measurable.} Different methods achieve different positions in the energy-precision diagram. SAC and SIGReg operate closer to the thermodynamic bound, making them more efficient.

    \item \textbf{Temperature matters.} In high-noise environments ($\kB T$ large), maintaining precision is more expensive. This formalizes the intuition that learning from noisy data is harder.
\end{enumerate}


\subsection{BEDS in Three Sentences}
\label{subsec:beds-three-sentences}

\begin{keyresult}[BEDS in Three Sentences]
\begin{enumerate}
    \item \textbf{Under A1--A3, optimal regularization uses Fisher--Rao distance.} This is the unique geometry respecting intrinsic information measures (\v{C}encov's theorem via A1).
    \item \textbf{Four parameters $(\mu, \tau, \phi, \kappa)$ on $\HH^2 \times \MvM$ capture belief states.} This parameterization arises from maximum-entropy distributions (A2).
    \item \textbf{One equation unifies all methods:} $\mathcal{L} = \dF^2(\theta, \theta^*) + \lambda \cdot \mathcal{L}_{\text{data}}$
\end{enumerate}
\end{keyresult}

Figure~\ref{fig:five-contributions} provides a visual roadmap of these contributions: the left column shows the assumptions (A1--A3), the middle column the mathematical machinery they unlock (Fisher--Rao, maximum entropy, quasi-static processes), and the right column the resulting tools (unique geometry, thermodynamic bounds, interpretable state space, problem taxonomy).


\subsection{Problem Classes Preview}
\label{subsec:problem-classes-preview}

The product structure $\HH^2 \times \MvM$ suggests a natural classification of learning problems based on whether the system can reach a stable equilibrium or must continuously adapt. Since each component---spatial precision $\tau$ and temporal coherence $\kappa$---can independently be either crystallizable or maintainable, we obtain six distinct problem classes (Figure~\ref{fig:problem-taxonomy}). This taxonomy is not arbitrary: it follows directly from assumptions A1--A3 and determines which regularization strategies are thermodynamically appropriate. We formalize this as:

\begin{itemize}
    \item \textbf{BEDS-crystallizable}: The environment is stationary, and the system can converge to a fixed belief state. Examples: image classification, static regression.
    \item \textbf{BEDS-maintainable}: The environment is non-stationary, and the system must continuously track a moving target. Examples: online learning, continual learning, reinforcement learning in changing environments.
\end{itemize}

This distinction has practical consequences: BEDS-crystallizable problems can use high $\kappa$ (strong temporal coherence), while BEDS-maintainable problems require lower $\kappa$ to remain adaptive.

The formal treatment in Section~\ref{sec:beds-framework} develops these ideas rigorously.

With this intuitive overview complete, we turn to the rigorous development
of the framework. Part II establishes the mathematical foundations, while
Part III develops the practical implications for existing learning methods.

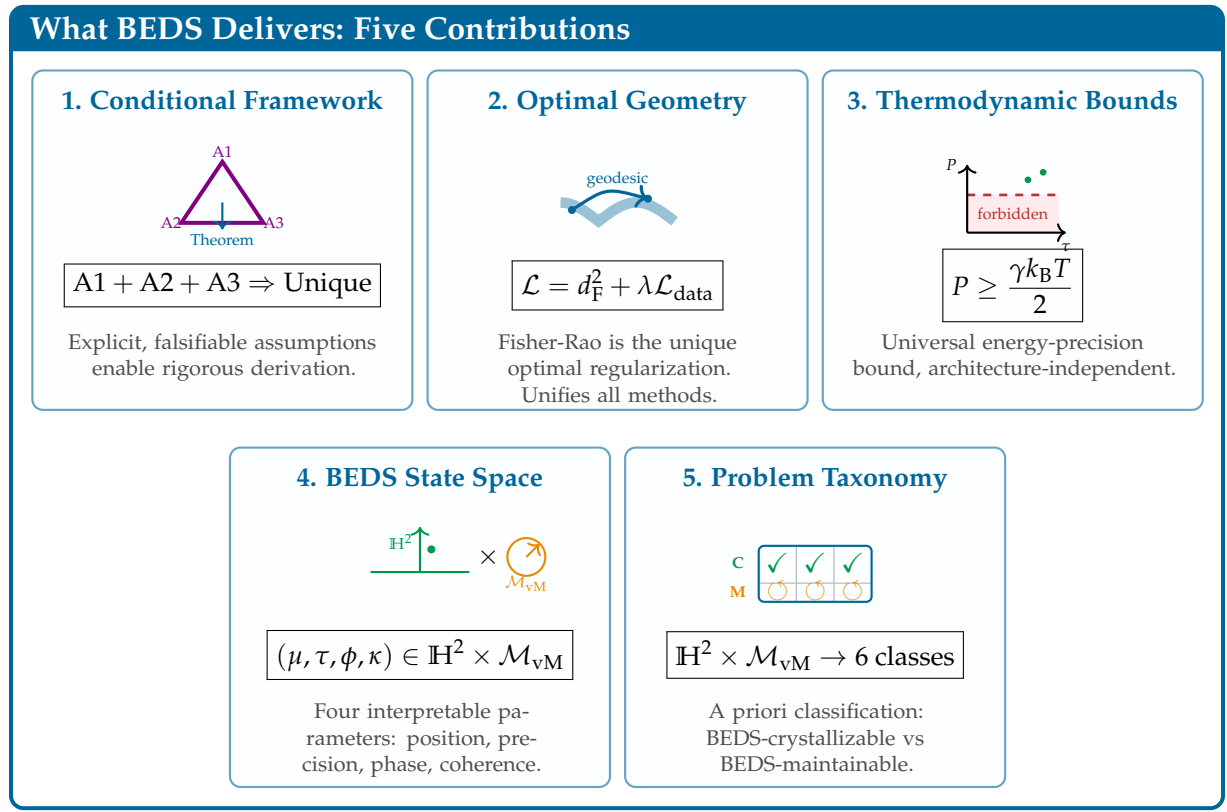
\begin{figure}[H]
\begin{tcolorbox}[
  colback=white,
  colframe=bedsblue,
  title={\textbf{\large What BEDS Delivers: Five Contributions}},
  fonttitle=\large,
  boxrule=1.5pt,
  arc=4pt,
  left=4pt, right=4pt, top=4pt, bottom=4pt
]
\centering
\begin{tikzpicture}[
    box/.style={rectangle, rounded corners=4pt,
        minimum width=5.0cm, minimum height=4.5cm,
        draw=bedsblue!60, line width=0.8pt, fill=lightblue!10, align=center},
    title/.style={font=\small\bfseries, text=bedsblue},
    formula/.style={font=\small},
    desc/.style={font=\scriptsize, text width=4.2cm, align=center, text=black!70}
]

\node[box] (b1) at (0,0) {};
\node[title, anchor=north] at ([yshift=-4pt]b1.north) {1. Conditional Framework};
\begin{scope}[shift={(0,0.5)}, scale=0.45]
    \draw[bedspurple, line width=1.5pt] (0,1.2) -- (-1.2,-0.6) -- (1.2,-0.6) -- cycle;
    \node[font=\tiny, bedspurple] at (0,1.5) {A1};
    \node[font=\tiny, bedspurple] at (-1.5,-0.6) {A2};
    \node[font=\tiny, bedspurple] at (1.5,-0.6) {A3};
    \draw[->, thick, bedsblue] (0,0) -- (0,-0.8);
    \node[font=\tiny, bedsblue] at (0,-1.1) {Theorem};
\end{scope}
\node[formula] at (0,-0.6) {$\boxed{\text{A1}+\text{A2}+\text{A3} \Rightarrow \text{Unique}}$};
\node[desc, anchor=north] at (0,-1.1) {Explicit, falsifiable assumptions enable rigorous derivation.};

\node[box] (b2) at (5.2,0) {};
\node[title, anchor=north] at ([yshift=-4pt]b2.north) {2. Optimal Geometry};
\begin{scope}[shift={(5.2,0.5)}, scale=0.5]
    \draw[bedsblue!40, line width=4pt] (-1.5,0) cos (-0.5,-0.5) sin (0.5,0) cos (1.5,-0.5);
    \draw[bedsblue, line width=1.2pt, ->] (-1.2,-0.2) .. controls (-0.3,0.4) .. (0.8,0.1);
    \fill[bedsblue] (-1.2,-0.2) circle (3pt);
    \fill[bedsblue] (0.8,0.1) circle (3pt);
    \node[font=\tiny, bedsblue] at (0,0.6) {geodesic};
\end{scope}
\node[formula] at (5.2,-0.6) {$\boxed{\mathcal{L} = \dF^2 + \lambda \mathcal{L}_{\text{data}}}$};
\node[desc, anchor=north] at (5.2,-1.1) {Fisher-Rao is the unique optimal regularization. Unifies all methods.};

\node[box] (b3) at (10.4,0) {};
\node[title, anchor=north] at ([yshift=-4pt]b3.north) {3. Thermodynamic Bounds};
\begin{scope}[shift={(10.4,0.5)}, scale=0.5]
    \fill[lightred] (-1.2,-0.8) rectangle (1.2,0.2);
    \draw[bedsred, line width=1pt, dashed] (-1.2,0.2) -- (1.2,0.2);
    \node[font=\tiny, bedsred] at (0,-0.3) {forbidden};
    \draw[->, thick] (-1.2,-0.8) -- (-1.2,1.0) node[left, font=\tiny] {$P$};
    \draw[->, thick] (-1.2,-0.8) -- (1.4,-0.8) node[below, font=\tiny] {$\tau$};
    \fill[bedsgreen] (0.4,0.6) circle (2.5pt);
    \fill[bedsgreen] (0.8,0.8) circle (2.5pt);
\end{scope}
\node[formula] at (10.4,-0.6) {$\boxed{P \geq \dfrac{\gamma \kB T}{2}}$};
\node[desc, anchor=north] at (10.4,-1.1) {Universal energy-precision bound, architecture-independent.};

\node[box] (b4) at (2.6,-5) {};
\node[title, anchor=north] at ([yshift=-4pt]b4.north) {4. BEDS State Space};
\begin{scope}[shift={(2.6,-4.4)}, scale=0.5]
    \draw[bedsgreen, thick] (-1.3,0) -- (1.3,0);
    \draw[bedsgreen, thick, ->] (0,0) -- (0,1.2);
    \fill[bedsgreen] (0.3,0.6) circle (3pt);
    \node[font=\tiny, bedsgreen] at (-0.5,0.8) {$\HH^2$};
    \node[font=\small] at (1.8,0.4) {$\times$};
    \draw[bedsorange, thick] (2.8,0.4) circle (0.5);
    \draw[bedsorange, thick, ->] (2.8,0.4) -- (3.15,0.75);
    \node[font=\tiny, bedsorange] at (2.8,-0.3) {$\MvM$};
\end{scope}
\node[formula] at (2.6,-5.5) {$\boxed{(\mu,\tau,\phi,\kappa) \in \HH^2 \times \MvM}$};
\node[desc, anchor=north] at (2.6,-6.0) {Four interpretable parameters: position, precision, phase, coherence.};

\node[box] (b5) at (7.8,-5) {};
\node[title, anchor=north] at ([yshift=-4pt]b5.north) {5. Problem Taxonomy};
\begin{scope}[shift={(7.8,-4.3)}, scale=0.55]
    \draw[gray!50, line width=0.5pt] (-1.35,-0.45) -- (1.35,-0.45);
    \draw[gray!50, line width=0.5pt] (-0.45,-0.9) -- (-0.45,0.45);
    \draw[gray!50, line width=0.5pt] (0.45,-0.9) -- (0.45,0.45);
    \draw[bedsblue, line width=0.8pt, rounded corners=2pt] (-1.35,-0.9) rectangle (1.35,0.45);
    \node[font=\tiny\bfseries, bedsgreen] at (-1.85,0.05) {C};
    \node[font=\tiny\bfseries, bedsorange] at (-1.85,-0.65) {M};
    \node[font=\small, bedsgreen] at (-0.9,0.05) {\checkmark};
    \node[font=\small, bedsgreen] at (0,0.05) {\checkmark};
    \node[font=\small, bedsgreen] at (0.9,0.05) {\checkmark};
    \node[font=\small, bedsorange] at (-0.9,-0.65) {$\circlearrowleft$};
    \node[font=\small, bedsorange] at (0,-0.65) {$\circlearrowleft$};
    \node[font=\small, bedsorange] at (0.9,-0.65) {$\circlearrowleft$};
\end{scope}
\node[formula] at (7.8,-5.5) {$\boxed{\HH^2 \times \MvM \to 6 \text{ classes}}$};
\node[desc, anchor=north] at (7.8,-6.0) {A priori classification: BEDS-crystallizable vs BEDS-maintainable.};

\end{tikzpicture}
\end{tcolorbox}
\caption{\textbf{The five contributions of BEDS.} All results are proven under assumptions A1--A3. The framework provides: (1) explicit foundations, (2) unique optimal geometry, (3) thermodynamic bounds, (4) interpretable state space, (5) problem classification.}
\label{fig:five-contributions}
\end{figure}


\section{Related Work}
\label{sec:related-work}

This section positions the BEDS framework within existing literature, clarifying both the foundations we build upon and the novel contributions we introduce. We organize the discussion along three axes: (1) information geometry and complexity measures in machine learning, (2) thermodynamics of information and computation, and (3) recent attempts to bridge these two domains.

\subsection{Information Geometry and Complexity Measures}
\label{subsec:info-geometry}

\subsubsection{The Fisher--Rao Metric: Foundations}

The use of differential geometry in statistical inference has a rich history. Rao~\cite{rao_1945} first introduced the Fisher information matrix as a Riemannian metric on statistical manifolds, enabling the measurement of distances between probability distributions. The fundamental uniqueness result came from \v{C}encov~\cite{cencov_1982}: \emph{the Fisher--Rao metric is the only Riemannian metric (up to a constant factor) that is invariant under sufficient statistics and Markov morphisms}. This theorem provides the mathematical foundation for our Assumption A1 (Intrinsic Measure)---we do not choose Fisher--Rao arbitrarily; it is the unique choice satisfying coordinate-independence.

Amari and Nagaoka~\cite{amari_nagaoka_2000} developed the comprehensive treatment of information geometry in their monograph \emph{Methods of Information Geometry}, establishing the dual connection structure ($\nabla$ and $\nabla^*$) that characterizes statistical manifolds. For Gaussian families, the Fisher--Rao metric takes the explicit form:
\begin{equation}
ds^2 = \tau \, d\mu^2 + \frac{1}{2\tau^2} d\tau^2
\end{equation}
which is precisely the metric of the hyperbolic half-plane $\HH^2$. This geometric structure underlies the spatial component of the BEDS state space.

\subsubsection{Fisher--Rao for Neural Networks: The Liang et al.\ Framework}

The work most closely related to ours is Liang et al.~\cite{liang_2019}, ``Fisher-Rao Metric, Geometry, and Complexity of Neural Networks.'' They introduced the \emph{Fisher--Rao norm} as a capacity measure for deep networks:
\begin{equation}
\|\theta\|_{\text{fr}}^2 = \langle \theta, I(\theta) \theta \rangle = \E\left[\langle \nabla_\theta \ell(f_\theta(X), Y), \theta \rangle^2\right]
\end{equation}
and proved an elegant analytical formula:
\begin{equation}
\|\theta\|_{\text{fr}}^2 = (L+1)^2 \, \E\left[\left\langle \frac{\partial \ell}{\partial f_\theta}, f_\theta \right\rangle^2\right]
\end{equation}
where $L$ is the network depth.

Their central result is the \textbf{umbrella theorem}: the Fisher--Rao norm dominates all other commonly used complexity measures:
\begin{equation}
\frac{1}{L+1} \|\theta\|_{\text{fr}} \leq \|\theta\|_X \quad \text{for } X \in \{\text{spectral, group, path}\}
\end{equation}
This implies that spectral norm bounds~\cite{bartlett_2017}, group norms, and path norms~\cite{neyshabur_2015} are all upper-bounded by Fisher--Rao.

\begin{table}[htbp]
\centering
\caption{\textbf{Comparison: Liang et al.~\cite{liang_2019} vs.\ BEDS.} Both frameworks use Fisher--Rao geometry, but with fundamentally different objectives.}
\label{tab:liang-vs-beds}
\begin{tabular}{p{3.5cm}p{5cm}p{5cm}}
\toprule
\textbf{Aspect} & \textbf{Liang et al.~\cite{liang_2019}} & \textbf{BEDS (this work)} \\
\midrule
Object of study & Norm $\|\theta\|_{\text{fr}}^2$ & Distance $\dF^2(\theta, \theta^*)$ \\
Question addressed & ``What properties does FR have?'' & ``Why is FR optimal?'' \\
Methodological stance & Descriptive (capacity measure) & Prescriptive (regularization target) \\
Main result & FR dominates other norms & FR is the \emph{unique} optimal regularization \\
Assumptions & Implicit (ReLU architectures) & Explicit (A1, A2, A3) \\
Reference point & Implicit $\theta^* = 0$ & Explicit prior $\theta^*$ \\
Temporal dynamics & Not addressed & Coherence $\kappa$ on $\MvM$ \\
\bottomrule
\end{tabular}
\end{table}

\textbf{Key distinction:} Liang et al.\ use Fisher--Rao as a \emph{post hoc diagnostic}---a way to characterize trained networks. BEDS instead proposes Fisher--Rao distance as a \emph{prescriptive objective}---the quantity to minimize during training. Furthermore, we derive this prescription from thermodynamic first principles and extend the state space to include temporal coherence.

\subsubsection{Other Complexity Measures}

The quest for appropriate complexity measures in deep learning has produced numerous proposals:

\begin{itemize}
    \item \textbf{VC dimension and Rademacher complexity}: Classical measures that grow with network size but fail to explain generalization in overparameterized regimes~\cite{golowich_2018}.

    \item \textbf{PAC-Bayes bounds}: Provide tighter generalization guarantees by incorporating prior knowledge~\cite{mcallester_1999,dziugaite_roy_2017}.

    \item \textbf{Compression-based measures}: Arora et al.~\cite{arora_2018} showed that compressibility correlates with generalization.
\end{itemize}

Hu et al.~\cite{hu_2021} provide a comprehensive survey distinguishing \emph{expressive capacity} (what functions a model can represent) from \emph{effective complexity} (what functions it actually learns). BEDS complements this taxonomic view with a generative principle: under A1--A3, the appropriate complexity measure is Fisher--Rao distance to a reference state.

\subsubsection{Natural Gradient and Optimization}

Amari~\cite{amari_1998} introduced \emph{natural gradient descent}, showing that optimization in Fisher--Rao geometry converges faster than Euclidean gradient descent:
\begin{equation}
\theta_{t+1} = \theta_t - \eta \, I(\theta_t)^{-1} \nabla \mathcal{L}(\theta_t)
\end{equation}
This is the steepest descent direction when distance is measured by Fisher--Rao rather than Euclidean metric.

Practical implementations include K-FAC~\cite{martens_grosse_2015}, which approximates the Fisher information matrix using Kronecker factorization, and Shampoo~\cite{gupta_2018}, which uses tensor decomposition. Martens~\cite{martens_2020} unified these approaches, showing they all approximate the natural gradient.

\textbf{Connection to BEDS:} These methods use Fisher geometry for \emph{optimization direction}, not as a \emph{regularization target}. Under BEDS, natural gradient is the dynamics of minimal dissipation---the trajectory that least disturbs the belief state.

\subsubsection{The Flat Minima Debate}

Hochreiter and Schmidhuber~\cite{hochreiter_schmidhuber_1997_flat} hypothesized that ``flat'' minima of the loss landscape generalize better than ``sharp'' ones. However, Dinh et al.~\cite{dinh_2017} demonstrated a fundamental problem: \emph{flatness is sensitive to reparametrization}. By rescaling weights between layers, one can make the Hessian arbitrarily large without changing the function computed.

Their key observation: ``Flatness is sensitive to network parametrization whereas generalization should be invariant.''

\textbf{BEDS resolution:} This paradox motivated our Assumption A1. Fisher--Rao distance is \emph{intrinsically invariant}---it depends only on the function $f_\theta$, not on the particular parametrization. If $f_{\theta_1} = f_{\theta_2}$, then $\dF(\theta_1, \theta^*) = \dF(\theta_2, \theta^*)$. Thus BEDS provides a notion of ``flatness'' that is mathematically well-defined.

\subsection{Thermodynamics of Information and Computation}
\label{subsec:thermodynamics}

\subsubsection{Landauer's Principle and the Cost of Erasure}

Landauer~\cite{landauer_1961} established the fundamental link between information and thermodynamics: \emph{erasing one bit of information requires dissipating at least $\kB T \ln 2$ joules of energy}. At room temperature (300K), this is approximately $2.87 \times 10^{-21}$ J/bit.

Bennett~\cite{bennett_1982} clarified that computation itself can be thermodynamically reversible---only the erasure of information is irreducibly irreversible. This distinction is crucial: forward inference in a neural network can in principle be reversible, but regularization (which discards information about the precise parameter values) cannot.

B\'erut et al.~\cite{berut_2012} experimentally verified Landauer's bound using a colloidal particle in a double-well potential, confirming the physical reality of information-energy equivalence.

\textbf{Connection to BEDS:} The Energy-Precision Bound (Corollary 4.4) is a direct application of Landauer's principle:
\begin{equation}
P_{\min} = \kB T \ln 2 \cdot D_{\KL}(q \| q^*)
\end{equation}
Regularization erases information (moving from a sharp posterior toward a broader prior), and this erasure has an unavoidable thermodynamic cost.

\subsubsection{Dissipative Structures}

Prigogine~\cite{prigogine_1977_nobel,prigogine_stengers_1984} developed the theory of \emph{dissipative structures}: ordered patterns that exist far from thermodynamic equilibrium and are maintained by continuous energy dissipation. The canonical example is B\'enard convection---when a fluid layer is heated from below beyond a critical threshold, it spontaneously organizes into hexagonal convection cells.

Four conditions characterize dissipative structures:
\begin{enumerate}
    \item \textbf{Openness}: Exchange of energy/matter with environment
    \item \textbf{Far from equilibrium}: Maintained state differs from equilibrium
    \item \textbf{Nonlinearity}: Feedback mechanisms
    \item \textbf{Dissipation}: Continuous entropy export
\end{enumerate}

Schr\"odinger~\cite{schrodinger_1944} anticipated these ideas in \emph{What is Life?}, proposing that living systems maintain order by ``feeding on negative entropy''---exporting disorder to maintain internal structure.

\textbf{Connection to BEDS:} Neural networks during training satisfy all four conditions: (1) data streams provide flux, (2) trained weights differ from random initialization, (3) nonlinear activations and attention provide feedback, (4) regularization exports entropy. BEDS makes this analogy precise and quantitative.

\subsubsection{Maximum Entropy Inference}

Jaynes~\cite{jaynes_1957} established maximum entropy as a principle for rational inference: given constraints, one should choose the probability distribution that maximizes entropy subject to those constraints. This is not arbitrary---it is the unique distribution that makes no assumptions beyond the given constraints.

For neural network representations constrained only to have finite variance, the maximum entropy distribution is Gaussian. This justifies Assumption A2: the reference state $\theta^* \sim \Normal(0, I)$ is the least presumptuous prior.

\subsection{Bridging Machine Learning and Thermodynamics}
\label{subsec:bridging}

\subsubsection{Historical Connections}

The connection between statistical mechanics and machine learning is not new. Hopfield~\cite{hopfield_1982} introduced associative memories with energy functions, and Hinton and Sejnowski~\cite{hinton_sejnowski_1983,ackley_hinton_sejnowski_1985} developed Boltzmann machines that explicitly frame learning as sampling from thermal distributions.

More recently, Sohl-Dickstein et al.~\cite{sohl_dickstein_2015} developed diffusion models by drawing on non-equilibrium thermodynamics, showing that iterative denoising can be understood as reversing a diffusion process. This perspective has proven remarkably fruitful, leading to state-of-the-art generative models.

\subsubsection{Recent Thermodynamic Bounds}

Bhattacharyya et al.~\cite{bhattacharyya_2024}, in ``Thermodynamic bounds on energy use in quasi-static Deep Neural Networks,'' map feedforward architectures onto physical Hamiltonians and derive bounds on training energy. They show that inference can be thermodynamically reversible while training is irreducibly dissipative.

\textbf{Distinction from BEDS:} Their work focuses on \emph{physical implementation}---the actual energy consumed by hardware. BEDS instead develops a \emph{conceptual framework} where thermodynamic language guides algorithm design. The two perspectives are complementary: Bhattacharyya et al.\ ask ``How much energy does this computation cost?'' while BEDS asks ``What regularization strategy is thermodynamically optimal?''

Whitelam~\cite{whitelam_2024} explores generative thermodynamic computing, using Langevin dynamics and thermal fluctuations for generation. This provides another perspective on thermodynamically-aware computation.

\subsubsection{The Free Energy Principle}

Friston~\cite{friston_2010} proposed the Free Energy Principle in neuroscience: biological systems minimize variational free energy, connecting Bayesian inference to thermodynamics. While influential, this framework has not been systematically applied to derive regularization strategies for machine learning.

BEDS can be seen as making the Friston connection explicit and actionable for artificial neural networks: the variational free energy decomposes into data fit (accuracy) and KL divergence from prior (complexity), exactly paralleling our fundamental equation.

\paragraph{Relation to Optimal Transport.}
BEDS is closely related to unbalanced optimal transport geometries such as
Wasserstein--Fisher--Rao, which combine mass transport with local creation and
destruction. While these frameworks focus on optimal paths between fixed distributions, BEDS instead studies viable, dissipative belief trajectories driven by continual learning and observation, with optimal transport arising only as a limiting case.

\subsubsection{Self-Supervised Learning Methods}

Modern self-supervised learning provides empirical motivation for BEDS. The evolution of these methods reveals a convergence toward thermodynamically sensible designs:

\begin{itemize}
    \item \textbf{SimCLR}~\cite{chen_2020_simclr}: Contrastive learning with large batches
    \item \textbf{BYOL}~\cite{grill_2020}: No negatives, uses EMA teacher
    \item \textbf{DINO}~\cite{caron_2021}: EMA + self-distillation + centering
    \item \textbf{VICReg}~\cite{bardes_2022}: Explicit variance-invariance-covariance regularization
    \item \textbf{SIGReg}~\cite{garrido_2023}: Regularization toward $\Normal(0, I)$
    \item \textbf{I-JEPA}~\cite{assran_2023}: Masked prediction in latent space
\end{itemize}

BEDS provides a unified interpretation:
\begin{itemize}
    \item SIGReg directly implements Fisher--Rao regularization toward the maximum entropy distribution
    \item EMA in BYOL/DINO controls temporal coherence: $\kappa \approx 1/(1-m)$ where $m$ is the momentum
    \item Centering in DINO exports entropy, preventing mode collapse
    \item Masking in I-JEPA creates information flux driving learning
\end{itemize}

These methods, developed through empirical trial-and-error, are revealed as approximations to thermodynamically optimal regularization.

\subsubsection{Reinforcement Learning and Entropy Regularization}

Haarnoja et al.~\cite{haarnoja_2018} introduced Soft Actor-Critic (SAC), adding entropy regularization to the RL objective:
\begin{equation}
J(\pi) = \E\left[\sum_t r_t + \alpha \mathcal{H}(\pi(\cdot|s_t))\right]
\end{equation}
The temperature $\alpha$ (RL temperature, distinct from $\alpha(t)$ in GNC, Section~\ref{subsec:gnc}) controls the exploration-exploitation tradeoff.

BEDS interprets SAC within our framework: the entropy bonus corresponds to low temporal coherence ($\kappa = 1/\alpha$), maintaining the system in an exploratory regime where phase $\phi$ is dispersed rather than locked.

\subsection{What BEDS Contributes Beyond Prior Work}
\label{subsec:novel-contributions}

To clarify precisely what is novel, we summarize the contributions that did not exist before this work:

\begin{table}[htbp]
\centering
\caption{\textbf{Novel contributions of BEDS.} Each row identifies a specific result and its epistemic status.}
\label{tab:novel-contributions}
\begin{tabular}{p{5cm}p{3cm}p{5.5cm}}
\toprule
\textbf{Contribution} & \textbf{Status} & \textbf{What prior work provided} \\
\midrule
Conditional Optimality Theorem (Thm.~4.3) & \proven & FR dominance (Liang), but not uniqueness \\
Energy-Precision Bound (Cor.~4.4) & \proven & Landauer bound, but not for ML \\
Euclidean suboptimality (Cor.~4.5) & \proven & Empirical observation, not formal proof \\
$\HH^2 \times \MvM$ state space & \proven & $\HH^2$ only (Amari) \\
Temporal coherence $\kappa$ & \textsc{Novel} & EMA as heuristic (Grill et al.) \\
Crystallization index $\mathcal{C} = \tau \cdot \kappa$ & \textsc{Novel} & None \\
Six-class problem taxonomy (Prop.~\ref{prop:taxonomy}) & \proven & Ad hoc problem classification \\
Unified equation $\mathcal{L} = \dF^2 + \lambda \mathcal{L}_{\text{data}}$ & \textsc{Novel} & Separate loss functions \\
\bottomrule
\end{tabular}
\end{table}

The central novelty is the \emph{synthesis}: combining \v{C}encov's uniqueness theorem (A1), maximum entropy inference (A2), and quasi-static optimality (A3) to derive a single equation that unifies regularization across supervised, self-supervised, and reinforcement learning paradigms.

\begin{keyresult}[Summary]
BEDS does not claim to discover Fisher--Rao geometry---this is classical (\v{C}encov, Rao, Amari). It does not claim to discover thermodynamic costs of computation---this is Landauer. What BEDS contributes is:
\begin{enumerate}
    \item The \emph{derivation} of Fisher--Rao regularization as the unique optimal strategy under explicit assumptions
    \item The \emph{extension} to temporal dynamics via the von Mises component
    \item The \emph{unification} of disparate heuristics into a single thermodynamic framework
\end{enumerate}
\end{keyresult}


\begin{table}[htbp]
\centering
\caption{\textbf{BEDS Interpretation of Machine Learning Concepts.}
Classical regularization and stabilization techniques find unified interpretation
within the thermodynamic framework. This ``Rosetta Stone'' enables practitioners to
translate familiar ML phenomena into the BEDS language.}
\label{tab:beds-rosetta}
\begin{tabular}{p{3.2cm}p{4.2cm}p{5.8cm}}
\toprule
\textbf{ML Concept} & \textbf{Classical View} & \textbf{BEDS Interpretation} \\
\midrule
\multicolumn{3}{l}{\textit{Regularization Techniques}} \\[2pt]
Weight decay & Prevents large weights & Dissipation toward prior ($\gamma > 0$) \\
Dropout & Prevents co-adaptation & Stochastic entropy injection \\
Batch normalization & Stabilizes activations & Entropy export via standardization \\
Early stopping & Prevents overfitting & Halts before over-crystallization ($\tau \to \infty$) \\
\midrule
\multicolumn{3}{l}{\textit{Training Pathologies}} \\[2pt]
Overfitting & Memorization of noise & Over-crystallization ($\tau \to \infty$ prematurely) \\
Mode collapse (GANs) & Generator ignores diversity & Premature crystallization ($\kappa \to \infty$) \\
Catastrophic forgetting & New learning erases old & Insufficient dissipation structure in prior \\
Representation collapse & All outputs identical & System falls to center (maximum entropy state) \\
\midrule
\multicolumn{3}{l}{\textit{Stabilization Techniques}} \\[2pt]
EMA (BYOL, DINO) & Stabilizes teacher network & Temporal coherence control ($\kappa$ regulation) \\
Entropy bonus (SAC) & Encourages exploration & Maintains low $\kappa$ in policy space \\
Learning rate decay & Smaller updates over time & Decreasing dissipation rate $\gamma(t) \to 0$ \\
Curriculum learning & Easy to hard examples & GNC schedule $\alpha(t)$ (Section~\ref{subsec:gnc}) \\
Knowledge distillation & Teacher guides student & Information transfer via $\dF^2$ minimization \\
Data augmentation & Artificial variation & Increases effective temperature $T$ \\
\bottomrule
\end{tabular}
\end{table}

\noindent


\section{How to Use This Framework}
\label{sec:how-to-use}

This section provides practical guidance for researchers seeking to apply
the BEDS perspective to their work. It is not a tutorial or implementation
guide, but rather a \emph{conceptual toolkit} for reasoning about learning systems. Table~\ref{tab:beds-rosetta} serves as a quick reference for researchers familiar with machine learning but new to the BEDS framework. Each row provides an immediate translation: when you observe a phenomenon or apply a technique, the BEDS column tells you which thermodynamic quantity is at play.

\subsection{Interpreting Learning Dynamics}

The BEDS framework offers new language for familiar phenomena:

\begin{itemize}
    \item \textbf{To interpret instability}: When training becomes unstable
    (loss spikes, gradient explosions), ask: Is dissipation ($\gamma$) too high
    relative to information influx? Is the system approaching collapse
    (trajectory toward center in Figure~\ref{fig:beds-state-space-canonical})?

    \item \textbf{To diagnose overfitting}: Overfitting is over-crystallization.
    The crystallization index $\mathcal{C} = \tau \cdot \kappa$ growing without
    bound signals dangerous rigidity. Regularization increases $\gamma$,
    preventing excessive crystallization.

    \item \textbf{To understand forgetting}: Catastrophic forgetting occurs when
    new learning overwrites crystallized structure. BEDS suggests maintaining
    appropriate dissipation to prevent over-rigid priors while preserving
    essential structure.

    \item \textbf{To compare methods}: Different regularization techniques
    (dropout, weight decay, EMA) implement different dissipation mechanisms.
    Table~\ref{tab:beds-rosetta} provides a quick reference.
\end{itemize}

\subsection{Classifying Your Problem}

Before choosing an algorithm, classify your problem using the BEDS taxonomy
(detailed in Section~\ref{sec:problem-taxonomy}):

\begin{center}
\begin{tabular}{lp{9cm}}
\toprule
\textbf{Question} & \textbf{BEDS Classification} \\
\midrule
Is your data distribution fixed? & Yes $\to$ BEDS-crystallizable ($\tau$-crystallizable) \\
Must your model track a changing target? & Yes $\to$ BEDS-maintainable ($\tau$-maintainable) \\
Do you need exploration (RL)? & Yes $\to$ BEDS-maintainable ($\kappa$-maintainable) \\
Is convergence the goal? & Yes $\to$ BEDS-crystallizable (full) \\
\bottomrule
\end{tabular}
\end{center}

\textbf{C-full} problems (supervised learning, converged SSL) benefit from
strong regularization toward equilibrium. \textbf{M-full} problems (continual RL,
online learning) require careful entropy management to prevent premature convergence.

\subsection{Designing Diagnostics}

The BEDS framework suggests specific quantities to monitor during training:

\begin{enumerate}
    \item \textbf{Crystallization index} $\mathcal{C} = \tau \cdot \kappa$:
    Track over training epochs. Unbounded growth signals over-crystallization;
    sudden drops may indicate collapse.

    \item \textbf{Effective dissipation} $\gamma_{\text{eff}}$:
    Estimate from weight decay coefficient, dropout rate, and learning rate.
    Should decrease as training progresses for BEDS-crystallizable problems.

    \item \textbf{Fisher--Rao step size} $\dF^2(\theta_t, \theta_{t-1})$:
    Measures information-geometric distance per update. Large values indicate
    non-quasi-static dynamics (potential inefficiency or instability).
\end{enumerate}

\begin{insight}[The BEDS Practitioner's Heuristic]
\textbf{If your model overfits}: increase $\gamma$ (more regularization, higher dropout, stronger weight decay).\\[3pt]
\textbf{If your model underfits}: decrease $\gamma$ or increase data flux (more data, longer training).\\[3pt]
\textbf{If your model collapses}: you have crossed a phase boundary---reduce $\gamma$ dramatically, restructure the loss, or add diversity-promoting terms.
\end{insight}

\vspace{0.5cm}

\vspace{0.5cm}
\hrule
\vspace{0.3cm}
\noindent\textit{Part II develops the mathematical foundations rigorously.
Having established the physical and conceptual motivation, we now make
our assumptions explicit and derive their consequences with full proofs.}
\vspace{0.3cm}
\hrule
\vspace{0.5cm}


\FloatBarrier
\part{Theoretical Foundations}

\section{Explicit Assumptions}
\label{sec:assumptions}

The conditional nature of this work requires that we state our foundational
assumptions with precision. Unlike empirical findings that stand or fall with
data, the results derived in subsequent sections are logically conditional:
they hold if and only if the assumptions presented here are accepted. This
transparency is essential for scientific rigor---it allows readers to evaluate
the framework's applicability to their specific contexts and to identify which
assumptions might be relaxed in future work.

We make explicit the assumptions underlying our framework. These are not theorems; they are modeling choices. Our results are conditional upon their acceptance.

\begin{assumption}[A1: Intrinsic Measure of Information]
\label{ass:intrinsic}
Thermodynamic optimality requires an intrinsic, parametrization-invariant measure of information. Fundamental quantities such as minimal energy cost must not depend on arbitrary choices of coordinates or parameterizations.
\end{assumption}

\textbf{Motivation}: Physical laws do not depend on coordinate choices. If the ``information content'' of a belief state changed under reparametrization, the thermodynamic cost would be coordinate-dependent---physically meaningless.

\textbf{Consequence}: By \v{C}encov's theorem, the Fisher--Rao metric is the unique (up to scale) Riemannian metric satisfying this requirement.

\begin{assumption}[A2: Maximum Entropy Belief States]
\label{ass:maxent}
Belief states of a learning system are modeled by maximum-entropy distributions subject to known constraints.
\end{assumption}

\textbf{Motivation}: The maximum-entropy principle (Jaynes) asserts that among all distributions satisfying given constraints, one should choose the distribution with maximal entropy---it makes the fewest additional assumptions.

\textbf{Consequence}: For constraints on mean and variance, this yields Gaussian distributions. For constraints on circular mean, this yields von Mises distributions. The space of belief states inherits geometric structure from these families.

\begin{assumption}[A3: Quasi-Static Processes]
\label{ass:quasistatic}
Thermodynamic optimality corresponds to quasi-static (infinitesimally slow) processes. Real learning dynamics may deviate from this idealization, but the quasi-static regime defines the fundamental lower bound on dissipation.
\end{assumption}

\textbf{Motivation}: In thermodynamics, reversible processes---those proceeding infinitely slowly through equilibrium states---achieve minimum entropy production. Finite-rate processes necessarily dissipate more.

\textbf{Consequence}: The optimal learning trajectory is a geodesic in belief space. Deviations from geodesics incur excess dissipation.

\begin{keyresult}[Epistemic Status]
\begin{itemize}
    \item \textbf{A1--A3}: Explicit modeling assumptions (not theorems)
    \item \textbf{Main Theorem}: Proven conditional on A1--A3
    \item \textbf{Corollaries}: Logically follow from the theorem
    \item \textbf{Extensions}: Conjectured or speculative (marked explicitly)
\end{itemize}
\end{keyresult}

\subsubsection{Discrete and Continuous Descriptions of Learning Systems}

Learning algorithms are ultimately implemented on discrete computational substrates.
Model parameters are stored as finite-precision bit strings, and updates correspond to
discrete memory operations.
At this level, the natural cost of change is combinatorial and additive, and distances
such as Hamming or $\ell_1$ metrics provide an appropriate description of computational effort.

Most learning theory, however, operates under a continuous approximation, in which
parameters are modeled as elements of $\mathbb{R}^n$, updates are treated as infinitesimal,
and loss landscapes are assumed differentiable.
Within this approximation, Euclidean ($\ell_2$) geometry arises naturally as a smooth surrogate
for discrete update costs, enabling gradient-based optimization and local quadratic analysis.

The BEDS framework adopts a further level of abstraction, modeling learning states as
probability distributions and measuring change in terms of information divergence.
In this setting, Kullback--Leibler divergence quantifies information change, and the
Fisher--Rao metric emerges as its infinitesimal limit, defining an intrinsic geometry
on the space of belief states.

These descriptions correspond to distinct but compatible levels of analysis.
The discrete computational level captures physical implementation costs,
the continuous approximation enables algorithmic analysis,
and the information-geometric level characterizes idealized bounds on information dissipation.
The results of this work should therefore be interpreted as statements about theoretical
limits and reference geometries, rather than prescriptions for discrete hardware-level
optimality.


\section{The Conditional Optimality Theorem}
\label{sec:main-theorem}

\subsection{Background: Fisher--Rao Geometry}
\label{subsec:fisher-background}

The proof of our main theorem relies on classical results from information
geometry. We briefly recall the key definitions and propositions that
will be invoked in the argument.

\begin{definition}[Fisher Information Matrix]
\label{def:fisher}
For a parametric family $p(x|\theta)$, the Fisher information matrix is:
\begin{equation}
\mathcal{I}(\theta)_{ij} = \E_{p(x|\theta)}\left[\frac{\partial \log p(x|\theta)}{\partial \theta_i} \frac{\partial \log p(x|\theta)}{\partial \theta_j}\right]
\label{eq:fisher-matrix}
\end{equation}
\end{definition}

\begin{proposition}[\v{C}encov's Theorem, 1982]
\label{prop:cencov}
Up to a positive scalar multiple, the Fisher--Rao metric is the unique Riemannian metric on the space of probability distributions that is invariant under sufficient statistics and Markov morphisms.
\end{proposition}

This is the mathematical content of Assumption A1: requiring coordinate-independence uniquely determines the metric.

\begin{proposition}[Local KL-Fisher Correspondence]
\label{prop:kl-fisher}
For nearby distributions $p$ and $p + dp$:
\begin{equation}
D_{\KL}(p \| p + dp) = \frac{1}{2} \dF^2(p, p+dp) + O(\|dp\|^3)
\end{equation}
Fisher--Rao distance locally measures information divergence.
\end{proposition}

\subsection{The Main Theorem}
\label{subsec:main-theorem}

We now state the central result of this work. The theorem establishes that,
under assumptions A1--A3, there is no freedom in choosing a regularization
strategy: Fisher--Rao regularization is not merely one option among many, but
the \emph{unique} optimal choice dictated by the geometric and thermodynamic
constraints we have assumed. This uniqueness is the key contribution---it
transforms regularization from an empirical tuning problem into a principled
design decision.

\begin{theorem}[Conditional Optimality of Fisher--Rao Regularization]
\label{thm:conditional-optimality}
Under Assumptions A1--A3, thermodynamically optimal regularization is uniquely characterized by minimizing squared Fisher--Rao distance to a reference belief state. Specifically:
\begin{enumerate}
    \item The metric on belief space must be the Fisher--Rao metric.
    \item The optimal regularization functional is $\mathcal{L}_{\mathrm{reg}} = \dF^2(q, q^*)$.
    \item The optimal learning trajectory is a geodesic in Fisher--Rao geometry.
\end{enumerate}
\end{theorem}

\begin{proof}
We prove this in four steps.

\textbf{Step 1 (Uniqueness of Fisher--Rao):} By Assumption A1, the metric measuring information change must be intrinsic and invariant under reparametrization. By \v{C}encov's theorem (Proposition~\ref{prop:cencov}), the Fisher--Rao metric is the unique Riemannian metric with this property. Therefore, the geometry of belief space is Fisher--Rao.

\textbf{Step 2 (Geometry of belief space):} By Assumption A2, belief states are maximum-entropy distributions under constraints. For fixed mean and variance, this yields Gaussian distributions; other constraints yield other exponential families. In all cases, the intrinsic geometry is induced by the Fisher--Rao metric on these families.

\textbf{Step 3 (Quasi-static optimality):} By Assumption A3, thermodynamic optimality corresponds to quasi-static processes. The infinitesimal information erased along a path $\gamma$ in belief space is proportional to the Fisher--Rao arc length:
\begin{equation}
I_{\text{erased}} = \int_{\gamma} ds_F
\end{equation}
This integral is minimized by geodesics. Therefore, the optimal regularization path---the trajectory from current belief to target belief with minimum information erasure---is a geodesic in Fisher--Rao geometry.

\textbf{Step 4 (Regularization functional):} For a regularization functional to be consistent with geodesic optimality, it must penalize deviation from the target measured in Fisher--Rao distance. The squared distance $\dF^2(q, q^*)$ is the natural choice (analogous to squared Euclidean distance in flat space), giving:
\begin{equation}
\mathcal{L}_{\text{reg}} = \dF^2(q, q^*)
\end{equation}

Combining these steps completes the proof.
\end{proof}

\subsection{Energy-Precision Bound (Corollary)}
\label{subsec:energy-precision-corollary}

The Conditional Optimality Theorem has an immediate consequence that
connects information geometry to thermodynamics. The Energy-Precision bound,
which was the central theorem in v2, now follows as a corollary.

\begin{corollary}[Energy-Precision Bound]
\label{cor:energy-precision}
The minimal energy dissipated during regularization from belief $q$ to reference $q^*$ is:
\begin{equation}
E_{\min} = \kB T \ln 2 \cdot D_{\KL}(q \| q^*)
\end{equation}
For a system maintaining precision $\tau$ against dissipation rate $\gamma$ in the quasi-static limit:
\begin{equation}
\boxed{P_{\min} \geq \frac{\gamma \kB T}{2}}
\label{eq:power-bound}
\end{equation}
\end{corollary}

\begin{proof}
By the Landauer principle, erasing one bit of information requires dissipating at least $\kB T \ln 2$ joules. By Proposition~\ref{prop:kl-fisher}, the information erased along a geodesic equals $D_{\KL}(q \| q^*)$ (in the local limit). Therefore:
\begin{equation}
E_{\min} = \kB T \ln 2 \cdot D_{\KL}(q \| q^*)
\end{equation}
For continuous maintenance against dissipation rate $\gamma$, rate-distortion theory gives:
\begin{equation}
\dot{I}_{\min} = \frac{\gamma \tau^*}{2 \ln 2} \quad \text{bits/time}
\end{equation}
Multiplying by the Landauer cost per bit:
\begin{equation}
P_{\min} = \dot{I}_{\min} \cdot \kB T \ln 2 = \frac{\gamma \tau^* \kB T}{2}
\end{equation}
For unit precision ($\tau^* = 1$), this gives the stated bound.
\end{proof}

\subsection{Structural Suboptimality of Euclidean Regularization}
\label{subsec:euclidean-suboptimality}

Having established that Fisher--Rao regularization is optimal, we can
now quantify the inefficiency of the more common Euclidean alternative.

\begin{corollary}[Euclidean Regularization is Suboptimal]
\label{cor:euclidean-suboptimal}
Standard regularization methods minimizing Euclidean distance in parameter space:
\begin{equation}
\mathcal{L}_{\text{ridge}} = \|\theta - \theta^*\|_2^2 + \lambda \mathcal{L}_{\text{data}}
\end{equation}
violate Assumption A1 and cannot guarantee thermodynamic optimality.
\end{corollary}

\begin{proof}
The Euclidean metric is not invariant under reparametrization. A change of coordinates $\theta \mapsto f(\theta)$ changes distances. This violates A1.
\end{proof}

\begin{proposition}[Quantifying Euclidean Suboptimality]
\label{prop:euclidean-ratio}
For Gaussian beliefs with precision $\tau$, the ratio between Euclidean and Fisher--Rao squared distances is:
\begin{equation}
\frac{d_{\text{Euclid}}^2}{d_F^2} = \frac{1}{\tau} = \sigma^2
\end{equation}
This ratio can be arbitrarily large (low precision) or small (high precision), demonstrating that Euclidean regularization can be arbitrarily suboptimal.
\end{proposition}

\begin{insight}[Why This Matters]
Ridge regression uses Euclidean distance. Under our assumptions, this is suboptimal---not because ridge ``doesn't work,'' but because it doesn't respect the intrinsic geometry of belief space. The suboptimality is most severe when precision varies significantly across parameters.
\end{insight}



\section{The BEDS Framework}
\label{sec:beds-framework}

Having established Fisher--Rao as the optimal geometry, we now develop a practical parameterization for belief states.

\subsection{From Assumptions to BEDS}
\label{subsec:assumptions-to-beds}

We now trace the logical path from our three assumptions to the concrete
BEDS parameterization. The path from Assumptions A1--A3 to the BEDS parameterization is:
\begin{enumerate}
    \item \textbf{A1} $\to$ \textbf{\v{C}encov} $\to$ Fisher--Rao metric is unique
    \item \textbf{A2} $\to$ \textbf{Jaynes} $\to$ Belief states are Gaussian (spatial) / von Mises (temporal)
    \item \textbf{Fisher--Rao on Gaussians} $\to$ Hyperbolic geometry $\HH^2$
    \item \textbf{Fisher--Rao on von Mises} $\to$ von Mises manifold $\MvM$
    \item \textbf{Product structure} $\to$ BEDS state space $\HH^2 \times \MvM$
\end{enumerate}

\subsection{The BEDS State Space}
\label{subsec:beds-state}

The geometric structure derived from assumptions A1--A3 naturally suggests a
parameterization of belief states. Rather than working with arbitrary probability
distributions, we focus on a minimal representation that captures the essential
degrees of freedom for learning dynamics. This parameterization balances
expressiveness with tractability: it is rich enough to describe the phenomena
of interest (crystallization, dissipation, coherence) while remaining
computationally manageable.

\begin{definition}[BEDS State]
\label{def:beds-state}
A \textbf{BEDS state} is a quadruple $(\mu, \tau, \phi, \kappa)$ where $\mu \in \R^d$ denotes the \textbf{position} (mean belief), $\tau \in \R^+$ the \textbf{precision} (inverse variance, certainty), $\phi \in [0, 2\pi)^n$ the \textbf{phase} (temporal position), and $\kappa \in \R^+_0$ the \textbf{coherence} (synchronization strength).
\end{definition}

The interpretation is geometric: the pair $(\mu, \tau)$ describes \emph{what} the system believes and \emph{how confident} it is, forming the spatial component on $\HH^2$. The pair $(\phi, \kappa)$ describes \emph{when} features align and \emph{how synchronized} they are, forming the temporal component on $\MvM$.

\begin{table}[htbp]
\centering
\caption{\textbf{The four BEDS parameters and their geometric origins.}}
\label{tab:beds-parameters}
\begin{tabular}{lllll}
\toprule
\textbf{Parameter} & \textbf{Symbol} & \textbf{Meaning} & \textbf{Space} & \textbf{Origin} \\
\midrule
Position & $\mu$ & Belief content & $\R^d$ & Gaussian mean \\
Precision & $\tau = 1/\sigma^2$ & Certainty & $\R^+$ & Gaussian precision \\
Phase & $\phi$ & Temporal alignment & $S^1$ & von Mises mean \\
Coherence & $\kappa$ & Synchronization & $\R^+$ & von Mises concentration \\
\bottomrule
\end{tabular}
\end{table}

\subsection{The Product Manifold $\HH^2 \times \MvM$}
\label{subsec:product-manifold}

The spatial and temporal components combine into a product structure
that defines the complete belief space.

\begin{proposition}[BEDS Manifold]
\label{prop:beds-manifold}
The full BEDS state space is the product manifold:
\begin{equation}
\bedsM = \HH^2 \times \MvM
\end{equation}
with metric:
\begin{equation}
ds^2_{\text{BEDS}} = \underbrace{\tau \, d\mu^2 + \frac{1}{2\tau^2} d\tau^2}_{\text{spatial: hyperbolic}} + \underbrace{A(\kappa) \, d\phi^2 + B(\kappa) \, d\kappa^2}_{\text{temporal: von Mises}}
\label{eq:beds-metric}
\end{equation}
where $A(\kappa) = \kappa I_1(\kappa)/I_0(\kappa)$ and $B(\kappa)$ is the appropriate Fisher coefficient for $\kappa$.
\end{proposition}

\proven

The hyperbolic half-plane $\HH^2$ has constant negative curvature $K = -1$. Its geodesics are semicircles perpendicular to the $\mu$-axis, and distances expand exponentially---the hallmark of hyperbolic geometry.

\subsection{Dissipation Dynamics}
\label{subsec:dissipation}

In the absence of observations, belief states do not remain static---they
evolve toward maximum entropy according to characteristic dissipative
dynamics.

\begin{definition}[Dissipative Dynamics]
\label{def:dissipation}
Under dissipation rate $\gamma > 0$:
\begin{align}
\frac{d\tau}{dt} &= -2\gamma \tau \quad \text{(precision decays)} \\
\frac{d\kappa}{dt} &= -\gamma_\kappa \kappa \quad \text{(coherence decays)}
\end{align}
\end{definition}

\textbf{Interpretation}: Without information influx, uncertainty increases exponentially. This is the ``forgetting'' inherent in any learning system.

\begin{definition}[Crystallization Index]
\label{def:crystallization}
\begin{equation}
\boxed{\mathcal{C} = \tau \cdot \kappa}
\end{equation}
This product measures the overall ``solidity'' of the learned state. Values $\mathcal{C} \ll 1$ indicate a fluid state favoring exploration. Values $\mathcal{C} \approx 1$ mark the transition phase of active learning. Values $\mathcal{C} \gg 1$ indicate crystallization---stable but potentially rigid.
\end{definition}

\begin{keyresult}[BEDS in Three Sentences]
\begin{enumerate}
    \item \textbf{Under A1--A3, optimal regularization uses Fisher--Rao distance.} This is the unique geometry respecting intrinsic information measures.
    \item \textbf{Four parameters $(\mu, \tau, \phi, \kappa)$ on $\HH^2 \times \MvM$ capture belief states.} This parameterization arises from maximum-entropy distributions.
    \item \textbf{One equation unifies all methods:} $\mathcal{L} = \dF^2(\theta, \theta^*) + \lambda \cdot \mathcal{L}_{\text{data}}$
\end{enumerate}
\end{keyresult}

\vspace{0.5cm}
\hrule
\vspace{0.3cm}
\noindent\textit{With the theoretical foundations established, Part III
demonstrates the practical power of the framework. We show how existing
methods emerge as special cases, develop efficiency measures, and extend
BEDS to hierarchical and multi-agent systems.}
\vspace{0.3cm}
\hrule
\vspace{0.5cm}


\FloatBarrier
\part{Unification and Applications}


\section{One Equation Unifies All}
\label{sec:unification}

\subsection{The Fundamental Equation}
\label{subsec:fundamental-equation}

The preceding sections have established the geometric foundations of belief
space under assumptions A1--A3. We now present the central equation that
emerges from this analysis---a single formula that unifies all regularized
learning within the BEDS framework. The remarkable aspect of this equation
is its generality: despite being derived from abstract principles, it
encompasses as special cases nearly all regularization techniques in
common use today.

\begin{keyresult}[Fundamental BEDS Equation]
Under Assumptions A1--A3, all thermodynamically optimal regularized learning takes the form:
\begin{equation}
\boxed{\bedsL = \dF^2(\theta, \theta^*) + \lambda \cdot \mathcal{L}_{\text{data}}}
\label{eq:beds-fundamental}
\end{equation}
where $\dF^2(\theta, \theta^*)$ denotes the squared Fisher--Rao distance to the target state, $\mathcal{L}_{\text{data}}$ the data-fitting term, and $\lambda > 0$ the balance parameter between regularization and data fidelity.
\end{keyresult}

\proven

In BEDS coordinates, this decomposes as:
\begin{equation}
\bedsL = \underbrace{\tau(\mu - \mu^*)^2 + \frac{(\tau - \tau^*)^2}{2\tau\tau^*}}_{\text{spatial}} + \underbrace{\kappa(1 - \cos(\phi - \phi^*)) + \frac{(\kappa - \kappa^*)^2}{2\kappa\kappa^*}}_{\text{temporal}} + \lambda \mathcal{L}_{\text{data}}
\end{equation}

\subsection{Existing Methods as Special Cases}
\label{subsec:special-cases}

The unifying power of the BEDS equation becomes concrete when we show how
existing methods emerge as special cases or approximations. Each method
makes implicit choices about the reference state $\theta^*$, the balance
parameter $\lambda$, and the treatment of the temporal component. By
making these choices explicit, we gain insight into why certain methods
work in certain contexts---and where they may fail.

Table~\ref{tab:beds-methods-mapping} provides a synoptic view of how each method maps onto the BEDS equation terms.

\begin{table}[H]
\centering
\caption{Mapping of existing methods to BEDS equation terms. Each column corresponds to a term in the fundamental equation. Dashes indicate terms that are absent or implicit in the method.}
\label{tab:beds-methods-mapping}
\begin{tabular}{@{}lccccc@{}}
\toprule
\textbf{Method} & $\tau(\mu-\mu^*)^2$ & $\frac{(\tau-\tau^*)^2}{2\tau\tau^*}$ & $\kappa(1-\cos\Delta\phi)$ & $\frac{(\kappa-\kappa^*)^2}{2\kappa\kappa^*}$ & \textbf{Key assumption} \\
\midrule
Ridge & $\|\theta\|^2$ & fixed $\tau$ & -- & $\kappa \to \infty$ & constant precision \\
SIGReg & -- & $(\sigma_i - 1)^2$ & -- & -- & isotropic target \\
EMA & -- & -- & implicit & $\kappa = \frac{1}{1-m}$ & temporal only \\
SAC & -- & -- & $\mathcal{H}[\pi]$ & $\alpha = 1/\kappa$ & exploration regime \\
Attention & posterior & $1/\sqrt{d_k}$ & softmax & -- & per-token update \\
\bottomrule
\end{tabular}
\end{table}

\subsubsection{Ridge Regression}

Ridge regression provides the simplest illustration of the BEDS framework.

\begin{equation}
\mathcal{L}_{\text{Ridge}} = \|y - X\theta\|^2 + \lambda\|\theta\|^2
\end{equation}

\begin{proposition}[Ridge as Approximate BEDS]
Ridge regression constitutes an approximation to the BEDS framework under specific conditions. The target state is the origin $\theta^* = 0$ with fixed precision $\tau^* = \lambda$. The temporal component is absent, corresponding to the limit $\kappa \to \infty$ of perfect coherence. The Euclidean norm $\|\theta\|^2$ approximates the Fisher--Rao distance $\dF^2$ precisely when precision remains constant across parameters. This reveals why Ridge becomes suboptimal in heterogeneous settings: when parameter importance varies, the Euclidean approximation fails.
\end{proposition}

\subsubsection{SIGReg (Spectral Information Gain Regularization)}

Among modern SSL methods, SIGReg provides the most direct implementation
of Fisher--Rao principles.

\begin{equation}
\mathcal{L}_{\text{SIGReg}} = -\sum_i \log \sigma_i(Z) + \alpha \sum_i (\sigma_i - 1)^2
\end{equation}

\begin{proposition}[SIGReg as BEDS]
SIGReg directly implements Fisher--Rao regularization toward $\Normal(0, I)$. The log-determinant term prevents collapse ($\tau_i \to \infty$), while the variance term prevents explosion ($\tau_i \to 0$). The target is an isotropic unit Gaussian, making SIGReg the most direct realization of BEDS principles among current SSL methods.
\end{proposition}

\subsubsection{Exponential Moving Average (EMA)}

The exponential moving average, ubiquitous in self-supervised learning,
controls the temporal component of BEDS.

\begin{equation}
\theta_{\text{target}} \leftarrow m \cdot \theta_{\text{target}} + (1-m) \cdot \theta_{\text{online}}
\end{equation}

\begin{proposition}[EMA as Temporal BEDS]
The momentum coefficient controls temporal coherence:
\begin{equation}
\kappa \approx \frac{1}{1-m}
\end{equation}
High momentum ($m \to 1$) implies high temporal coherence ($\kappa \to \infty$).
\end{proposition}

\subsubsection{Soft Actor-Critic (SAC)}

Soft Actor-Critic illustrates how BEDS applies to reinforcement learning.

\begin{equation}
J(\pi) = \sum_t \E\left[r(s_t, a_t) + \alpha \mathcal{H}[\pi(\cdot|s_t)]\right]
\end{equation}

\begin{proposition}[SAC as Low-$\kappa$ BEDS]
SAC operates in the exploration regime:
\begin{itemize}
    \item Temperature $\alpha = 1/\kappa$
    \item Policy entropy corresponds to phase dispersion
    \item Low $\kappa$: exploration dominates
\end{itemize}
\end{proposition}

\subsubsection{Transformers and Attention}

The Transformer architecture, now ubiquitous in both language and vision,
implements a Bayesian belief update mechanism that maps directly onto BEDS coordinates.

\paragraph{Attention as Posterior Computation}
The core attention mechanism:
\begin{equation}
\text{Attention}(Q, K, V) = \text{softmax}\left(\frac{QK^\top}{\sqrt{d_k}}\right) V
\label{eq:attention-beds}
\end{equation}
admits a precise BEDS interpretation. The compatibility scores $QK^\top$ compute
relative log-likelihoods; the softmax normalizes these into a posterior distribution
over attention targets; the weighted combination with $V$ performs the Bayesian update.

\begin{proposition}[Attention as BEDS Update]
The attention mechanism implements a Bayesian belief update in BEDS coordinates:
\begin{itemize}
    \item \textbf{Temperature:} The scaling factor $\sqrt{d_k}$ acts as inverse
          temperature, controlling attention sharpness. This corresponds to
          effective coherence: $\kappa_{\text{eff}} \propto \sqrt{d_k}$.
    \item \textbf{Phase:} Positional encodings---particularly sinusoidal
          embeddings---explicitly represent phase $\phi$ at multiple frequencies,
          corresponding to Fourier moments of temporal SIGReg.
    \item \textbf{Precision:} Each attention head refines the representation,
          increasing effective precision $\tau$ through successive layers.
\end{itemize}
\end{proposition}

The attention posterior over positions takes the explicit form:
\begin{equation}
p(\text{attend}_i | q) = \frac{\exp(q \cdot k_i / \sqrt{d_k})}{\sum_j \exp(q \cdot k_j / \sqrt{d_k})}
\label{eq:attention-posterior}
\end{equation}
This is a von Mises-like distribution when positions are interpreted as phases,
with concentration parameter $\kappa \propto \sqrt{d_k}$.

\paragraph{Implications for LLMs}
In large language models, context position plays the role of phase $\phi$:
\begin{itemize}
    \item Causal attention enforces phase ordering (past $\to$ future)
    \item Context window limits define coherence bounds ($\kappa_{\max}$)
    \item Generation temperature $T$ directly controls output entropy via $\kappa = 1/T$
\end{itemize}
This interpretation suggests that scaling laws may have thermodynamic underpinnings:
larger models achieve higher precision $\tau$ through deeper processing,
while longer contexts extend coherence $\kappa$.

\subsubsection{Diffusion Models}

Diffusion models provide perhaps the most explicit realization of BEDS dynamics:
the forward process is pure dissipation, the reverse process is reconstruction
guided by learned observations.

\paragraph{Forward Process as Dissipation}
The forward diffusion process follows an Ornstein-Uhlenbeck equation:
\begin{equation}
dx_t = -\gamma x_t \, dt + \sigma \, dW_t
\label{eq:diffusion-forward-beds}
\end{equation}
This is explicit BEDS dissipation with transparent parameter correspondences:
\begin{itemize}
    \item Precision decreases: $\tau(t) \propto 1/\sigma_t^2 \to 0$
    \item Entropy is injected via Brownian noise $dW_t$
    \item Final state $x_T \sim \mathcal{N}(0, I)$ represents thermal equilibrium
          (maximum entropy, minimum precision)
\end{itemize}

\begin{proposition}[Diffusion as BEDS Cycle]
Diffusion models implement a complete BEDS cycle:
\begin{enumerate}
    \item \textbf{Forward (dissipation):} Information flows from data to noise,
          $\tau \to 0$, entropy increases
    \item \textbf{Reverse (reconstruction):} Learned score injects information,
          $\tau$ increases, structure re-emerges
\end{enumerate}
\end{proposition}

\paragraph{Reverse Process and Score Matching}
The reverse process reconstructs structure via:
\begin{equation}
dx_t = \left[\gamma x_t + \sigma^2 \nabla_x \ln p_t(x)\right] dt + \sigma \, d\tilde{W}_t
\label{eq:diffusion-reverse-beds}
\end{equation}
The score function $\nabla_x \ln p_t(x)$, learned through denoising score matching,
plays the role of observation in BEDS: it injects the information necessary to
reverse dissipation and reconstruct the original distribution.

\paragraph{BEDS-Diffusion Correspondences}
\begin{table}[h]
\centering
\begin{tabular}{ll}
\toprule
\textbf{Diffusion} & \textbf{BEDS} \\
\midrule
Forward process & Dissipation ($\tau \to 0$) \\
Noise level $\sigma_t$ & $1/\sqrt{\tau}$ \\
Score $\nabla \ln p_t$ & Observation term \\
$x_T \sim \mathcal{N}(0, I)$ & Maximum-entropy state \\
Reverse process & Reconstruction ($\tau$ increases) \\
Denoising steps & Geodesic return in belief space \\
\bottomrule
\end{tabular}
\end{table}

This interpretation reveals why diffusion models succeed: they explicitly implement
the thermodynamically optimal path---first dissipating to maximum entropy (erasing
structure), then reconstructing along the geodesic defined by the learned score.

Figure~\ref{fig:unification} synthesizes these mappings visually. The left branch shows spatial methods (weight decay, batch normalization) as approximations to Fisher--Rao distance on $\HH^2$; the right branch shows temporal methods (momentum, EMA) as controlling coherence on $\MvM$. Diffusion models, uniquely, traverse the \emph{full} BEDS cycle: the forward process dissipates structure ($\tau \to 0$), while the reverse process reconstructs it via learned score functions. This unification is not metaphorical---under A1--A3, these methods implement thermodynamically necessary operations.

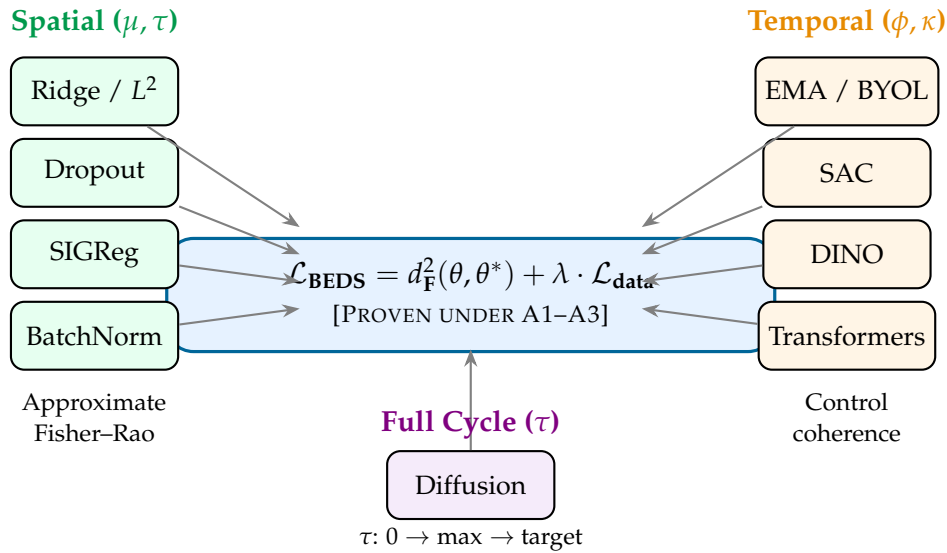
\begin{figure}[H]
\centering
\begin{tikzpicture}[
    scale=0.9,
    every node/.style={font=\small},
    method/.style={draw, rounded corners=5pt, minimum width=2.2cm, minimum height=0.9cm, thick},
    arrow/.style={-{Stealth}, thick, gray}
]

\node[draw, rounded corners=10pt, fill=lightblue, minimum width=8cm, minimum height=1.5cm,
      very thick, draw=bedsblue] (central) at (0,0) {};
\node[font=\bfseries] at (0,0.3) {$\bedsL = \dF^2(\theta, \theta^*) + \lambda \cdot \mathcal{L}_{\text{data}}$};
\node[font=\footnotesize] at (0,-0.3) {\proven};

\node[method, fill=lightgreen] (ridge) at (-5.5,3) {Ridge / $L^2$};
\node[method, fill=lightgreen] (dropout) at (-5.5,1.8) {Dropout};
\node[method, fill=lightgreen] (sigreg) at (-5.5,0.6) {SIGReg};
\node[method, fill=lightgreen] (batchnorm) at (-5.5,-0.6) {BatchNorm};

\node[method, fill=lightorange] (ema) at (5.5,3) {EMA / BYOL};
\node[method, fill=lightorange] (sac) at (5.5,1.8) {SAC};
\node[method, fill=lightorange] (dino) at (5.5,0.6) {DINO};
\node[method, fill=lightorange] (transformers) at (5.5,-0.6) {Transformers};

\node[method, fill=lightpurple] (diffusion) at (0,-2.8) {Diffusion};

\draw[arrow] (ridge) -- (-2.5,1);
\draw[arrow] (dropout) -- (-2.5,0.6);
\draw[arrow] (sigreg) -- (-2.5,0.2);
\draw[arrow] (batchnorm) -- (-2.5,-0.2);

\draw[arrow] (ema) -- (2.5,1);
\draw[arrow] (sac) -- (2.5,0.6);
\draw[arrow] (dino) -- (2.5,0.2);
\draw[arrow] (transformers) -- (2.5,-0.2);

\draw[arrow] (diffusion) -- (0,-0.8);

\node[font=\bfseries, bedsgreen] at (-5.5,4) {Spatial ($\mu, \tau$)};
\node[font=\bfseries, bedsorange] at (5.5,4) {Temporal ($\phi, \kappa$)};
\node[font=\bfseries, bedspurple] at (0,-1.9) {Full Cycle ($\tau$)};

\node[font=\footnotesize, text width=2.5cm, align=center] at (-5.5,-1.8)
    {Approximate\\Fisher--Rao};
\node[font=\footnotesize, text width=2.5cm, align=center] at (5.5,-1.8)
    {Control\\coherence};
\node[font=\footnotesize] at (0,-3.6)
    {$\tau$: $0 \to \max \to$ target};

\end{tikzpicture}
\caption{\textbf{Unification of ML methods under BEDS.} Existing regularization methods are special cases or approximations of the fundamental Fisher--Rao equation. Transformers implement Bayesian belief updates with temperature-controlled precision, while diffusion models traverse the full BEDS cycle---dissipation followed by reconstruction.}
\label{fig:unification}
\end{figure}


\section{Thermodynamic Efficiency of Learning}
\label{sec:efficiency}

\subsection{Definition of Efficiency}
\label{subsec:efficiency-definition}

The Energy-Precision Bound establishes a minimum energy cost; we now ask
how close real learning systems come to this theoretical limit.

\begin{definition}[Thermodynamic Efficiency]
\label{def:efficiency}
The thermodynamic efficiency of a learning process is:
\begin{equation}
\boxed{\eta = \frac{E_{\text{Landauer}}}{E_{\text{actual}}} = \frac{\kB T \ln 2 \cdot I}{E_{\text{actual}}} \in (0, 1]}
\end{equation}
where $I$ is the information erased during learning and $E_{\text{actual}}$ is the actual energy expended. Efficiency $\eta = 1$ corresponds to thermodynamic optimality.
\end{definition}

\subsection{Sources of Inefficiency}
\label{subsec:inefficiency-sources}

We decompose inefficiency multiplicatively:
\begin{equation}
\frac{1}{\eta} = \underbrace{\frac{E_{\text{hardware}}}{E_{\text{Landauer}}}}_{\text{hardware overhead}} \times \underbrace{\frac{E_{\text{algorithm}}}{E_{\text{optimal}}}}_{\text{algorithmic inefficiency}} \times \underbrace{\frac{E_{\text{dissipated}}}{E_{\text{necessary}}}}_{\text{dissipative overhead}}
\end{equation}

\begin{table}[htbp]
\centering
\caption{\textbf{Sources of inefficiency in learning.}}
\label{tab:inefficiency}
\begin{tabular}{lll}
\toprule
\textbf{Factor} & \textbf{Source} & \textbf{BEDS Relevance} \\
\midrule
Hardware overhead & GPU/TPU vs reversible computing & Independent of BEDS \\
Algorithmic inefficiency & Suboptimal optimization & Partially addressed \\
Dissipative overhead & Non-geodesic paths & \textbf{Directly addressed} \\
\bottomrule
\end{tabular}
\end{table}

The BEDS framework addresses the third factor: by using Fisher--Rao regularization, we minimize dissipative overhead.

\subsection{Efficiency of Different Regularization Schemes}
\label{subsec:efficiency-comparison}

Different regularization strategies achieve different efficiencies relative
to the thermodynamic optimum.

\begin{proposition}[Relative Efficiency]
For Gaussian beliefs, the efficiency ratio between Euclidean and Fisher--Rao regularization is:
\begin{equation}
\frac{\eta_{\text{Euclid}}}{\eta_{FR}} \leq 1
\end{equation}
with equality only when precision is constant across all parameters.
\end{proposition}

\begin{proof}
Euclidean regularization does not respect the natural geometry, hence follows non-geodesic paths, hence dissipates more than necessary.
\end{proof}

\subsection{Self-Supervised Learning and Efficiency}
\label{subsec:ssl-efficiency}

The thermodynamic perspective suggests a fundamental advantage for
self-supervised learning.

\begin{conjbox}[Conjecture: SSL Thermodynamic Advantage]
\label{conj:ssl}
Under the BEDS framework, self-supervised learning achieves higher thermodynamic efficiency than supervised learning:
\begin{equation}
\eta_{\text{SSL}} \geq \eta_{\text{supervised}}
\end{equation}
with equality if and only if labels provide no information beyond the data structure.
\end{conjbox}

\conjectured

\textbf{Argument}: Supervised learning injects external information (labels) that must be maintained against dissipation. SSL converges to equilibrium with the data distribution, requiring no maintenance energy for external bias.

\textbf{Empirical support}: DINO/DINOv2 produce better representations than supervised counterparts; SSL pre-training outperforms supervised pre-training on transfer tasks.

These efficiency considerations provide a thermodynamic lens for evaluating
and comparing learning methods. The next section extends BEDS to systems
with multiple interacting agents.


\section{Recursive Hierarchy and Multi-Agent Systems}
\label{sec:recursive-hierarchy}

The BEDS framework developed so far applies to individual learning systems.
However, many practical applications involve multiple interacting agents:
federated learning, ensemble methods, distributed optimization, and multi-sensor
fusion. This section extends BEDS to these settings, showing that the same
thermodynamic principles govern multi-agent dynamics. The key insight is that
BEDS possesses a natural recursive structure: the posterior at one level
becomes the prior at the next, enabling hierarchical organization of knowledge.

The BEDS framework possesses a recursive property: what is learned at one level crystallizes into assumed structure at the next. This section develops this recursive structure, connects it to Markov Random Fields and Energy-Based Models, and establishes the formal relationship between BEDS and joint embedding architectures like JEPA.
This section is conceptual and exploratory in nature. Its purpose is not to propose concrete multi-agent algorithms, but to show that the dissipative constraints formalized by BEDS apply recursively to any system maintaining collective belief states under finite resources.
\subsection{The Recursive Nature of BEDS}
\label{subsec:recursive-nature}

The fundamental principle underlying BEDS hierarchy is crystallization
propagation across levels.

\begin{principle}[Recursive Crystallization]
\label{principle:recursive}
The posterior at level $n$ becomes the prior at level $n+1$:
\begin{equation}
\boxed{p_{n+1}(\theta) = p_n(\theta | D_n)}
\label{eq:recursive-prior}
\end{equation}
What is learned at level $n$ crystallizes into assumed structure at level $n+1$.
\end{principle}

\proven

This principle has profound implications. Consider the analogy of a river and a mill from Part I. The riverbed crystallizes first---water learns the path to the sea. Once stable, this crystallized structure becomes the \emph{prior} for the mill: ``water flows here, at this rate.'' The mill doesn't need to learn this; it inherits it as assumed structure. Its own learning operates on top of this inherited certainty.

\begin{figure}[H]
\centering
\begin{tikzpicture}[
    scale=0.9,
    level/.style={rectangle, draw, thick, minimum width=3.5cm, minimum height=1.2cm, rounded corners=3pt},
    arrow/.style={->, thick, >=stealth},
    crystal/.style={->, thick, >=stealth, bedsgreen, dashed}
]

\node[level, fill=lightblue] (l0) at (0, 4) {Level 0: $(\mu_0, \tau_0)$};
\node[right=0.3cm of l0, font=\small\itshape, text=gray] {Raw features};

\draw[crystal] (0, 3.3) -- (0, 2.7) node[midway, right, font=\scriptsize, text=bedsgreen] {crystallize};

\node[level, fill=lightorange] (l1) at (0, 2) {Level 1: $(\mu_1, \tau_1)$};
\node[right=0.3cm of l1, font=\small\itshape, text=gray] {Edges, textures};

\draw[arrow, bedsorange] (-1.75, 3.3) to[out=-90, in=90] node[midway, left, font=\scriptsize] {$p_1 = p_0(\cdot|D_0)$} (-1.75, 2.7);

\draw[crystal] (0, 1.3) -- (0, 0.7) node[midway, right, font=\scriptsize, text=bedsgreen] {crystallize};

\node[level, fill=lightgreen] (l2) at (0, 0) {Level 2: $(\mu_2, \tau_2)$};
\node[right=0.3cm of l2, font=\small\itshape, text=gray] {Parts, objects};

\draw[arrow, bedsorange] (-1.75, 1.3) to[out=-90, in=90] node[midway, left, font=\scriptsize] {$p_2 = p_1(\cdot|D_1)$} (-1.75, 0.7);

\draw[crystal] (0, -0.7) -- (0, -1.3) node[midway, right, font=\scriptsize, text=bedsgreen] {crystallize};

\node[level, fill=lightpurple] (l3) at (0, -2) {Level 3: $(\mu_3, \tau_3)$};
\node[right=0.3cm of l3, font=\small\itshape, text=gray] {Semantic concepts};

\draw[arrow, bedsorange] (-1.75, -0.7) to[out=-90, in=90] node[midway, left, font=\scriptsize] {$p_3 = p_2(\cdot|D_2)$} (-1.75, -1.3);

\node[font=\footnotesize, text=gray] at (4.5, 1.5) {Each level inherits};
\node[font=\footnotesize, text=gray] at (4.5, 1.0) {crystallized structure};
\node[font=\footnotesize, text=gray] at (4.5, 0.5) {from below};

\end{tikzpicture}
\caption{\textbf{Recursive crystallization in BEDS.} The posterior at each level becomes the prior for the next. Level 0 processes raw features; Level 1 builds on crystallized edges and textures; Level 2 assembles parts and objects; Level 3 forms semantic concepts. This mirrors the hierarchy in deep neural networks.}
\label{fig:recursive-levels}
\end{figure}
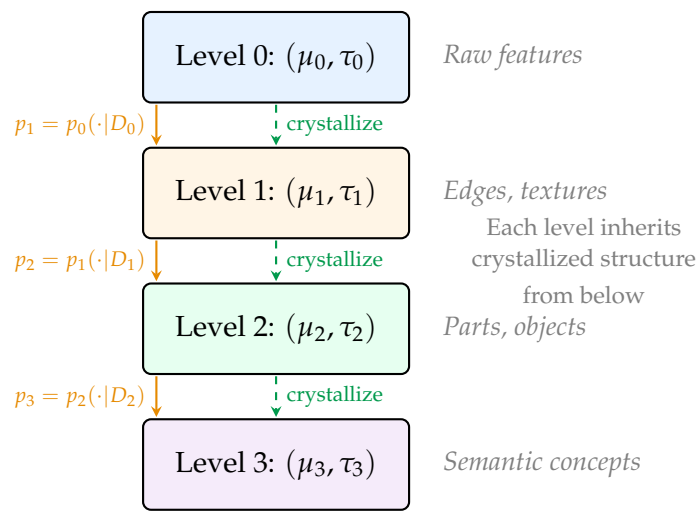

This structure appears naturally in deep learning. In a convolutional network:
\begin{itemize}
    \item \textbf{Level 0}: Raw pixels (maximum uncertainty about structure)
    \item \textbf{Level 1}: Edge detectors crystallize (inherited as ``edges exist'')
    \item \textbf{Level 2}: Texture patterns crystallize (inherited as ``textures compose'')
    \item \textbf{Level 3}: Object parts crystallize (inherited as ``parts combine'')
    \item \textbf{Level 4}: Objects crystallize (inherited as ``objects exist'')
\end{itemize}

Each layer doesn't re-learn what the previous layers established. It \emph{inherits} those crystallized beliefs as its prior, then learns its own structure on top.

\subsection{From Linear Chain to Emergent Graph}
\label{subsec:chain-to-graph}

The linear formulation generalizes naturally to networks of interacting
agents. The recursive formulation presented so far is \emph{linear}---a single chain of crystallizing structures. But in practice, at each level, multiple agents observe, interact, and crystallize simultaneously.

\begin{figure}[H]
\centering
\begin{tikzpicture}[scale=0.85]

\begin{scope}[xshift=-4cm]
\node[font=\bfseries\small, text=bedsblue] at (0, 3.5) {LINEAR BEDS};
\node[font=\scriptsize, text=gray] at (0, 3.0) {(simplified)};

\node[draw, circle, fill=lightblue, minimum size=0.8cm] (s0) at (0, 2) {$S_0$};
\node[draw, circle, fill=lightorange, minimum size=0.8cm] (s1) at (0, 0.5) {$S_1$};
\node[draw, circle, fill=lightgreen, minimum size=0.8cm] (s2) at (0, -1) {$S_2$};

\draw[->, thick] (s0) -- (s1);
\draw[->, thick] (s1) -- (s2);

\node[font=\scriptsize, text=gray] at (0, -2) {One structure per level};
\end{scope}

\draw[->, very thick, bedsblue] (-1, 0.5) -- (1, 0.5);
\node[above, font=\scriptsize] at (0, 0.6) {generalize};

\begin{scope}[xshift=4cm]
\node[font=\bfseries\small, text=bedsblue] at (0, 3.5) {BEDS AS GRAPH};
\node[font=\scriptsize, text=gray] at (0, 3.0) {(emergent)};

\node[draw, circle, fill=lightblue, minimum size=0.6cm, font=\tiny] (a0) at (-1.5, 2) {$S_0^a$};
\node[draw, circle, fill=lightblue, minimum size=0.6cm, font=\tiny] (b0) at (0, 2) {$S_0^b$};
\node[draw, circle, fill=lightblue, minimum size=0.6cm, font=\tiny] (c0) at (1.5, 2) {$S_0^c$};

\draw[thick, gray] (a0) -- (b0) node[midway, above, font=\tiny] {$\psi$};
\draw[thick, gray] (b0) -- (c0) node[midway, above, font=\tiny] {$\psi$};

\node[draw, circle, fill=lightorange, minimum size=0.6cm, font=\tiny] (a1) at (-0.75, 0.5) {$S_1^a$};
\node[draw, circle, fill=lightorange, minimum size=0.6cm, font=\tiny] (b1) at (0.75, 0.5) {$S_1^b$};

\draw[thick, gray] (a0) -- (a1);
\draw[thick, gray] (b0) -- (a1);
\draw[thick, gray] (b0) -- (b1);
\draw[thick, gray] (c0) -- (b1);

\draw[very thick, bedsblue] (a1) -- (b1) node[midway, above, font=\tiny] {$\psi$};

\node[draw, circle, fill=lightgreen, minimum size=0.6cm, font=\tiny] (s2g) at (0, -1) {$S_2$};
\draw[thick, gray] (a1) -- (s2g);
\draw[thick, gray] (b1) -- (s2g);

\node[font=\scriptsize, text=gray, text width=3cm, align=center] at (0, -2.2) {N agents per level\\Local interactions};
\end{scope}

\end{tikzpicture}
\caption{\textbf{From linear chain to emergent graph.} The linear BEDS formulation (left) generalizes to a multi-agent graph (right) where multiple agents at each level interact through learned potentials $\psi_{ij}$.}
\label{fig:linear-to-graph}
\end{figure}
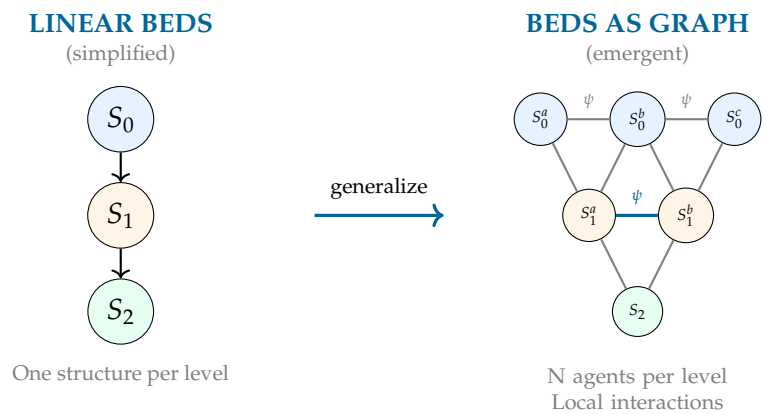

This generalization is natural and necessary. Figure~\ref{fig:linear-to-graph} illustrates the transition: the linear chain (left) assumes each agent interacts only with immediate neighbors, while the multi-agent graph (right) allows arbitrary interaction patterns. The interaction potentials $\psi_{ij}$ are not fixed \emph{a priori}---they emerge from learning, allowing the network to discover which agents should trust which others. In reality:
\begin{itemize}
    \item Multiple sensors observe different aspects of the same phenomenon
    \item Multiple models make predictions about overlapping domains
    \item Multiple agents learn from different data subsets
\end{itemize}

The structure of their interactions---who communicates with whom, and how strongly---\emph{emerges} from learning.

\subsection{Three Levels of Learning}
\label{subsec:three-learning-levels}

Multi-agent BEDS systems learn at three distinct timescales. In a classical Markov Random Field (MRF), the graph structure and interaction potentials are \emph{fixed}. Inference finds the beliefs; learning adjusts parameters. In BEDS, \emph{everything is learned}, operating at three distinct timescales:

\begin{table}[htbp]
\centering
\caption{\textbf{Three levels of learning in BEDS networks.} Each level operates at a different timescale with a distinct mechanism.}
\label{tab:three-levels}
\begin{tabular}{llll}
\toprule
\textbf{Level} & \textbf{What is Learned} & \textbf{Timescale} & \textbf{Mechanism} \\
\midrule
1. Beliefs & $(\mu_i, \tau_i)$ per agent & Fast & Bayesian fusion \\
2. Potentials & $\psi_{ij}$ between agents & Medium & $\Delta\psi_{ij} \propto \text{agreement}(i,j) - \text{baseline}$ \\
3. Topology & Structure of graph $G$ & Slow & Pruning weak connections ($\psi_{ij} \to 0$) \\
\bottomrule
\end{tabular}
\end{table}

\textbf{Level 1 (Beliefs):} Each agent maintains and updates its belief state $(\mu_i, \tau_i)$ through Bayesian fusion. When agent $i$ receives a message from agent $j$, precisions add and means are precision-weighted averaged---exactly the standard BEDS update.

\textbf{Level 2 (Potentials):} The interaction strengths $\psi_{ij}$ are learned from agreement history. When agents $i$ and $j$ consistently agree, their coupling strengthens. When they consistently disagree, it weakens:
\begin{equation}
\Delta\psi_{ij} \propto \mathbb{E}\left[\text{agreement}(i,j)\right] - \text{baseline}
\end{equation}

\textbf{Level 3 (Topology):} Over longer timescales, weak connections are pruned entirely ($\psi_{ij} \to 0$), and the graph structure itself emerges. This is the slowest form of learning---structural crystallization.

\subsection{Energy-Based Formulation: BEDS as Markov Random Field}
\label{subsec:mrf-formulation}

The multi-agent BEDS system admits a natural interpretation as a Markov
Random Field. Consider $N$ agents with belief states $\{S_1, \ldots, S_N\}$ and graph structure $G = (V, E)$.

\begin{proposition}[MRF Structure]
\label{prop:mrf-structure}
The joint distribution over agent states follows the Gibbs form:
\begin{equation}
P(\{S\}) = \frac{1}{Z} \exp\left(-\frac{E(\{S\})}{T}\right)
\label{eq:gibbs}
\end{equation}
where the total energy decomposes into three terms:
\begin{equation}
\boxed{E = \underbrace{\sum_i D_{\KL}(q_i \| p_i^{\text{data}})}_{\text{data fidelity } E_{\text{data}}} + \underbrace{\sum_{(i,j) \in E} \psi_{ij} \cdot \dF^2(q_i, q_j)}_{\text{coherence } E_{\text{interact}}} + \underbrace{\sum_i D_{\KL}(q_i \| \pi_i)}_{\text{prior } E_{\text{prior}}}}
\label{eq:beds-energy}
\end{equation}
\end{proposition}

\conjectured

Each term has a clear interpretation:

\begin{itemize}
    \item $\mathbf{E_{\text{data}}}$ (Data Fidelity): How well does each agent's belief $q_i$ match its local observations? This term penalizes beliefs that diverge from the data.

    \item $\mathbf{E_{\text{interact}}}$ (Coherence): How consistent are neighboring beliefs? The Fisher--Rao distance $\dF^2(q_i, q_j)$ measures the geodesic distance between belief distributions. Strong coupling ($\psi_{ij}$ large) heavily penalizes disagreement.

    \item $\mathbf{E_{\text{prior}}}$ (Prior): How far has each agent drifted from its inherited prior $\pi_i$? This term anchors beliefs to crystallized structure from lower levels.
\end{itemize}

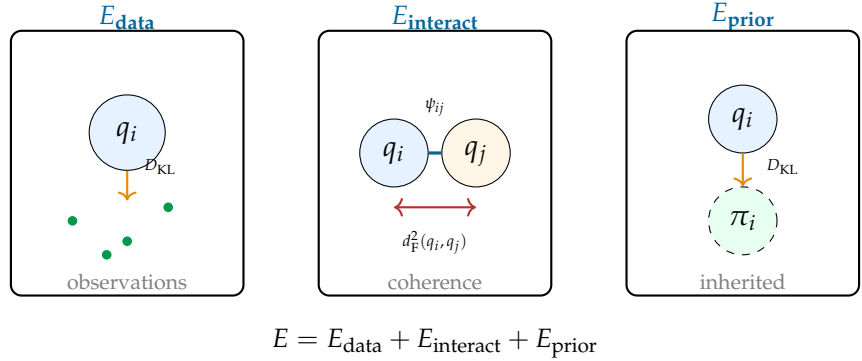
\begin{figure}[H]
\centering
\begin{tikzpicture}[scale=0.9]

\begin{scope}[xshift=-4.5cm]
\node[font=\bfseries\small, text=bedsblue] at (0, 2.5) {$E_{\text{data}}$};
\draw[thick, rounded corners] (-1.7, -1.6) rectangle (1.7, 2.3);

\node[draw, circle, fill=lightblue, minimum size=1cm] (agent) at (0, 0.8) {$q_i$};

\foreach \x/\y in {-0.8/-0.5, 0/-0.8, 0.6/-0.3, -0.3/-1.0} {
    \fill[bedsgreen] (\x, \y) circle (2pt);
}
\node[font=\scriptsize, text=gray] at (0, -1.4) {observations};

\draw[->, thick, bedsorange] (agent) -- (0, -0.2);
\node[right, font=\tiny] at (0.1, 0.3) {$D_{\KL}$};
\end{scope}

\begin{scope}[xshift=0cm]
\node[font=\bfseries\small, text=bedsblue] at (0, 2.5) {$E_{\text{interact}}$};
\draw[thick, rounded corners] (-1.7, -1.6) rectangle (1.7, 2.3);

\node[draw, circle, fill=lightblue, minimum size=0.9cm] (ai) at (-0.6, 0.5) {$q_i$};
\node[draw, circle, fill=lightorange, minimum size=0.9cm] (aj) at (0.6, 0.5) {$q_j$};

\draw[very thick, bedsblue] (ai) -- (aj);
\node[above, font=\tiny] at (0, 0.9) {$\psi_{ij}$};

\draw[<->, thick, bedsred] (-0.6, -0.3) -- (0.6, -0.3);
\node[below, font=\tiny] at (0, -0.5) {$\dF^2(q_i, q_j)$};

\node[font=\scriptsize, text=gray] at (0, -1.4) {coherence};
\end{scope}

\begin{scope}[xshift=4.5cm]
\node[font=\bfseries\small, text=bedsblue] at (0, 2.5) {$E_{\text{prior}}$};
\draw[thick, rounded corners] (-1.7, -1.6) rectangle (1.7, 2.3);

\node[draw, circle, fill=lightblue, minimum size=0.9cm] (qi) at (0, 1.0) {$q_i$};

\node[draw, circle, fill=lightgreen, minimum size=0.9cm, dashed] (pi) at (0, -0.5) {$\pi_i$};

\draw[->, thick, bedsorange] (qi) -- (pi);
\node[right, font=\tiny] at (0.2, 0.3) {$D_{\KL}$};

\node[font=\scriptsize, text=gray] at (0, -1.4) {inherited};
\end{scope}

\node[font=\small] at (0, -2.3) {$E = E_{\text{data}} + E_{\text{interact}} + E_{\text{prior}}$};

\end{tikzpicture}
\caption{\textbf{BEDS energy decomposition.} The total energy consists of three terms: data fidelity (fit to observations), coherence (consistency with neighbors), and prior (fidelity to inherited structure).}
\label{fig:energy-decomposition}
\end{figure}

\subsection{Belief Propagation and the BEDS Protocol}
\label{subsec:bp-beds}

Inference in graphical models connects directly to BEDS update rules. Figure~\ref{fig:energy-decomposition} illustrates the three-term energy structure: each agent $i$ balances fidelity to data ($E_{\text{data}}$), consistency with neighbors via interaction potentials ($E_{\text{interact}}$), and adherence to inherited priors ($E_{\text{prior}}$). Minimizing total energy while respecting the dissipation constraint yields the BEDS update equations.
In graphical models, inference is performed via \emph{message passing}. The canonical algorithm is Belief Propagation (BP):
\begin{equation}
m_{i \to j}(x_j) = \sum_{x_i} \psi_{ij}(x_i, x_j) \cdot b_i(x_i) \cdot \prod_{k \in N(i) \setminus j} m_{k \to i}(x_i)
\label{eq:bp-general}
\end{equation}

For Gaussian beliefs, this dramatically simplifies. The sum-product rule becomes precision-weighted averaging:

\begin{proposition}[Gaussian Belief Propagation]
\label{prop:gaussian-bp}
For Gaussian beliefs $q_i = \Normal(\mu_i, \tau_i^{-1})$ and $q_j = \Normal(\mu_j, \tau_j^{-1})$, belief propagation implements:
\begin{align}
\tau_j^{\text{new}} &= \tau_j + \tau_i \\
\mu_j^{\text{new}} &= \frac{\tau_j \mu_j + \tau_i \mu_i}{\tau_j^{\text{new}}}
\end{align}
This is exactly the BEDS fusion rule.
\end{proposition}

\begin{keyresult}[BP-BEDS Equivalence]
The BEDS P2P protocol is Belief Propagation on a learned graph with Gaussian beliefs. Standard results from graphical models (convergence conditions, loopy BP analysis) apply directly.
\end{keyresult}

\conjectured

This equivalence has practical implications: decades of research on efficient message-passing inference transfer directly to BEDS networks.

\subsection{Emergent Properties}
\label{subsec:emergent-properties}

When interaction potentials are learned rather than fixed, several
structural properties emerge spontaneously. When potentials are learned (rather than fixed), four remarkable properties emerge:

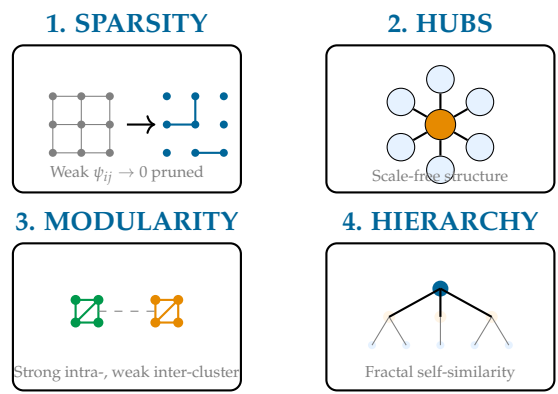
\begin{figure}[H]
\centering
\begin{tikzpicture}[scale=0.75]

\begin{scope}[xshift=-5.5cm, yshift=2cm]
\node[font=\bfseries\small, text=bedsblue] at (0, 1.8) {1. SPARSITY};
\draw[thick, rounded corners] (-2, -1.2) rectangle (2, 1.4);

\begin{scope}[xshift=-0.8cm]
\foreach \i in {0,0.5,1} {
    \foreach \j in {0,0.5,1} {
        \fill[gray] (\i-0.5, \j-0.5) circle (2pt);
    }
}
\draw[gray, thin] (-0.5,-0.5) grid[step=0.5] (0.5,0.5);
\end{scope}

\draw[->, thick] (0, 0) -- (0.5, 0);

\begin{scope}[xshift=1.2cm]
\foreach \i in {0,0.5,1} {
    \foreach \j in {0,0.5,1} {
        \fill[bedsblue] (\i-0.5, \j-0.5) circle (2pt);
    }
}
\draw[bedsblue, thick] (-0.5,0) -- (0,0) -- (0,0.5);
\draw[bedsblue, thick] (0,-0.5) -- (0.5,-0.5);
\end{scope}

\node[font=\tiny, text=gray] at (0, -0.9) {Weak $\psi_{ij} \to 0$ pruned};
\end{scope}

\begin{scope}[xshift=0cm, yshift=2cm]
\node[font=\bfseries\small, text=bedsblue] at (0, 1.8) {2. HUBS};
\draw[thick, rounded corners] (-2, -1.2) rectangle (2, 1.4);

\node[draw, circle, fill=bedsorange, minimum size=0.4cm] (hub) at (0, 0) {};
\foreach \angle in {30, 90, 150, 210, 270, 330} {
    \node[draw, circle, fill=lightblue, minimum size=0.25cm] (n\angle) at (\angle:0.8) {};
    \draw[thick] (hub) -- (n\angle);
}

\node[font=\tiny, text=gray] at (0, -0.9) {Scale-free structure};
\end{scope}

\begin{scope}[xshift=-5.5cm, yshift=-1.5cm]
\node[font=\bfseries\small, text=bedsblue] at (0, 1.8) {3. MODULARITY};
\draw[thick, rounded corners] (-2, -1.2) rectangle (2, 1.4);

\begin{scope}[xshift=-0.9cm]
\foreach \i in {0,0.4} {
    \foreach \j in {0,0.4} {
        \fill[bedsgreen] (\i, \j) circle (2.5pt);
    }
}
\draw[bedsgreen, thick] (0,0) -- (0.4,0) -- (0.4,0.4) -- (0,0.4) -- (0,0);
\draw[bedsgreen, thick] (0,0) -- (0.4,0.4);
\end{scope}

\begin{scope}[xshift=0.5cm]
\foreach \i in {0,0.4} {
    \foreach \j in {0,0.4} {
        \fill[bedsorange] (\i, \j) circle (2.5pt);
    }
}
\draw[bedsorange, thick] (0,0) -- (0.4,0) -- (0.4,0.4) -- (0,0.4) -- (0,0);
\draw[bedsorange, thick] (0,0) -- (0.4,0.4);
\end{scope}

\draw[gray, dashed, thin] (-0.5, 0.2) -- (0.5, 0.2);

\node[font=\tiny, text=gray] at (0, -0.9) {Strong intra-, weak inter-cluster};
\end{scope}

\begin{scope}[xshift=0cm, yshift=-1.5cm]
\node[font=\bfseries\small, text=bedsblue] at (0, 1.8) {4. HIERARCHY};
\draw[thick, rounded corners] (-2, -1.2) rectangle (2, 1.4);

\foreach \x in {-1.2, -0.6, 0, 0.6, 1.2} {
    \fill[lightblue] (\x, -0.4) circle (2pt);
}
\foreach \x in {-0.9, 0, 0.9} {
    \fill[lightorange] (\x, 0.1) circle (3pt);
}
\fill[bedsblue] (0, 0.6) circle (4pt);

\draw[gray, thin] (-1.2,-0.4) -- (-0.9,0.1);
\draw[gray, thin] (-0.6,-0.4) -- (-0.9,0.1);
\draw[gray, thin] (0,-0.4) -- (0,0.1);
\draw[gray, thin] (0.6,-0.4) -- (0.9,0.1);
\draw[gray, thin] (1.2,-0.4) -- (0.9,0.1);

\draw[thick] (-0.9,0.1) -- (0,0.6);
\draw[thick] (0,0.1) -- (0,0.6);
\draw[thick] (0.9,0.1) -- (0,0.6);

\node[font=\tiny, text=gray] at (0, -0.9) {Fractal self-similarity};
\end{scope}

\end{tikzpicture}
\caption{\textbf{Emergent properties of learned BEDS networks.} When potentials are learned rather than fixed, the network spontaneously develops: (1) sparsity through pruning, (2) hub structure (scale-free), (3) modular clustering, and (4) hierarchical organization.}
\label{fig:emergent-properties}
\end{figure}

These properties are not designed in---they \emph{emerge} from the learning dynamics. Figure~\ref{fig:emergent-properties} illustrates the four emergent structures: (1) sparsity emerges because weak connections ($\psi_{ij} \approx 0$) waste energy without improving coordination; (2) hubs emerge because reliable agents attract trust from many others; (3) modularity emerges because tightly-coupled subgroups minimize internal communication cost; (4) hierarchy emerges recursively within modules. These patterns mirror biological neural networks---not by design, but by thermodynamic necessity:
\begin{enumerate}
    \item \textbf{Sparsity}: Weak connections ($\psi_{ij} \to 0$) are pruned. The graph becomes sparse, reducing communication cost.

    \item \textbf{Hubs}: Some agents become highly connected because they are reliable---others learn to trust them. This produces scale-free structure.

    \item \textbf{Modularity}: Groups of agents that consistently agree form tightly connected clusters. Inter-cluster connections remain weak.

    \item \textbf{Hierarchy}: Within clusters, sub-clusters form. The structure is fractal---self-similar at multiple scales.
\end{enumerate}

\subsection{BEDS versus Classical Frameworks}
\label{subsec:beds-vs-classical}

BEDS integrates and extends several classical frameworks from machine
learning and statistical physics:

\begin{table}[htbp]
\centering
\caption{\textbf{BEDS in relation to classical frameworks.} Each framework contributes essential ingredients; BEDS unifies them.}
\label{tab:frameworks}
\begin{tabular}{lll}
\toprule
\textbf{Framework} & \textbf{What It Provides} & \textbf{BEDS Adds} \\
\midrule
MRF & Local interactions, Gibbs distribution & Learned potentials, emergent topology \\
EBM (LeCun) & Energy minimization, contrastive & Hierarchical structure, dissipation \\
JEPA & Joint embedding, latent prediction & Distributed agents, crystallization \\
FEP (Friston) & Variational inference, active & Recursive priors, P2P implementation \\
Boltzmann Machine & Learned potentials, stochastic & Open system, entropy export \\
\bottomrule
\end{tabular}
\end{table}

A key distinction from classical Energy-Based Models:

\begin{keyresult}[EBM as Non-Dissipative BEDS]
Energy-Based Models are BEDS without dissipation ($\gamma = 0$). They crystallize immediately and permanently, which explains their failure under distribution shift---they cannot ``melt'' and re-adapt.
\end{keyresult}

This insight is practical: EBMs work when the distribution is stationary (crystallization is appropriate). They fail when the distribution drifts (dissipation is required for re-adaptation).

\subsection{Connection to JEPA and LeJEPA}
\label{subsec:jepa-connection}

The Joint Embedding Predictive Architecture (JEPA) and its theoretical foundation LeJEPA (Balestriero \& LeCun, 2025) have a precise interpretation in BEDS:

\begin{table}[htbp]
\centering
\caption{\textbf{LeJEPA $\to$ BEDS mapping.} Each LeJEPA component corresponds to a BEDS mechanism.}
\label{tab:lejepa-mapping}
\begin{tabular}{lll}
\toprule
\textbf{LeJEPA} & \textbf{BEDS} & \textbf{Role} \\
\midrule
Embedding $z$ & $\mu$ & Latent position \\
Variance of $z$ & $1/\tau$ & Uncertainty \\
Student vs. Teacher & $\phi$ (phase difference) & Synchronization state \\
EMA coefficient $\alpha$ & $\kappa$ controller & Coherence maintenance \\
SIGReg regularization & $\gamma_\tau$ & Precision dissipation \\
Prediction loss & $I_{\text{obs}}$ & Information injection \\
\bottomrule
\end{tabular}
\end{table}

\begin{keyresult}[JEPA-BEDS Interpretation]
JEPA prescribes \emph{what} to build (joint embeddings, predictive architecture). BEDS explains \emph{why} it works:
\begin{itemize}
    \item The isotropic Gaussian $\Normal(0, I)$ target of LeJEPA is the \emph{maximum entropy distribution} under fixed variance---the dissipative equilibrium of BEDS.
    \item SIGReg implements dissipation ($\gamma_\tau$), preventing precision explosion.
    \item EMA maintains coherence ($\kappa$), preventing student-teacher decoupling.
\end{itemize}
\end{keyresult}

This interpretation explains why LeJEPA is sensitive to its hyperparameters: they control the $(\tau, \kappa)$ balance. Removing any component disrupts the equilibrium, causing collapse.

\subsubsection{Gaussian Representations as Native Belief States}

The BEDS framework suggests that any representation naturally parameterized by
$(\mu, \tau)$---position and precision---constitutes a native belief-space encoding.
Gaussian splatting provides a concrete instantiation: a scene or image represented as
\begin{equation}
\mathcal{S} = \left\{ \mathcal{N}(\mu_i, \Sigma_i) \right\}_{i=1}^N
\label{eq:gaussian-splat-set}
\end{equation}
where each Gaussian primitive carries explicit position $\mu_i$ and precision
$\tau_i = \Sigma_i^{-1}$.

\paragraph{Gaussian Primitives as Belief States}
Under Assumption A2 (maximum-entropy distributions under constraints), each Gaussian
splat \emph{already is} a belief state---not a learned approximation to one. This yields:

\begin{proposition}[Gaussian Splatting as BEDS Representation]
A Gaussian splatting representation implements BEDS coordinates natively:
\begin{itemize}
    \item \textbf{Position $\mu$:} Gaussian mean encodes spatial belief
    \item \textbf{Precision $\tau$:} Inverse covariance $\Sigma^{-1}$ encodes certainty
    \item \textbf{Opacity/features:} Additional attributes can encode phase $\phi$
          or coherence $\kappa$ across neighboring primitives
\end{itemize}
\end{proposition}

\paragraph{Training as Quasi-Static Trajectory}
Standard Gaussian splatting optimization---slow adjustment of $\mu$ and $\Sigma$
under reconstruction losses---approximates quasi-static trajectories in Fisher--Rao
geometry. Common heuristics map directly to BEDS dissipation mechanisms:

\begin{table}[h]
\centering
\begin{tabular}{ll}
\toprule
\textbf{Gaussian Splatting} & \textbf{BEDS Interpretation} \\
\midrule
Gaussian mean $\mu_i$ & Position (spatial belief) \\
Covariance $\Sigma_i$ & $1/\tau$ (inverse precision) \\
Opacity decay & Entropy export (dissipation) \\
Splat pruning & Structural forgetting \\
Covariance regularization & Precision control ($\gamma_\tau$) \\
Limited splat count & Bounded state viability \\
\bottomrule
\end{tabular}
\caption{Correspondence between Gaussian splatting mechanisms and BEDS parameters.}
\label{tab:gs-beds}
\end{table}

\paragraph{Self-Supervised Learning on Gaussian Primitives}
This observation suggests that JEPA-style masked prediction can operate directly
on Gaussian representations:
\begin{enumerate}
    \item Represent input as a set of Gaussians $\mathcal{S}$
    \item Mask a subset $M \subset \mathcal{S}$
    \item Predict masked parameters from visible primitives
    \item Minimize Fisher--Rao distance:
\end{enumerate}
\begin{equation}
\mathcal{L} = \sum_{i \in M} d^2_{\text{FR}}\left(
    \mathcal{N}(\mu_i, \Sigma_i),\,
    \mathcal{N}(\hat{\mu}_i, \hat{\Sigma}_i)
\right)
\label{eq:gs-jepa-loss}
\end{equation}

\begin{remark}[Independent Validation: Wasserstein-Fisher-Rao Gradient Flows]
\label{rem:wfr-validation}
Recent independent work by Daniels and Rigollet~\cite{daniels_rigollet_2025} provides
strong mathematical support for our Prediction P-GS (Section~\ref{subsec:prediction-gs}).
They prove that:
\begin{enumerate}
    \item Gaussian splats naturally live on the \emph{Bures-Wasserstein manifold}
    $\mathsf{BW}_\rho(\mathbb{R}^d)$, a geodesically convex subset of the
    2-Wasserstein space $\mathcal{W}_2(\mathbb{R}^d)$;
    \item Standard Gaussian Splatting optimization can be recovered as
    \emph{Wasserstein-Fisher-Rao gradient descent}, with explicit gradient formulas
    for both the Wasserstein and Fisher-Rao components;
    \item Training heuristics---including splat pruning, position noising, and
    opacity decay---admit principled interpretations as birth-death dynamics
    and entropic regularization within the WFR geometry.
\end{enumerate}
This confirms that Gaussian Splatting naturally approximates thermodynamically
optimal trajectories in precisely the Fisher--Rao geometry predicted by the
BEDS framework under Assumptions A1--A3.
\end{remark}

This is JEPA \emph{directly in belief space}, without requiring an encoder to first
learn a probabilistic latent representation. The geometry is explicit rather than emergent.

\subsection{Crystallization Thresholds and Bounded Energy}
\label{subsec:crystallization-bounded}

The multi-agent setting introduces threshold phenomena in crystallization.
Crystallization in the multi-agent setting occurs in stages:

\begin{definition}[Crystallization Thresholds]
\label{def:crystal-thresholds}
\begin{itemize}
    \item \textbf{Position crystallization}: $\tau > \tau_{\text{crit}}$ implies $\mu$ frozen
    \item \textbf{Phase crystallization}: $\kappa > \kappa_{\text{crit}}$ implies $\phi$ frozen
    \item \textbf{Complete crystallization}: Both $\tau$ and $\kappa$ exceed their thresholds
\end{itemize}
\end{definition}

A key theoretical result bounds the total energy of a hierarchical BEDS system:

\begin{theorem}[Bounded Total Energy]
\label{thm:bounded-energy-main}
For a hierarchy with geometrically decreasing dissipation $\gamma_n = \gamma_0 r^n$ where $0 < r < 1$:
\begin{equation}
\boxed{E_{\text{total}} < \frac{E_0}{1-r}}
\label{eq:bounded-energy}
\end{equation}
\end{theorem}

\proven

\begin{proof}
At level $n$, the dissipation rate is $\gamma_n = \gamma_0 r^n$. The maintenance energy required at level $n$ is proportional to $\gamma_n$:
\[
E_n = E_0 \cdot r^n
\]
The total energy is the geometric series:
\[
E_{\text{total}} = \sum_{n=0}^{\infty} E_n = E_0 \sum_{n=0}^{\infty} r^n = \frac{E_0}{1-r}
\]
Since $0 < r < 1$, this sum converges.
\end{proof}

\textbf{Interpretation}: Abstract knowledge (deep layers, crystallized structure) costs less to maintain than concrete knowledge (early layers, raw features). This explains why pretrained representations transfer efficiently---the crystallized knowledge requires little ongoing maintenance.

\subsection{Multi-Agent Applications}
\label{subsec:multi-agent-apps}

The multi-agent formulation applies directly to several practical scenarios:

\begin{itemize}
    \item \textbf{Federated Learning}: Each client is a BEDS agent. They share belief updates (not raw data), and the global model emerges from distributed crystallization.

    \item \textbf{Ensemble Methods}: Each model is a BEDS agent. The ensemble prediction is the crystallized consensus, weighted by precision.

    \item \textbf{Distributed Optimization}: Each worker is a BEDS agent. Potentials encode trust in each worker's gradients, learned from agreement history.

    \item \textbf{Multi-Sensor Fusion}: Each sensor is a BEDS agent. Sensor fusion follows the standard BEDS protocol---precision-weighted averaging with learned trust weights.
\end{itemize}

See Appendix~\ref{app:multiagent-details} for detailed treatment of the MRF formulation and three-level learning dynamics.

The recursive and multi-agent extensions demonstrate that BEDS principles
scale beyond individual learning systems to distributed and hierarchical
architectures.


\section{Problem Taxonomy}
\label{sec:problem-taxonomy}

The product structure $\bedsM = \HH^2 \times \MvM$ established in Section~\ref{sec:beds-framework} has a remarkable consequence: it induces a natural classification of machine learning problems based on the asymptotic behavior of each geometric factor. This taxonomy provides practitioners with a principled way to categorize problems \emph{a priori}, before training begins.

\subsection{From Geometry to Classification}
\label{subsec:taxonomy-intro}

Each factor of the product manifold can exhibit one of two fundamental behaviors:
\begin{itemize}
    \item \textbf{Convergent behavior}: the component reaches a stable equilibrium
    \item \textbf{Tracking behavior}: the component must continuously adapt to a moving target
\end{itemize}

Since the two factors ($\tau$ for spatial precision, $\kappa$ for temporal coherence) are geometrically independent under our assumptions, we obtain a combinatorial structure that classifies all learning problems into exactly six practically-relevant classes.

\subsection{Formal Definitions}
\label{subsec:taxonomy-definitions}

Before presenting the formal taxonomy, we establish precise definitions for
the two fundamental regimes that characterize learning dynamics. These
definitions operationalize the intuitive distinction between problems that
``settle down'' (BEDS-crystallizable) and problems that require continuous
adaptation (BEDS-maintainable). The mathematical precision of these definitions
is essential for the classification scheme that follows.

\begin{definition}[BEDS-Crystallizable Regime]
\label{def:crystallizable}
A learning problem is \emph{BEDS-crystallizable} on a component $\theta \in \{\tau, \kappa\}$ if the dynamics converge to a fixed point:
\begin{equation}
    \lim_{t \to \infty} \theta(t) = \theta^* \quad \text{where } \theta^* \text{ is a stable equilibrium.}
\end{equation}
\end{definition}

In the BEDS-crystallizable regime, the system can ``freeze'' into a stable configuration, analogous to crystallization in physical systems. Once equilibrium is reached, no further energy expenditure is required to maintain the state.

\begin{definition}[BEDS-Maintainable Regime]
\label{def:maintainable}
A learning problem is \emph{BEDS-maintainable} on a component $\theta \in \{\tau, \kappa\}$ if the optimal state is time-varying:
\begin{equation}
    \theta^*(t) \neq \text{const} \quad \Rightarrow \quad \text{continuous tracking required.}
\end{equation}
\end{definition}

In the BEDS-maintainable regime, the target itself evolves over time (due to distribution shift, changing objectives, or exploration requirements), and the system must continuously expend energy to track it.

\subsection{The Six Problem Classes}
\label{subsec:six-classes}

The formal classification of learning problems emerges directly from the
product structure of the BEDS manifold. Since the spatial component ($\tau$)
and the temporal component ($\kappa$) evolve on geometrically independent
factors under our assumptions, we can characterize each problem by the
asymptotic behavior of each factor separately. This leads to a precise
taxonomy that practitioners can apply \emph{a priori}, before training
begins, to guide algorithm selection and hyperparameter tuning.

\begin{proposition}[Problem Taxonomy]
\label{prop:taxonomy}
Under Assumptions A1--A3, the product structure $\bedsM = \HH^2 \times \MvM$ induces exactly six classes of learning problems:
\begin{center}
\begin{tabular}{llll}
\toprule
\textbf{Class} & \textbf{Component $\tau$} & \textbf{Component $\kappa$} & \textbf{Description} \\
\midrule
\textbf{C-$\tau$} & BEDS-crystallizable & (any) & Spatial precision converges \\
\textbf{C-$\kappa$} & (any) & BEDS-crystallizable & Temporal coherence converges \\
\textbf{C-full} & BEDS-crystallizable & BEDS-crystallizable & Complete crystallization \\
\textbf{M-$\tau$} & BEDS-maintainable & (any) & Precision tracking required \\
\textbf{M-$\kappa$} & (any) & BEDS-maintainable & Coherence maintenance required \\
\textbf{M-full} & BEDS-maintainable & BEDS-maintainable & Full tracking mode \\
\bottomrule
\end{tabular}
\end{center}
\end{proposition}

\begin{proof}
The product structure decomposes belief space into two geometrically independent factors. Each factor admits exactly two asymptotic regimes: convergence to equilibrium (BEDS-crystallizable) or continuous tracking (BEDS-maintainable). The six classes arise from three modes of practical interest---$\tau$-dominant, $\kappa$-dominant, and coupled---combined with two possible regimes, yielding $3 \times 2 = 6$ classes. Independence of the factors under A1--A3 ensures these classes are well-defined and mutually exclusive when considering component-specific behavior.
\end{proof}

\proven

This classification has immediate practical implications. Identifying the
regime of a problem informs the choice of algorithm, the setting of
hyperparameters, and the design of monitoring diagnostics. The following
sections provide concrete criteria for classification and examples spanning
the major paradigms of machine learning.

\begin{figure}[H]
\centering
\begin{tikzpicture}[
    scale=1.0,
    every node/.style={font=\small},
    cell/.style={minimum width=4.2cm, minimum height=2.2cm, align=center},
    clabel/.style={font=\footnotesize\bfseries, text=bedsblue},
    example/.style={font=\scriptsize\itshape, text=gray}
]

\draw[thick] (-4.2,-2.2) rectangle (4.2,2.2);
\draw[thick] (0,-2.2) -- (0,2.2);
\draw[thick] (-4.2,0) -- (4.2,0);

\node[font=\bfseries, rotate=90] at (-5.3,0) {$\kappa$ (temporal)};
\node[font=\bfseries] at (0,-2.9) {$\tau$ (spatial)};

\node[font=\small, bedsgreen] at (-2.1,2.5) {BEDS-cryst.};
\node[font=\small, bedsorange] at (2.1,2.5) {BEDS-maint.};
\node[font=\small, bedsgreen, rotate=90] at (-4.55,1.1) {B-cryst.};
\node[font=\small, bedsorange, rotate=90] at (-4.55,-1.1) {B-maint.};

\fill[lightgreen, opacity=0.4] (-4.2,0) rectangle (0,2.2);
\node[clabel] at (-2.1,1.6) {\textbf{C-full}};
\node[example] at (-2.1,1.0) {DINO converged};
\node[example] at (-2.1,0.5) {Stable SSL};

\fill[lightorange, opacity=0.4] (0,0) rectangle (4.2,2.2);
\node[clabel] at (2.1,1.6) {\textbf{M-$\tau$ + C-$\kappa$}};
\node[example] at (2.1,1.0) {Online learning};
\node[example] at (2.1,0.5) {Concept drift};

\fill[lightorange, opacity=0.4] (-4.2,-2.2) rectangle (0,0);
\node[clabel] at (-2.1,-0.6) {\textbf{C-$\tau$ + M-$\kappa$}};
\node[example] at (-2.1,-1.2) {RL exploration};
\node[example] at (-2.1,-1.7) {SAC, PPO};

\fill[lightred, opacity=0.4] (0,-2.2) rectangle (4.2,0);
\node[clabel] at (2.1,-0.6) {\textbf{M-full}};
\node[example] at (2.1,-1.2) {Continual RL};
\node[example] at (2.1,-1.7) {World models};

\node[draw, rounded corners=2pt, fill=lightblue, font=\scriptsize] at (-5.9,1.1) {C-$\kappa$};
\node[draw, rounded corners=2pt, fill=lightblue, font=\scriptsize] at (-5.9,-1.1) {M-$\kappa$};
\node[draw, rounded corners=2pt, fill=lightblue, font=\scriptsize] at (-2.1,-3.3) {C-$\tau$};
\node[draw, rounded corners=2pt, fill=lightblue, font=\scriptsize] at (2.1,-3.3) {M-$\tau$};

\end{tikzpicture}
\caption{\textbf{The six BEDS problem classes.} The product structure $\HH^2 \times \MvM$ partitions learning problems based on whether each component (spatial precision $\tau$, temporal coherence $\kappa$) is BEDS-crystallizable or BEDS-maintainable. This classification follows directly from Assumptions A1--A3.}
\label{fig:problem-taxonomy}
\end{figure}
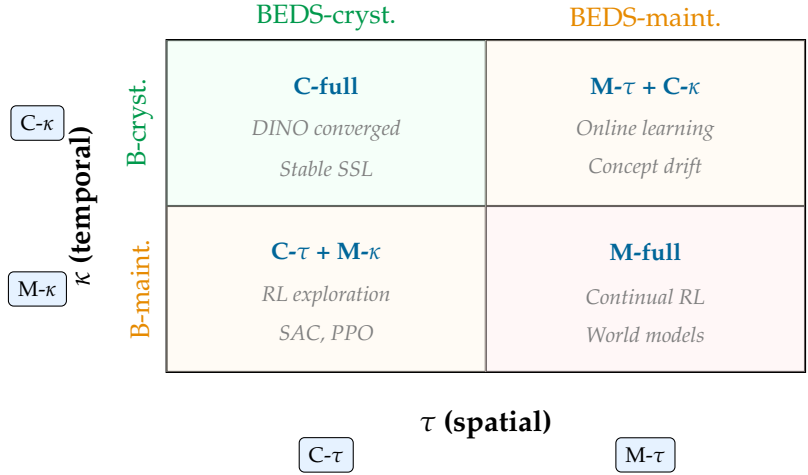

\subsection{Practical Classification Criteria}
\label{subsec:classification-criteria}

To classify a problem before training, practitioners can use the following criteria:

\begin{center}
\begin{tabular}{lcc}
\toprule
\textbf{Criterion} & \textbf{C (BEDS-crystallizable)} & \textbf{M (BEDS-maintainable)} \\
\midrule
Environment & Stationary & Non-stationary \\
Data distribution & Fixed & Drifting \\
Learning objective & Fixed point & Moving target \\
$\mathcal{C} = \tau \cdot \kappa$ & Can grow unbounded & Must remain bounded \\
Optimal $\dot{S}$ & $\to 0$ asymptotically & $> 0$ always \\
\bottomrule
\end{tabular}
\end{center}

The crystallization index $\mathcal{C}$ provides a quantitative diagnostic: in BEDS-crystallizable regimes, $\mathcal{C}$ may increase without bound as the system settles; in BEDS-maintainable regimes, excessive $\mathcal{C}$ indicates dangerous rigidity.

\subsection{Examples and Applications}
\label{subsec:taxonomy-examples}

\begin{table}[htbp]
\centering
\caption{\textbf{Classification of common ML problems.} Each example is categorized by its asymptotic behavior on the two geometric components.}
\label{tab:taxonomy-examples}
\begin{tabular}{llp{7cm}}
\toprule
\textbf{Class} & \textbf{Example} & \textbf{Characteristics} \\
\midrule
\textbf{C-$\tau$} & Ridge regression, image classification & Labels are fixed; model converges to stable precision over training \\
\textbf{C-$\kappa$} & BYOL initialization & Momentum teacher synchronization dominates; coherence stabilizes \\
\textbf{C-full} & DINO fully trained & Both precision and coherence reach equilibrium; stable representations \\
\textbf{M-$\tau$} & Online learning, concept drift & Precision must continuously adapt to shifting distributions \\
\textbf{M-$\kappa$} & RL exploration (SAC, PPO) & Low coherence enables exploration; entropy regularization \\
\textbf{M-full} & Continual RL, active world models & Both components in tracking mode; full adaptation required \\
\bottomrule
\end{tabular}
\end{table}

\begin{remark}[Taxonomy and Algorithm Selection]
The problem class directly informs algorithm choice. BEDS-crystallizable problems (C-full) benefit from aggressive crystallization (high learning rates, strong regularization toward equilibrium). BEDS-maintainable problems (M-full) require careful entropy management to prevent premature convergence. Mixed regimes (M-$\tau$ + C-$\kappa$ or C-$\tau$ + M-$\kappa$) suggest asymmetric treatment of the two geometric factors.
\end{remark}

This taxonomy provides a principled basis for algorithm selection. The
following section explores extensions of BEDS to optimization schedules
and connections with quantum mechanics.

\section{Extensions and Connections}
\label{sec:extensions}

The BEDS framework, while primarily focused on regularization and dissipation in machine learning, connects to broader themes in optimization and physics. This section explores two such connections: graduated non-convexity as a dissipative schedule, and structural parallels with quantum mechanics. These extensions are not central to the main argument but suggest that the thermodynamic perspective on learning may have deeper roots and wider applicability than initially apparent.

\subsection{Graduated Non-Convexity as Dissipative Continuation}
\label{subsec:gnc}

Graduated non-convexity (GNC) is a powerful technique in non-convex optimization where the objective function is progressively ``revealed'' from a smooth, convex surrogate to the full non-convex target. This can be formulated as an explicit continuation $\mathcal{L}_{\alpha(t)}$ from a smoothed objective ($\alpha \approx 0$) to the full non-convex objective ($\alpha = 1$). Note that $\alpha(t)$ here denotes a revelation schedule, distinct from SAC's temperature parameter $\alpha$ (Section~\ref{subsec:bridging}).

\paragraph{BEDS Interpretation.}
Within the BEDS paradigm, a natural way to formalize controlled forgetting is to couple this continuation with a dissipative term that penalizes deviation from a reference belief state:
\begin{equation}
\min_q \mathbb{E}_{q}[\mathcal{L}_{\alpha(t)}] + \beta(t) D_{\KL}(q\|\pi)
\label{eq:gnc-beds}
\end{equation}
Here $\alpha(t)$ progressively injects structural complexity (revealing fine-grained details of the objective landscape), while $\beta(t)$ controls information disposal and prevents premature crystallization.

This viewpoint interprets GNC not as a source of forgetting per se, but as a schedule that organizes \emph{when} fine-grained structure becomes learnable under explicit dissipation constraints.

\paragraph{Mathematical Formalization.}
The interplay between $\alpha(t)$ and $\beta(t)$ can be formalized as follows. Define the \emph{effective temperature} as the inverse of the dissipation strength:
\begin{equation}
T_{\text{eff}}(t) = \frac{1}{\beta(t)} \quad \text{(equivalently: } \beta(t) = 1/T_{\text{eff}}(t) \text{)}
\end{equation}
At high effective temperature (low $\beta$), the system explores broadly; at low effective temperature (high $\beta$), it crystallizes around local structure. The continuation schedule $\alpha(t)$ controls which features of the landscape are visible:
\begin{equation}
\mathcal{L}_{\alpha}(\theta) = (1-\alpha) \mathcal{L}_{\text{smooth}}(\theta) + \alpha \mathcal{L}_{\text{target}}(\theta)
\end{equation}

\paragraph{Connection to Simulated Annealing.}
Simulated annealing~\cite{kirkpatrick_1983} corresponds to the limiting case where $\alpha = 1$ throughout (full objective always visible) and only $\beta(t)$ varies---typically increasing from low to high. GNC generalizes this by also varying which aspects of the objective are revealed.

\paragraph{Examples and Applications.}
GNC principles appear in several machine learning contexts:
\begin{itemize}
    \item \textbf{Robust estimation}: Black and Anandan's~\cite{black_anandan_1996} graduated non-convexity for optical flow progressively sharpens the robust penalty function, following the foundational work of Blake and Zisserman~\cite{blake_zisserman_1987} on visual reconstruction.
    \item \textbf{Curriculum learning}: Training on easy examples first, then harder ones~\cite{bengio_2009}, implicitly implements $\alpha(t)$ by controlling which data features dominate the loss.
    \item \textbf{Knowledge distillation}: The teacher's soft targets provide a smooth surrogate; temperature annealing implements $\beta(t)$.
    \item \textbf{Diffusion models}: The denoising schedule can be viewed as an $\alpha(t)$ schedule revealing signal structure.
\end{itemize}

\paragraph{Testable Prediction.}
The BEDS framework predicts a specific relationship between optimal schedules:
\begin{conjbox}[GNC-Dissipation Coupling]
For optimal learning dynamics, the dissipation schedule $\beta(t)$ should increase as the revelation schedule $\alpha(t)$ increases. Intuitively: as more structure becomes learnable, the system should crystallize more strongly around it.
\end{conjbox}

This prediction is testable: compare GNC with coupled schedules (where $\beta$ tracks $\alpha$) against GNC with independent schedules. The coupled version should achieve better final solutions with less wasted computation.

\subsection{Structural Parallels with Quantum Mechanics}
\label{subsec:quantum-parallels}

The connection between BEDS and quantum mechanics is not merely analogical---it reflects deep structural constraints on any information-processing system that must preserve information while allowing interference between belief states.

\paragraph{Why This Connection Matters for ML.}
Understanding the quantum parallels illuminates three aspects of machine learning:
\begin{enumerate}
    \item \textbf{Quantum-inspired algorithms}: Techniques like quantum annealing~\cite{kadowaki_nishimori_1998} and variational quantum circuits~\cite{peruzzo_2014} may find classical analogues in BEDS dynamics. Recent surveys~\cite{biamonte_2017,preskill_2018} explore these connections.
    \item \textbf{Coherent vs.\ decoherent learning}: The distinction between reversible (Schr\"odinger-like) and irreversible (dissipative) dynamics clarifies when information is created, preserved, or destroyed during learning~\cite{zurek_2003,schlosshauer_2007}.
    \item \textbf{Fundamental limits}: The Heisenberg uncertainty principle has a BEDS analogue constraining simultaneous crystallization of conjugate variables~\cite{nielsen_chuang_2000}.
\end{enumerate}

The BEDS framework with its extended state $(\mu, \sigma, p, \varphi)$ exhibits deep structural parallels with quantum mechanics. When dissipation is negligible ($\gamma \approx 0$), the requirement that evolution must preserve information---neither creating nor destroying it---leads directly to quantum-like dynamics.

\paragraph{From BEDS Parameters to Complex Amplitudes.}
The phase parameter $\varphi$ necessitates a \emph{complex} representation. Just as a complex number $z = r e^{i\varphi}$ carries both magnitude and direction, the extended BEDS state encodes both uncertainty ($\sigma$) and accumulated history ($\varphi$). The full state can be written as a complex amplitude:
\begin{equation}
  \Psi(x) = \left(\frac{1}{2\pi\sigma^2}\right)^{1/4} \exp\left(-\frac{(x-\mu)^2}{4\sigma^2} + i\frac{p(x-\mu)}{\hbar} + i\varphi\right)
\end{equation}
This is precisely a Gaussian wave packet in quantum mechanics. Complex numbers are not imposed but \emph{emerge} because real numbers cannot encode phase, and phase is necessary for reversible evolution.

\paragraph{The Bloch Sphere Visualization.}
For the simplest case (a two-state system), the BEDS state can be visualized on the Bloch sphere (Figure~\ref{fig:bloch-sphere-ml}). Three processes operate on this space:

\begin{figure}[H]
\centering
\begin{tikzpicture}[scale=1.1]
  \shade[ball color=gray!10, opacity=0.3] (0,0) circle (2cm);
  \draw[thick] (0,0) circle (2cm);

  \draw[dashed, gray] (-2,0) arc (180:360:2 and 0.5);
  \draw[thick] (-2,0) arc (180:0:2 and 0.5);

  \draw[thick, ->] (0,-2.3) -- (0,2.5) node[above] {$|0\rangle$};
  \node[below] at (0,-2.3) {$|1\rangle$};
  \draw[thick, ->] (-2.3,0) -- (2.5,0) node[right] {$|+\rangle$};
  \node[left] at (-2.3,0) {$|-\rangle$};

  \fill[bedsblue] (0.8, 1.5) circle (3pt);
  \node[right, font=\small, text=bedsblue] at (1, 1.5) {pure state};

  \draw[very thick, bedsgreen, ->] (0.8, 1.5) arc (60:120:1.8 and 0.4);
  \node[above right, font=\scriptsize, text=bedsgreen] at (0.3, 1.9) {1. Schrödinger (rotation)};

  \fill[gray!50] (0.3, 0.5) circle (2pt);
  \draw[very thick, bedsorange, ->] (0.6, 1.2) -- (0.35, 0.6);
  \node[right, font=\scriptsize, text=bedsorange] at (0.7, 0.85) {2. Dissipation (fall)};

  \draw[very thick, bedsred, ->, dashed] (-0.5, 0.8) -- (-0.1, 1.9);
  \node[left, font=\scriptsize, text=bedsred] at (-0.6, 1.3) {3. Measurement};

  \fill[gray] (0,0) circle (2pt);
  \node[below, font=\scriptsize, text=gray] at (0, -0.5) {maximally mixed};

  \node[draw, thick, fill=white, font=\scriptsize, align=left] at (5.2, -0.5) {
    \textcolor{bedsgreen}{$\bullet$ Schrödinger}: reversible\\
    \textcolor{bedsorange}{$\bullet$ Dissipation}: info lost\\
    \textcolor{bedsred}{$\bullet$ Measurement}: crystallize
  };
\end{tikzpicture}
\caption{The Bloch sphere representation of BEDS states. Three processes: (1) Schrödinger evolution rotates on the surface (information conserved), (2) Dissipation pulls toward center (information lost), (3) Measurement jumps to poles (crystallization).}
\label{fig:bloch-sphere-ml}
\end{figure}
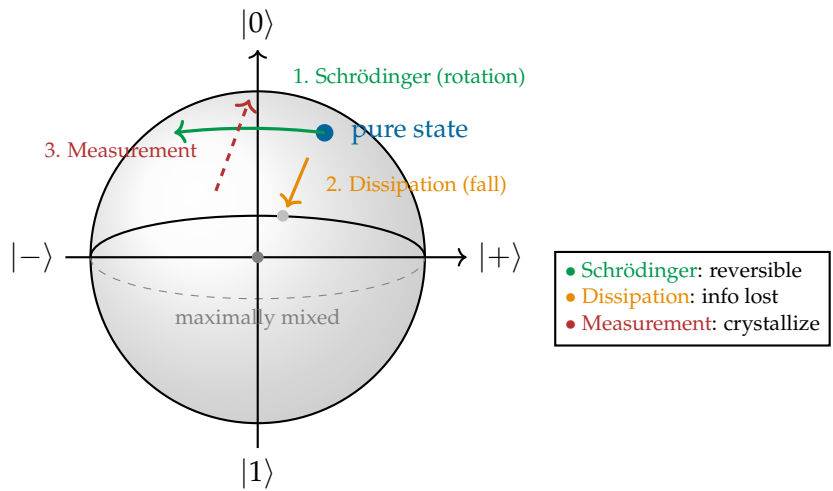

\begin{table}[ht]
\centering
\small
\begin{tabular}{lll}
\toprule
\textbf{Position on Sphere} & \textbf{State Type} & \textbf{BEDS Interpretation} \\
\midrule
Poles ($|0\rangle$, $|1\rangle$) & Crystallized & Belief frozen at definite value \\
Surface & Pure state & Minimum uncertainty (coherent) \\
Interior & Mixed state & Partial information (dissipated) \\
Center & Maximally mixed & No information (maximum entropy) \\
\bottomrule
\end{tabular}
\caption{Bloch sphere positions and their BEDS interpretation.}
\label{tab:bloch-interpretation}
\end{table}

\paragraph{BEDS--Quantum Mechanics Correspondence.}
The structural parallels between BEDS and quantum mechanics are summarized in Table~\ref{tab:beds-qm}.

\begin{table}[ht]
\centering
\small
\begin{tabular}{p{4cm}p{4cm}p{5.5cm}}
\toprule
\textbf{BEDS Concept} & \textbf{QM Concept} & \textbf{Connection} \\
\midrule
Complex amplitudes $(\mu, \sigma, p, \varphi)$ & Wave function $\psi$ & Phase necessary for reversibility \\
Fisher isometry (info-preserving) & Schrödinger equation & Unique evolution preserving information \\
Crystallization constraint & Heisenberg uncertainty & $\sigma_\mu \sigma_p \geq \hbar/2$ \\
Dissipation $\gamma\tau$ decay & Decoherence & Same differential equation \\
Observation/Crystallization & Measurement/Collapse & Irreversible information gain \\
Precision $\tau$ & Purity $\text{Tr}(\rho^2)$ & Information content measure \\
\bottomrule
\end{tabular}
\caption{Correspondence between BEDS framework and quantum mechanics.}
\label{tab:beds-qm}
\end{table}

The key insight is that the Schrödinger equation emerges as the \emph{unique} information-preserving evolution on the BEDS manifold, while the Heisenberg uncertainty principle reflects the impossibility of simultaneously crystallizing conjugate variables. The dissipative dynamics follow the Lindblad master equation~\cite{lindblad_1976,gorini_1976}, the most general form of Markovian open quantum dynamics. This suggests that quantum mechanics may be the natural language for describing any information-preserving system with phase degrees of freedom.

\paragraph{Limitations of the Analogy.}
Several important differences distinguish BEDS from quantum mechanics:
\begin{itemize}
    \item \textbf{No superposition of macrostates}: Unlike quantum systems, BEDS agents do not exist in superpositions of distinct belief states. The ``superposition'' in BEDS is over parameter values within a single belief distribution, not over qualitatively different configurations.
    \item \textbf{Classical communication}: BEDS agents communicate via classical messages (means and precisions), not entangled quantum states. There is no BEDS analogue of quantum non-locality.
    \item \textbf{Decoherence is controllable}: In quantum mechanics, decoherence is typically an undesirable environmental effect. In BEDS, dissipation ($\gamma > 0$) is a designed feature enabling adaptation.
    \item \textbf{No Planck's constant}: The BEDS ``$\hbar$'' is a dimensional parameter set by the information-geometric structure, not a fundamental physical constant. Different systems may have different effective ``$\hbar$'' values.
\end{itemize}

These limitations do not diminish the structural insight: the mathematics that emerges from requiring information-preserving dynamics on a curved manifold is formally identical to quantum mechanics. This suggests that quantum-like phenomena may arise in any sufficiently constrained information-processing system, independent of physical scale.


\section{Predictions and Structural Consequences}
\label{sec:predictions}

The BEDS framework does not aim to provide quantitative performance predictions
nor to describe the detailed dynamics of real learning algorithms.
Instead, it yields a set of qualitative and structural predictions that are expected
to hold across learning paradigms, levels of abstraction, and implementation regimes,
including discrete, stochastic, and non-equilibrium settings.

The value of a theoretical framework lies not only in its explanatory power
but in its predictive capacity. The BEDS framework, while primarily conceptual,
yields a set of structural predictions about learning systems. These predictions
are qualitative rather than quantitative---they concern regimes, trade-offs,
and failure modes rather than specific numerical bounds. Their purpose is to
guide intuition and suggest experimental directions, not to replace empirical
investigation.

These predictions are organized by category and should be interpreted as statements
about regimes, trade-offs, and failure modes rather than mechanistic laws.

\subsection*{A. Predictions about Forgetting and Viability}

\begin{itemize}
\item \textbf{Forgetting is a structural necessity.}  
Learning systems that continuously accumulate information without effective mechanisms
for information disposal will eventually become brittle, overconfident, or unable to adapt
to non-stationary environments. Forgetting is therefore a functional requirement for
long-term viability rather than a design flaw.

\item \textbf{Uncontrolled information retention leads to loss of adaptability.}  
As internal representations become increasingly rigid, the system’s ability to respond
to distributional shifts degrades, even if short-term performance improves.
This effect is expected to manifest in continual learning and digital twin settings.
\end{itemize}

Beyond forgetting, BEDS predicts characteristic patterns in learning stability.

\subsection*{B. Predictions about Stability and Regime Transitions}

\begin{itemize}
\item \textbf{Learning dynamics exhibit distinct stability regimes.}  
Overfitting, representation collapse, and catastrophic forgetting correspond to transitions
between qualitatively different dynamical regimes.
These transitions are governed by the balance between information acquisition,
model capacity, and effective forgetting.

\item \textbf{Failure modes are preceded by detectable changes in variability and coherence.}  
Before collapse or loss of adaptability occurs, learning dynamics typically exhibit
increased instability, reduced variability, or excessive synchronization,
suggesting that early-warning diagnostics are possible.
\end{itemize}

These stability regimes imply fundamental trade-offs.

\subsection*{C. Predictions about Trade-offs and Constraints}

\begin{itemize}
\item \textbf{Precision, stability, and plasticity are fundamentally coupled.}  
Improving local accuracy or precision tends to reduce long-term adaptability,
introducing an unavoidable trade-off between performance, stability, and plasticity.
No learning algorithm can simultaneously optimize all three indefinitely.

\item \textbf{Resource constraints shape learning behavior as much as data.}  
Finite memory, energy, and communication budgets impose constraints that are as fundamental
as data availability or model capacity, particularly in long-lived or distributed systems.
\end{itemize}

The existence of these trade-offs illuminates why certain heuristics work.

\subsection*{D. Predictions about Existing Heuristics}

\begin{itemize}
\item \textbf{Successful learning heuristics implicitly enforce dissipative constraints.}  
Techniques such as dropout, noise injection, data augmentation, masking,
exponential moving averages, and stochastic optimization can be interpreted
as mechanisms that regulate information retention and export.
Their empirical effectiveness is consistent with the need to control dissipation,
even though they are not derived from explicit thermodynamic principles.

\item \textbf{Heuristic diversity reflects multiple ways of managing dissipation.}  
Different regularization and stabilization techniques correspond to different
strategies for balancing information acquisition and forgetting,
rather than fundamentally different learning principles.
\end{itemize}

At the theoretical level, BEDS defines reference points for evaluating
learning systems.

\subsection*{E. Predictions about Idealized Optimality and Reference Geometry}

\begin{itemize}
\item \textbf{Information-geometric optimality defines a reference, not a mechanism.}  
The Fisher--Rao geometry and quasi-static trajectories characterize an idealized
lower bound on information dissipation in a continuous setting.
Real learning systems are not expected to realize this bound,
but their behavior may be meaningfully compared to it in terms of relative efficiency
and stability.

\item \textbf{Deviations from the ideal regime are expected and necessary.}  
Discrete computation, stochastic updates, and non-stationary environments
inevitably push real systems away from the quasi-static limit,
without invalidating the usefulness of the reference bound.
\end{itemize}

These principles become particularly relevant in long-lived systems.

\subsection*{F. Predictions about Continual and Distributed Systems}

\begin{itemize}
\item \textbf{Dissipation becomes dominant in long-lived and distributed settings.}  
In continual learning, multi-agent systems, and digital twins,
the cost of uncontrolled information accumulation grows with time,
making controlled forgetting more critical than asymptotic performance
on static benchmarks.

\item \textbf{Viability replaces optimality as the primary criterion.}  
In such systems, the relevant question is not whether learning converges to a global optimum,
but whether it remains stable, adaptive, and resource-bounded over extended operation.
\end{itemize}

Taken together, these predictions position BEDS as a framework for classifying
learning problems by regime, diagnosing failure modes, and reasoning about viability
under finite resources, rather than as a prescriptive theory of algorithmic optimization.
The following section examines the limitations of our assumptions and directions
for future work.

\subsection{Prediction: Gaussian Splatting as a Near-Optimal Dissipative Learning System}
\label{subsec:prediction-gs}

The BEDS framework is deliberately introduced as a reference theory rather than an immediately implementable algorithm. A central value of such a framework lies in its ability to generate concrete, falsifiable predictions about real learning systems.

We formulate the following prediction regarding Gaussian Splatting.

\paragraph{Prediction P-GS (Dissipative Optimality of Gaussian Splatting).}
\emph{Gaussian Splatting constitutes a near-optimal practical approximation of the BEDS learning dynamics under Assumptions A1--A3.}

More precisely:

\begin{enumerate}
    \item \textbf{Belief Representation (A2).}  
    Gaussian splats explicitly parameterize maximum-entropy belief states, namely Gaussian distributions over spatial structure. This provides a direct instantiation of Assumption A2, without requiring an intermediate learned latent representation.

    \item \textbf{Learning Geometry (A1).}  
    The optimization of Gaussian parameters (means and covariances) under small, incremental updates empirically approximates quasi-static trajectories in Fisher--Rao geometry, since the Fisher--Rao metric corresponds to the local second-order limit of Kullback--Leibler divergence between nearby Gaussian belief states.

    \item \textbf{Dissipation and Viability (A3).}  
    Mechanisms commonly employed in Gaussian Splatting---such as opacity decay, covariance regularization, densification control, and pruning---can be interpreted as explicit entropy-export mechanisms. These prevent over-crystallization (excessive precision or structural rigidity) and enable the maintenance of bounded, viable belief states over time.
\end{enumerate}

\paragraph{Consequence.}
Under this interpretation, Gaussian Splatting is not merely a rendering or reconstruction technique, but a \emph{dissipative belief field} whose training dynamics approximate the thermodynamically optimal reference trajectories predicted by the BEDS framework.
\emph{This interpretation receives independent mathematical support from recent work on Wasserstein-Fisher-Rao gradient flows for splat models~\cite{daniels_rigollet_2025}.}

\paragraph{Extended Prediction: JEPA on Gaussian Belief Fields.}
A direct consequence of this perspective is the following speculative but testable prediction:

\begin{quote}
\emph{JEPA-style masked prediction can be performed directly on Gaussian primitives, rather than on pixels or learned latent embeddings.}
\end{quote}

In this setting, subsets of Gaussian splats are masked and predicted from the remaining visible splats, with learning driven by minimizing Fisher--Rao distance between predicted and target belief states. This yields a fully self-supervised, dissipative learning process operating natively in belief space.

While not explored experimentally in this work, this prediction illustrates how the BEDS framework naturally unifies modern self-supervised learning and 3D reconstruction paradigms, and provides a concrete direction for future empirical validation.

\vspace{0.5cm}
\hrule
\vspace{0.3cm}
\noindent\textit{Part IV provides critical perspective. We examine what
the framework predicts, acknowledge its limitations, and identify
directions for future development.}
\vspace{0.3cm}
\hrule
\vspace{0.5cm}


\FloatBarrier
\part{Critical Analysis}


\section{Discussion}
\label{sec:discussion}

Every theoretical framework carries limitations, and intellectual honesty
demands that we make ours explicit. The assumptions underlying BEDS are
physically motivated but not universally applicable. The quasi-static
idealization, in particular, represents an extreme limit that real learning
systems never achieve. This section examines these limitations, discusses
practical implementation challenges, and identifies directions for future
development.

\subsection{Limitations of the Assumptions}
\label{subsec:limitations}

Each assumption underlying BEDS carries limitations that affect its
applicability.

\textbf{A1 (Intrinsic measure)}: While physically motivated, practical implementations often use coordinate-dependent approximations. The degree to which this matters empirically is an open question.

\textbf{A2 (Maximum entropy)}: Real belief states may not be well-approximated by simple exponential families. Deep networks may learn distributions far from maximum entropy. Extending to more expressive families is a key direction.

\textbf{A3 (Quasi-static)}: Real learning operates far from the quasi-static limit. SGD with finite learning rates is highly non-equilibrium. The quasi-static regime provides lower bounds, but actual inefficiencies may be much larger.

\subsection{The Quasi-Static Idealization}
\label{subsec:quasistatic-discussion}

The quasi-static assumption (A3) is the most restrictive. In practice, several factors violate the quasi-static idealization. SGD noise exceeds thermal noise by approximately $10^{10}$ times. Learning rate schedules inherently violate equilibrium assumptions. Batch processing is fundamentally irreversible.

Despite these violations, the framework remains valuable. It establishes fundamental bounds---lower limits on achievable efficiency. It provides directional guidance---move toward geodesic trajectories. It offers diagnostic tools---the crystallization index and phase diagrams that characterize learning dynamics.

\subsection{Practical Fisher--Rao Implementation}
\label{subsec:practical-fr}

A central challenge for applying BEDS is the computational cost of exact Fisher--Rao regularization. The Fisher information matrix $F(\theta)$ has dimension $d \times d$ where $d$ is the number of parameters---billions for modern language models. Computing and storing $F$ exactly is infeasible.

\paragraph{Tractable Approximations.}
Several approximation strategies exist, ordered by fidelity:

\begin{enumerate}
    \item \textbf{Diagonal Fisher}: Keep only diagonal entries $F_{ii}$, reducing storage to $O(d)$. This loses cross-parameter interactions but captures per-parameter precision. Empirically effective for Adam-style optimizers where diagonal preconditioning dominates.

    \item \textbf{K-FAC (Kronecker-Factored Approximate Curvature)}: For neural networks, approximate $F \approx A \otimes B$ where $A$ captures activation statistics and $B$ captures gradient statistics. Reduces complexity from $O(d^2)$ to $O(d)$ while preserving layer-wise structure. Introduced by Martens \& Grosse~\cite{martens_grosse_2015}.

    \item \textbf{Low-rank plus diagonal}: $F \approx D + UU^T$ where $D$ is diagonal and $U \in \R^{d \times r}$ with $r \ll d$. Captures the most important directions of curvature while remaining tractable.

    \item \textbf{Empirical Fisher}: Use $\E[\nabla \log p \cdot \nabla \log p^T]$ estimated from mini-batches. Biased but unbiased in expectation; commonly used in natural gradient methods.
\end{enumerate}

\paragraph{SIGReg as Fisher--Rao Proxy.}
The SIGReg regularizer (Garrido et al., 2023), which penalizes deviation from $\Normal(0, I)$, can be interpreted as an implicit Fisher--Rao regularization. For embeddings $z$ with empirical covariance $\Sigma$:
\begin{equation}
\mathcal{L}_{\text{SIGReg}} = \|\Sigma - I\|_F^2 + \|\bar{z}\|^2
\end{equation}
This is computationally cheap (linear in embedding dimension) and empirically effective, suggesting that even crude approximations to Fisher--Rao capture essential geometric structure.

\subsection{What is Proven vs Conjectured}
\label{subsec:proven-vs-conjectured}

For clarity, we summarize the epistemic status of each claim made in
this work.

\begin{table}[htbp]
\centering
\caption{\textbf{Epistemic status of claims.}}
\label{tab:epistemic-status}
\begin{tabular}{lll}
\toprule
\textbf{Claim} & \textbf{Status} & \textbf{Basis} \\
\midrule
Conditional Optimality Theorem & \proven & Direct proof \\
Energy-Precision Bound & \proven & Corollary \\
Euclidean Suboptimality & \proven & Corollary \\
$\HH^2 \times \MvM$ geometry & \proven & Standard result \\
Six-class problem taxonomy & \proven & Product structure argument \\
\midrule
SSL efficiency advantage & \conjectured & Theoretical argument + empirical \\
Multi-agent MRF structure & \conjectured & Formal analogy \\
Kuramoto connection & \conjectured & Equation isomorphism \\
\midrule
GLP Conjecture & \speculative & Pattern recognition \\
Hallucinations as flux deficit & \speculative & Interpretive framework \\
\bottomrule
\end{tabular}
\end{table}

\subsection{Open Questions}
\label{subsec:open-questions}

Several important questions remain beyond the scope of this work.

\begin{enumerate}

    \item \textbf{Quantifying non-equilibrium}: How far from quasi-static is SGD? Can we bound the inefficiency?

    \item \textbf{Practical Fisher--Rao}: Can we implement efficient Fisher--Rao regularization at scale?

    \item \textbf{Coupling $\tau$ and $\kappa$}: When does the product structure break down?

    \item \textbf{Native belief representations}: Can self-supervised learning on Gaussian splatting primitives outperform approaches based on learned latent embeddings?
\end{enumerate}

Addressing these questions will require both theoretical advances and
empirical validation. We hope this framework provides a useful starting
point for such investigations.


\subsection{Speculative Perspectives: The GLP Conjecture}
\label{app:glp}

\speculative

\textbf{Note}: The following is highly speculative and offered as a direction for future inquiry, not a claim.

\medskip

The BEDS framework reveals that learning systems must dissipate entropy to maintain structured beliefs. This constraint is not merely technical---it echoes deeper results from logic, physics, and thermodynamics. Three foundational discoveries of the twentieth century, each arising independently, share a remarkable structural similarity: all concern systems capable of self-reference, all identify pathologies arising from closure, and all resolve these pathologies through some form of openness. We conjecture that these parallels are not coincidental, but reflect a universal trade-off faced by any system that maintains order against entropy.

\begin{itemize}
\item \textbf{G\"odel}~\cite{godel_1931}: Any consistent formal system capable of expressing arithmetic contains true statements unprovable within the system.

\item \textbf{Landauer}~\cite{landauer_1961}: Irreversible computation requires dissipation $\geq \kB T \ln 2$ per bit.

\item \textbf{Prigogine}~\cite{prigogine_1977_nobel}: Open systems far from equilibrium can maintain order through entropy export.
  
\end{itemize}

All three results share a common structure. Each involves a system capable of self-reference. Each identifies a closure condition leading to pathology. Each resolves the pathology through some form of openness---whether logical, computational, or thermodynamic.

\begin{conjbox}[GLP Conjecture (Speculative)]
Systems maintaining order against entropy face universal trade-offs connecting logical incompleteness, computational irreversibility, and thermodynamic dissipation. The ODR conditions (Openness, Dissipation, Recursion) characterize systems that avoid closure pathologies.
\end{conjbox}

This remains unproven and possibly unprovable in its current form.

\section{Conclusion}
\label{sec:conclusion}

This work does not close the theory of learning. It reframes it.

For decades, machine learning has developed through empirical discovery:
techniques like dropout, weight decay, and batch normalization emerged from
trial and error, their effectiveness confirmed by practice but not explained
by principle. The BEDS framework offers a different perspective---one where
these techniques are not arbitrary heuristics but manifestations of deeper
thermodynamic constraints.

\begin{keyresult}[The Central Insight]
Learning is not optimization in a vacuum. It is the maintenance of order
against dissipation---a thermodynamic process subject to fundamental bounds.
The geometry of belief space is not Euclidean but hyperbolic, and this
curvature has consequences: Fisher--Rao regularization is not merely
better than Euclidean alternatives, it is \emph{uniquely optimal} under
information-theoretic constraints (Theorem~\ref{thm:conditional-optimality}).
The unifying equation
\begin{equation}
\mathcal{L} = \dF^2(\theta, \theta^*) + \lambda \cdot \mathcal{L}_{\text{data}}
\end{equation}
emerges not from empirical tuning but from geometric necessity.
\end{keyresult}

This reframing has immediate implications. Overfitting becomes
over-crystallization---not a failure of capacity but of dissipation control.
Catastrophic forgetting becomes insufficient structure in the prior hierarchy.
Mode collapse becomes a phase transition, predictable from the crystallization
index $\mathcal{C} = \tau \cdot \kappa$ (Figure~\ref{fig:beds-state-space-canonical}).
These are not metaphors; they are precise correspondences with testable
consequences. Table~\ref{tab:beds-rosetta} makes these translations explicit,
and Section~\ref{sec:how-to-use} provides practical guidance for applying them.

The deeper question this work raises is whether the thermodynamic perspective
extends beyond regularization. If learning is fundamentally dissipative,
then scaling laws, emergent capabilities, and the puzzling success of
overparameterized models may all find explanation within the same framework.
We do not claim to have answered these questions. We claim to have found
a language in which they can be precisely asked.

A natural next step is to instantiate the framework in belief-native representations, such as Gaussian fields, where dissipation, crystallization, and quasi-static trajectories can be observed directly.

\vspace{0.5cm}

\vspace{1cm}
\hrule
\vspace{0.5cm}

\FloatBarrier
\part{APPENDICES}

\appendix

\section{Mathematical Proofs}
\label{app:proofs}

\subsection{Rate-Distortion Foundation}
\label{app:rate-distortion}

The proofs rest on Shannon's rate-distortion theory. For a Gaussian source with variance $\sigma^2$, maintaining mean squared error at most $D$ requires:
\begin{equation}
R(D) = \frac{1}{2}\log_2\frac{\sigma^2}{D} \quad \text{bits}
\end{equation}

\subsection{Proof of Information Rate Theorem}
\label{app:proof-info-rate}

\begin{proof}
\textbf{Step 1}: Under dissipation, the innovation over interval $\Delta t$ is $\Delta\theta \sim \Normal(0, \gamma \Delta t)$.

\textbf{Step 2}: Rate-distortion with source variance $\gamma \Delta t$ and target distortion $\sigma^{*2}$:
\begin{equation}
R = \frac{1}{2}\log_2\frac{\gamma \Delta t}{\sigma^{*2}}
\end{equation}

\textbf{Step 3}: Taking $\Delta t \to 0$:
\begin{equation}
\dot{I} = \lim_{\Delta t \to 0} \frac{R}{\Delta t} = \frac{\gamma \tau^*}{2\ln 2}
\end{equation}
\end{proof}

\subsection{Proof of Energy-Precision Bound}
\label{app:proof-energy}

\begin{proof}
From the Information Rate Theorem: $\dot{I}_{\min} = \gamma \tau^* / (2\ln 2)$.

By Landauer's principle, each bit costs $\kB T \ln 2$ joules:
\begin{equation}
P_{\min} = \dot{I}_{\min} \cdot \kB T \ln 2 = \frac{\gamma \tau^* \kB T}{2}
\end{equation}

For $\tau^* = 1$: $P_{\min} \geq \gamma \kB T / 2$.
\end{proof}

\section{Fisher--Rao Metrics}
\label{app:fisher}

\subsection{Gaussian Family}

For $p(x|\mu, \sigma) = \Normal(\mu, \sigma^2)$:
\begin{equation}
ds^2 = \tau \, d\mu^2 + \frac{1}{2\tau^2} d\tau^2
\end{equation}
where $\tau = 1/\sigma^2$. This is the Poincar\'e half-plane metric.

\subsection{Von Mises Family}

For $p(\theta|\phi, \kappa) = \frac{\exp(\kappa \cos(\theta - \phi))}{2\pi I_0(\kappa)}$:
\begin{align}
g_{\phi\phi} &= \kappa A(\kappa) \\
g_{\kappa\kappa} &= 1 - A(\kappa)^2 - \frac{A(\kappa)}{\kappa}
\end{align}
where $A(\kappa) = I_1(\kappa)/I_0(\kappa)$.

\section{Multi-Agent BEDS (Extended)}
\label{app:multiagent-details}

\subsection{MRF Formulation}

For a network of $N$ agents with graph $G = (V, E)$:
\begin{equation}
P(\{S\}) \propto \exp(-E(\{S\})/T)
\end{equation}
where
\begin{equation}
E = \sum_i D_{\KL}(q_i \| p_i^{\text{data}}) + \sum_{(i,j) \in E} \lambda_{ij} \dF^2(q_i, q_j) + \sum_i D_{\KL}(q_i \| \pi_i)
\end{equation}

\subsection{Three Levels of Learning}

\begin{table}[htbp]
\centering
\caption{\textbf{Learning levels in BEDS networks.}}
\begin{tabular}{llll}
\toprule
\textbf{Level} & \textbf{What is learned} & \textbf{Timescale} & \textbf{Mechanism} \\
\midrule
1. Beliefs & $(\mu_i, \tau_i)$ & Fast & Bayesian fusion \\
2. Potentials & $\psi_{ij}$ & Medium & Agreement history \\
3. Topology & Structure of $G$ & Slow & Pruning weak links \\
\bottomrule
\end{tabular}
\end{table}

\subsection{BP-BEDS Equivalence}

For Gaussian beliefs, Belief Propagation exactly implements BEDS fusion: precisions add, means are precision-weighted averages.


\section{Correspondence Tables}
\label{app:correspondences}

\begin{table}[htbp]
\centering
\caption{\textbf{BEDS $\leftrightarrow$ Supervised ML.}}
\begin{tabular}{ll}
\toprule
\textbf{BEDS} & \textbf{Supervised ML} \\
\midrule
State $(\mu, \tau)$ & Parameters, confidence \\
Dissipation $\gamma$ & Weight decay rate \\
$\dF^2(\theta, \theta^*)$ & Regularization term \\
Crystallization & Convergence \\
\bottomrule
\end{tabular}
\end{table}

\begin{table}[htbp]
\centering
\caption{\textbf{BEDS $\leftrightarrow$ Self-Supervised Learning.}}
\begin{tabular}{ll}
\toprule
\textbf{BEDS} & \textbf{SSL} \\
\midrule
Spatial $(\mu, \tau)$ & Embedding distribution \\
SIGReg & Distance to $\Normal(0, I)$ \\
Temporal $(\phi, \kappa)$ & Teacher-student alignment \\
EMA momentum $m$ & $\kappa \approx 1/(1-m)$ \\
\bottomrule
\end{tabular}
\end{table}

\begin{table}[htbp]
\centering
\caption{\textbf{BEDS $\leftrightarrow$ Reinforcement Learning.}}
\begin{tabular}{ll}
\toprule
\textbf{BEDS} & \textbf{RL} \\
\midrule
State $(\mu, \tau)$ & Policy parameters \\
Coherence $\kappa$ & Inverse temperature \\
$1/\kappa$ & SAC temperature $\alpha$ \\
Low $\kappa$ & Exploration \\
High $\kappa$ & Exploitation \\
\bottomrule
\end{tabular}
\end{table}


\bibliographystyle{plainnat}
\bibliography{bibliography}

\end{document}